%% file: neurips_2025.tex
\documentclass{article}

\usepackage[numbers]{natbib}

\usepackage[preprint]{neurips_2025}

\usepackage[utf8]{inputenc} 
\usepackage[T1]{fontenc}    

\usepackage{booktabs}       
\usepackage{amsfonts}       
\usepackage{nicefrac}       
\usepackage{microtype}      
\usepackage[dvipsnames,svgnames]{xcolor}         
\definecolor{DarkKlein}{HTML}{122A82}
\definecolor{LinkBurgundy}{HTML}{8A1538}
\usepackage[
    colorlinks=true,
    linkcolor=LinkBurgundy,
    citecolor=DarkKlein,
    urlcolor=Black
]{hyperref}
\usepackage{xurl}            

\usepackage{graphicx}    
\usepackage{subcaption}  
\usepackage{wrapfig} 
 \usepackage{amssymb}
\usepackage{amsthm}
\usepackage[most]{tcolorbox}
\usepackage{amsmath} 
\usepackage{enumitem}
\usepackage[toc,page]{appendix}
\usepackage{multirow} 
\usepackage{thmtools, thm-restate}
\usepackage{cleveref} 
\usepackage{marvosym}
\usepackage{booktabs}
\usepackage{multirow}
\usepackage{makecell}
\usepackage{tabularx}
\usepackage{xfrac}

\newtcolorbox{block}[1][]{
    colback={rgb,255:red,248; green,249; blue,250}, 
    colframe={rgb,255:red,248; green,249; blue,250}, 
    boxrule=0pt,
    fonttitle=\bfseries,
    boxsep=2pt,
    left=2pt,
    right=2pt,
    #1
}

\definecolor{accentline}{HTML}{595959}
\definecolor{boxbg}{HTML}{fbfbfd}
\newtcolorbox{formalbox}[1][]{%
  enhanced,
  breakable,
  pad at break=0pt,                
  sharp corners,
  boxrule=0pt,
  frame hidden,                    
  borderline west={1.75pt}{0pt}{accentline}, 
  colback=boxbg,                   
  left=7pt,
  right=4pt,
  top=3pt,
  bottom=3pt,
  before skip=6pt,
  after skip=6pt,
  #1                               
}

\crefname{section}{\S}{\S\S}
\crefname{equation}{Eq.}{Eqs.}
\crefname{appendix}{Appx.}{Appxs.}
\crefname{thm}{Thm.}{Thms.}
\crefname{cor}{Cor.}{Cors.}
\crefname{prop}{Prop.}{Props.}
\crefname{asm}{Asm.}{Asms.}
\crefname{defn}{Defn.}{Defns.}
\crefname{lemma}{Lem.}{Lems.}
\crefname{exm}{Ex.}{Exs.}
\crefname{clar}{Clar.}{Clars.}
\crefname{figure}{Fig.}{Figs.}
\crefname{table}{Tab.}{Tabs.}

\declaretheorem[name=Remark,style=remark,numberwithin=section]{remark}

\newcommand{\E}{\mathop{\mathbb{E}}}
\newcommand{\sx}[1]{^{\scriptscriptstyle #1}}

\usepackage{tikz} 
\DeclareRobustCommand{\circled}[1]{%
  \tikz[baseline=(char.base)]\node[draw, shape=circle, inner sep=0.5pt](char){\small #1};%
}

\title{On the Value of Cross-Modal Misalignment in Multimodal Representation Learning}

\author{
  Yichao Cai\thanks{Equal contribution. $^\dag$Correspondence to: \texttt{yuhang.liu01@adelaide.edu.au}.}\qquad
  Yuhang Liu\footnotemark[1]\;\,\footnotemark[2]\qquad
  Erdun Gao\qquad
  Tianjiao Jiang  \\
  \textbf{Zhen Zhang\qquad Anton van den Hengel\qquad Javen Qinfeng Shi}\\
  Australian Institute for Machine Learning\\
  The University of Adelaide, SA 5000, Australia
}

\begin{document}

\maketitle
\addtocontents{toc}{\protect\setcounter{tocdepth}{-1}}

\input{Pages/0_abstract}
\input{Pages/1_introduction}
\input{Pages/2_prelinimaries}
\input{Pages/3_problem}
\input{Pages/4_theory}

\input{Pages/5_experiments}

\input{Pages/7_conclusion}

\begin{ack}
The authors gratefully acknowledge support from the Responsible AI Research (RAIR) Centre (J. Q. Shi and Z. Zhang), the Centre for Augmented Reasoning (CAR) (Y. Liu and E. Gao), and the Commonwealth Bank of Australia through the CommonBank Centre for Foundational AI Research (A. van den Hengel). This support was essential to the completion of the research presented in this publication. The authors also thank the anonymous reviewers for their constructive feedback.
\end{ack}


{
\small
\bibliographystyle{abbrv}
\bibliography{references}
}

\clearpage
\renewcommand{\contentsname}{\textbf{Content}}

\clearpage
\input{Pages/appendix}


\end{document}

%% file: Pages/0_abstract.tex
\begin{abstract}

Multimodal representation learning, exemplified by multimodal contrastive learning (MMCL) using image-text pairs, aims to learn powerful representations by aligning cues across modalities. This approach relies on the core assumption that the exemplar image-text pairs constitute two representations of an identical concept.  However, recent research has revealed that real-world datasets often exhibit \emph{cross-modal misalignment}. There are two distinct viewpoints on how to address this issue: one suggests mitigating the misalignment, and the other leveraging it. We seek here to reconcile these seemingly opposing perspectives and to provide a practical guide for practitioners. Using latent variable models, we thus formalize cross-modal misalignment by introducing two specific mechanisms: \emph{selection bias}, where some semantic variables are absent in the text, and \emph{perturbation bias}, where semantic variables are altered---both leading to misalignment in data pairs. Our theoretical analysis shows that, under suitable assumptions, population global minimizers of an alignment--entropy surrogate yield representations that depend only on the semantic variables unaffected by selection and perturbation biases and locally identify this subset almost surely. Extensive empirical studies on synthetic data and real image-text datasets further examine how cross-modal misalignment shapes semantic predictability and downstream performance, shedding light on its nuanced impact on multimodal representation learning.\footnote{Code available at \url{https://github.com/YichaoCai1/crossmodal_misalignment}}
\end{abstract} 

%% file: Pages/1_introduction.tex
\section{Introduction}

Modern multimodal learning has achieved remarkable success in jointly modeling information from heterogeneous sources such as vision, language, and audio. Multimodal contrastive learning (MMCL) on paired data has emerged as a dominant strategy for aligning modalities \citep{radford2021learning,jia2021scaling,wu2022scaling}. This approach is exemplified by vision-and-language models like CLIP, which attempt to identify a common representation space that maximizes the similarity of real image-text pairs while minimizing that of incorrect pairs \citep{radford2021learning}. 
One of the assumptions underpinning this approach is that the training pairs are aligned across modalities, meaning that they convey the same semantic information  \citep{NEURIPS2023_c89f0984,liu2024revealing}. This assumption, though convenient, is often violated in real-world scenarios, where multimodal data is inherently noisy or imprecisely paired \citep{miech2019howto100m,nakada2023understanding}. More critically, text taken from image captions typically only partially describes the paired image content, and often contains descriptive elements that are irrelevant or misleading, a situation referred to as \emph{cross-modal misalignment}. For example, in a large-scale video-text dataset, over 50\% of the purportedly aligned clip-caption pairs were found to be misaligned \citep{miech2019howto100m}. It is worth noting that the impact of cross-modal misalignment is mainly concentrated in the text modality, as text captions, often derived from images, are more susceptible to semantic incompleteness and interpretive variability than the typically consistent visual semantics.

The misalignment\footnote{For the sake of brevity, we use``cross-modal misalignment'' and ``misalignment'' interchangeably throughout the paper, as the context clearly indicates that we are referring to misalignment between different modalities.} discussed above has led to two seemingly opposing viewpoints. On one hand, misalignment is viewed as a form of disruption that should be mitigated \citep{liu2024alignrec,sun2024aligning,wu2022scaling,chen2024weakly,tao2024alignmif,kim2024expediting,xuan2022multimodal,chartsias2020disentangle,lv2021progressive}. For example, cross-modal misalignment can result in ``hallucination'' in multimodal models \citep{sun2024aligning,leng2024mitigating,xie2024v}. It also provides weak, noisy, and even misleading supervision for multimodal pre-training \citep{wu2022scaling,zhang2024rethinking}. On the other hand, an alternative viewpoint suggests that multimodal representations may actually benefit from cross-modal misalignment \citep{wu2023prompt,nakada2023understanding, cai2025clap,Kim_2023_ICCV}. For instance, fine-tuning the representations learned by CLIP through random text augmentation, which deliberately introduces misalignment in style-related information, can lead to more robust representations for zero-shot learning, few-shot learning, and even adversarial attacks \citep{cai2025clap}. This contrast raises a crucial question: 

\begin{center}
\emph{How can we theoretically reconcile these two opposing views on misalignment,\\ and, more importantly, determine which should guide practical applications?}
\end{center}

In light of this, we offer a theoretical perspective that not only facilitates the understanding of misalignment but also provides guidance for real-world applications. Specifically, we formulate the problem using a latent variable model (LVM), which depicts the underlying generative process of image-text data with misalignment, as shown in \Cref{fig:image_text_pairing}. In it, the latent space consists of latent semantic variables representing factors common to both modalities (e.g., object shapes and colors), along with modality-specific subspaces that capture unique variations in images and text. To make the LVM more adaptable to diverse real-world scenarios, we allow for an arbitrary causal structure among the latent semantic variables, providing flexibility in multimodal contexts. To model misalignment, we introduce two mechanisms: selection bias and perturbation bias. Both act on latent semantic information but differ in their effects. Selection bias determines which semantic information is preserved in the text. For example, when describing an object, the text might preserve information about its color (``black'') but omit its texture details or shape. On the other hand, perturbation bias introduces errors, such as misannotating ``black'' (correct color) to ``red'' (incorrect color). Moreover, given their heterogeneity, we model the two modalities with separate generative processes.

Building on the proposed LVM (i.e., the \emph{generative model}), we present a theoretical identifiability analysis within the MMCL framework (i.e., the \textit{inference model}). We show, under the stated assumptions, that the subset of semantic variables unaffected by selection and perturbation biases is locally block-identifiable almost surely (\Cref{def:blockid}) at population global minima of the alignment--entropy surrogate. Specifically, the learned representations depend only on this unaffected subset, and the corresponding maps from this subset to the representations are locally invertible around almost every realization. The remaining misaligned semantic variables affected by either bias are functionally excluded from the learned representations, irrespective of the latent causal structure among the semantic variables. This result provides a unified perspective on the seemingly opposing views discussed above. While misalignment can be problematic in tasks that rely on fully preserving semantic information to maximize downstream utility, it may paradoxically become beneficial in scenarios where robustness to distribution shifts is desired. In such cases, cross-modal misalignment naturally acts as a regularizer, encouraging models to focus on stable, invariant semantic factors shared across modalities.

\textbf{Contributions.} Our main contributions are presented below, with an in-depth discussion of related work available in \Cref{sec:app_realtedwork}.
\begin{itemize}[left=1.8em]
    \item We propose an LVM for multimodal data generation that explicitly characterizes cross-modal misalignment through two mechanisms: selection bias and perturbation bias (\Cref{sec:problemformulate}). 
    \item We establish local block recovery of the semantic variables unaffected by these biases at population global optima of an alignment--max-entropy surrogate, without imposing a fixed causal ordering among semantic variables (\Cref{sec:identifiability}).
    \item We extend this result to two practical scenarios, tasks requiring common representations and those targeting invariant representations, offering actionable insights into how misalignment should inform real-world applications (\Cref{sec:practical_insights}). 
    \item We empirically validate our theoretical findings through extensive experiments on both real-world and synthetic image-text datasets under various cross-modal misalignment conditions. Additionally, we present a case study on pretrained CLIP models (\Cref{sec:experiments}).
\end{itemize}

%% file: Pages/2_prelinimaries.tex
\section{Preliminaries: Multimodal Contrastive Learning}

Multimodal contrastive learning (MMCL)~\citep{radford2021learning, jia2021scaling, zhang2022contrastive} aims to learn joint representations by aligning paired samples from different modalities, i.e., \(\mathbf{t} \in \mathcal{T}\) for text and \(\mathbf{x} \in \mathcal{X}\) for images, while pushing apart unpaired (negative) samples. In practice, MMCL typically employs two modality-specific encoders, i.e., \(f_t(\mathbf{t})\) for text and \(f_x(\mathbf{x})\) for images which project observed paired data into a shared representation space. The learning objective is generally formulated as the following contrastive loss:
\begin{equation}
\label{eq:loss_exp}
    \mathcal{L}_{\text{\tiny MMCL}}(f_x, f_t) = -\frac{1}{2K}\left[
    \sum\limits_{i=1}^K \log \frac{e^{\kappa\left(f_x(\mathbf{x}_i), f_t(\mathbf{t}_i)\right)/\tau}}{\sum\nolimits_{j=1}^K e^{\kappa\left( f_x(\mathbf{x}_i), f_t(\mathbf{t}_j)\right)/\tau}}
    +
    \sum\limits_{i=1}^K \log \frac{e^{\kappa\left( f_x(\mathbf{x}_i), f_t(\mathbf{t}_i)\right)/\tau}}{\sum\nolimits_{j=1}^K e^{\kappa\left( f_x(\mathbf{x}_j), f_t(\mathbf{t}_i)\right)/\tau}}
    \right],
\end{equation}
where \(\{\mathbf{x}_i, \mathbf{t}_i\}_{i=1}^K\) are sampled paired data, \(K\) denotes the number of training pairs, \(\tau\) is a temperature hyperparameter controlling the concentration of the similarity distribution, and \(\kappa(\cdot, \cdot)\) denotes a similarity measure. Motivated by the alignment and uniformity perspective on contrastive learning \citep{wang2020understanding}, and following prior identifiability analyses \citep{lyu2022understanding,daunhawer2022identifiability,yao2023multi}, we study the following population alignment--entropy analytical surrogate:
\begin{equation}
\label{eq:loss_thm}
    \mathcal{L}_{\text{\tiny SymAlignMaxEnt}}(f_x, f_t) = \E\nolimits_{(\mathbf{x}, \mathbf{t}) \sim p_{\mathbf{x}, \mathbf{t}}} \Bigl[\|f_x(\mathbf{x}) - f_t(\mathbf{t})\|_2 \Bigr] - \frac{1}{2} \Bigl( H\bigl(f_x(\mathbf{x})\bigr) + H\bigl(f_t(\mathbf{t})\bigr) \Bigr),
\end{equation}
Here, $H(\cdot)$ denotes differential entropy. The first term aligns paired representations, while the second maximizes the entropy of their marginal distributions. We adopt this objective for theoretical analysis, replacing the log kernel averages that arise in the normalized ($\text{}-\log K$), large-batch limit of \Cref{eq:loss_exp} with the marginal differential entropies in \Cref{eq:loss_thm}. Since these two objectives do not necessarily share the same minimizers, our formal guarantees hold for the population global minimizers of $\mathcal{L}_{\text{\tiny SymAlignMaxEnt}}$, while we evaluate practical loss empirically. See \Cref{sec:objective_scope} for further details.


Prior work uses alignment--entropy objectives to study how shared latent factors are recovered in multimodal and partially observed multi-view settings \citep{daunhawer2022identifiability,yao2023multi}. We instead study how semantic omission and perturbation in text supervision shape the population optima of this surrogate. To formalize these effects, we introduce a latent variable model (\Cref{sec:problemformulate}) that captures selection and perturbation biases without imposing a fixed causal ordering among semantic variables. Our analysis (\Cref{sec:theory}) characterizes which semantic coordinates are functionally excluded and proves almost-sure local recovery of the retained block. This connects the modeled misalignment mechanisms to the distinction between preserving semantic coverage and inducing invariance to selected attributes.

%% file: Pages/3_problem.tex
\section{Problem Formulation via a Generative Perspective}
\label{sec:problemformulate}

In this section, we introduce a novel latent variable model (LVM) to formalize the underlying generative processes of image-text data, explicitly characterizing cross-modal misalignment (\Cref{sec:lvm}). 
Building on this LVM, we introduce the technical assumptions underlying image-text pairs under misalignment for subsequent analysis (\Cref{sec:assumptions}). See \Cref{sec:notation} for a complete notation table.

\subsection{A Latent Variable Model Characterizing Cross-Modal Misalignment}
\label{sec:lvm}

\Cref{fig:image_text_pairing} illustrates the proposed LVM. In the following, we provide a detailed explanation of the model from three aspects: the latent space, and the image and text generation processes.

\paragraph{Latent space.} We partition the entire latent space \(\mathcal{Z}\) into three simply connected, open subspaces, i.e., \( \mathcal{Z} = \mathcal{S} \times \mathcal{M}_x \times \mathcal{M}_t\), where each defines the support of a distinct group of latent variables. We denote the latent variables in \(\mathcal{Z}\) as \(\mathbf{z} = (\mathbf{s}, \mathbf{m}_x, \mathbf{m}_t)\), where \(\mathbf{s}\), \(\mathbf{m}_x\), and \(\mathbf{m}_t\) lie in \(\mathcal{S}\), \(\mathcal{M}_x\), and \(\mathcal{M}_t\), respectively. Below, we describe the characteristics of the latent variables \(\mathbf{s}\), \(\mathbf{m}_x\), and \(\mathbf{m}_t\):

\begin{itemize}[left=1.8em]
    \item \(\mathbf{s} \in \mathcal{S} \subseteq \mathbb{R}^{n_s}\) (\emph{Semantic variables}): Latent variables capturing the semantic content of the data, i.e., information that is interpretable or describable through human knowledge (e.g., object shape, color). We denote the index set of semantic variables as \(\mathbb{I}_\mathbf{s} := [n_s]\)\footnote{Throughout this paper, we use the notation \([d]\) to represent the set \(\{1, \dots, d\}\) for any integer \(d > 1\).} for future reference.
    \item \(\mathbf{m}_x \in \mathcal{M}_x\)\footnote{For simplicity of notation, we omit the dimensions of certain variables throughout this work, such as \(\mathbf{m}_x\).} (\emph{Image-specific variables}): Latent variables capturing non-semantic, image-specific factors that are independent of semantic variables \(\mathbf{s}\) (e.g., camera noise or background artifacts). 
    \item \(\mathbf{m}_t \in \mathcal{M}_t\) (\emph{Text-specific variables}): Latent variables capturing non-semantic, text-specific factors, independent of both semantic variables \(\mathbf{s}\) and image-specific variables \(\mathbf{m}_x\) (e.g., writing style or tone).
\end{itemize}

A key advantage of the proposed latent space structure is its flexibility and applicability to real-world scenarios. To this end, we depart from prior works that impose somewhat restrictive assumptions on the latent structure. Specifically, unlike approaches that enforce certain fixed graphical structures among semantic variables (e.g., assuming content causally determines style)~\citep{daunhawer2022identifiability}, or methods based on nonlinear ICA that assume complete (or conditional on an auxiliary variable) independence among latent variables~\citep{khemakhem2020variational, sorrenson2020disentanglement}, we allow for arbitrary dependency structures among semantic variables \(\mathbf{s}\), since the true latent graph structure among semantic variables is often unknown in practice.

\begin{figure}[t]
    \centering
    \includegraphics[width=0.5\textwidth]{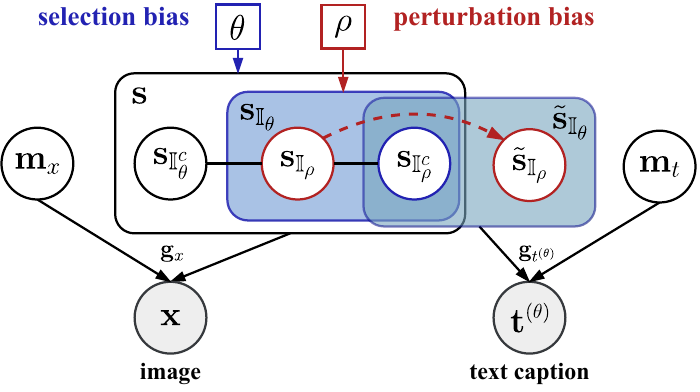}
    \hfill 
    \includegraphics[width=0.4\textwidth]{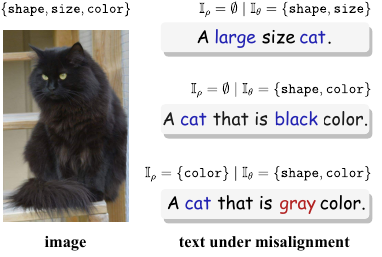}
    \caption{Illustration of the proposed latent variable model (left), with misalignment across modalities modeled via selection and perturbation bias.
    Image $\mathbf{x}$ is generated from semantic variables $\mathbf{s}$ and image-specific variables $\mathbf{m}_x$ via the generative process $g_x$. The corresponding text $\mathbf{t}\sx{(\theta)}$ is generated by \(g_{t\sx{(\theta)}}\) which acts on a biased subset of semantic variables \(\tilde{\mathbf{s}}_{\mathbb{I}_\theta}=(\mathbf{s}_{\mathbb{I}^c_\rho}, \tilde{\mathbf{s}}_{\mathbb{I}_\rho})\),  influenced by {selection bias} $\theta$ and {perturbation bias} $\rho$, along with text-specific variables $\mathbf{m}_t$. Selection bias omits $\mathbf{s}_{\mathbb{I}^c_\theta}$, while perturbation bias marks a subset $\mathbb{I}_\rho \subseteq \mathbb{I}_\theta$ whose components may be randomly replaced to form $\tilde{\mathbf{s}}_{\mathbb{I}_\rho}$. Example image-text pairs (right) illustrate the misalignment induced by these two biases.
    }
    \label{fig:image_text_pairing}
    \vspace{-1em}
\end{figure}

\paragraph{Image generation.}
An image \(\mathbf{x} \in \mathcal{X}\) is generated from latent variables \(\mathbf{z}_x = (\mathbf{s}, \mathbf{m}_x)\) via a diffeomorphism (i.e., a bijection with smooth inverse function) \(g_x: \mathcal{S}\times\mathcal{M}_x \to \mathcal{X}\):
\begin{equation}
\label{eq:img_gen}
\mathbf{s} \sim p_\mathbf{s},\quad \mathbf{m}_x \sim p_{\mathbf{m}_x}, \quad  \mathbf{z}_x = (\mathbf{s}, \mathbf{m}_x),
\quad 
\mathbf{x} = g_x(\mathbf{z}_x).
\end{equation}
Here, \(p_\mathbf{s}\) and \(p_{\mathbf{m}_x}\) are prior distributions over subspaces \(\mathcal{S}\) and \(\mathcal{M}_x\), respectively; \(\mathcal{X}\) defines an observation space that contains all generated images. This generative process formalizes that images fully encapsulate semantics, reflecting the fact that they are both informative and semantically rich.

\paragraph{Text generation.}
Unlike other multimodal data (e.g., camera-LiDAR \citep{NEURIPS2022_43d2b7fb}) acquired via sensors, given an image, its corresponding text inherently exhibits flexibility in semantic richness and may also include distortions introduced by humans or captioning models \citep{li2022blip, xu2024altogether}, leading to cross-modal misalignment. Here, we formalize the root of such misalignment by introducing two types of biases: {selection bias} and {perturbation bias}. Accordingly, before formulating the text generation process, we first provide definitions of these two types of biases.

\begin{restatable}[Selection bias $\theta$]{defn}{selection}
\label{defn:selection}
Let $\mathcal{P}^+(\mathbb{I}_\mathbf{s})$ denote the set of all non-empty subsets\footnote{Without loss of generality, we fix a graded lexicographic order (see \cref{defn:order}) over \(\mathcal{P}^+(\mathbb{I}_\mathbf{s})\). See \cref{clar:fixedorder} for a clarification for this choice of order.} of the index set $\mathbb{I}_\mathbf{s}$, defined as \(\mathcal{P}^+(\mathbb{I}_\mathbf{s}) := \mathcal{P}(\mathbb{I}_\mathbf{s}) \setminus \{\emptyset\}\), where \(\mathcal{P}(\cdot)\) denotes the power set. The \emph{selection bias} \(\theta\) is defined as an integer index in the range \( [2^{n_s} - 1]
\), corresponding to a specific non-empty semantic subset \(\mathbb{I}_\theta \in \mathcal{P}^+(\mathbb{I}_\mathbf{s})\). The complement \(\mathbb{I}_\theta^c = \mathbb{I}_\mathbf{s} \setminus \mathbb{I}_\theta\) denotes the omitted semantic subset.
\end{restatable}

Note that each \(\theta\) uniquely determines a non-empty semantic subset \(\mathbb{I}_\theta \in \mathcal{P}^+(\mathbb{I}_\mathbf{s})\), which defines the semantic information to be expressed in the generated text. It also specifies a text generation mapping \(g_{t\sx{(\theta)}} : \mathcal{S}_{\mathbb{I}_\theta} \times \mathcal{M}_t \to \mathcal{T}\sx{(\theta)}\), selected from a class of diffeomorphisms \(\mathcal{G}_{t}\). Here, \(\mathcal{T}\sx{(\theta)}\) denotes an observation space that contains all the observed text \(\mathbf{t}\sx{(\theta)}\) under selection bias \(\theta\). 

\begin{restatable}[Perturbation bias \(\rho\)]{defn}{perturbation}
\label{defn:perturbation}
Let \(\mathcal{P}_{\text{prop}}(\mathbb{I}_\theta) := \mathcal{P}(\mathbb{I}_\theta) \setminus \{\mathbb{I}_\theta\}\) denote the set of all proper subsets\footnote{Again, we fix a graded lexicographic order over the proper subsets in \(\mathcal{P}_{\text{prop}}(\mathbb{I}_\theta)\) throughout the paper.} of the selected index subset \(\mathbb{I}_\theta\). The \emph{perturbation bias} \(\rho\) is defined as an integer index in the range
\(
\rho \in [2^{|\mathbb{I}_\theta|} - 1],
\)
corresponding to a unique subset \(\mathbb{I}_\rho \in \mathcal{P}_{\text{prop}}(\mathbb{I}_\theta)\) subject to perturbation. The complement \(\mathbb{I}_\rho^c := \mathbb{I}_\theta \setminus \mathbb{I}_\rho\) denotes the semantic subset that remains unperturbed across modalities.
\end{restatable}

\begin{formalbox}
\begin{restatable}{exm}{bias_exm}
\label{exm:bias_exm}
Let the full semantic index set be \(\mathbb{I}_\mathbf{s} = \)\{\texttt{shape}, \texttt{size}, \texttt{color}\}, so that the set of non-empty semantic subsets is \(
\mathcal{P}^+(\mathbb{I}_\mathbf{s}) = \{
\{\texttt{shape}\}, 
\{\texttt{size}\}, 
\dots,
\{\texttt{shape}, \texttt{size}, \texttt{color}\}
\}
\). 
Then, a selection bias \(\theta = 5\) corresponds to \(\mathbb{I}_\theta = \{\texttt{shape}, \texttt{color}\}\), with the omitted semantics given by \(\mathbb{I}_\theta^c = \{\texttt{size}\}\). The corresponding text-generation mapping \(g_{t\sx{(\theta)}}\) uses only the selected semantic variables in \(\mathbb{I}_\theta\) to generate text \(\mathbf{t}\sx{(\theta)}\).
The set of proper subsets of \(\mathbb{I}_\theta\) is
\(
\mathcal{P}_{\text{prop}}(\mathbb{I}_\theta) = \{
\emptyset, 
\{\texttt{shape}\}, 
\{\texttt{color}\}
\}
\).
Then, a perturbation bias \(\rho = 3\) corresponds to the subset \(\mathbb{I}_\rho = \{\texttt{color}\}\), and its complement within \(\mathbb{I}_\theta\) is \(\mathbb{I}_\rho^c = \{\texttt{shape}\}\). Under the combined biases \(\theta = 5\) and \(\rho = 3\), only \texttt{shape} is unbiasedly preserved, while \texttt{color} is subject to perturbation. A resulting text might be \emph{"A cat that is red color"}, even though the image shows a large-sized black cat.
\end{restatable}
\end{formalbox}

Building upon the previously defined selection bias \(\theta\) and perturbation bias \(\rho\), we now formalize the text generation process, explicitly characterizing the misalignment induced by these two biases. Consider an image \(\mathbf{x}\) generated by \Cref{eq:img_gen}, with associated semantic variables \(\mathbf{s} = (\mathbf{s}_{\mathbb{I}_\theta}, \mathbf{s}_{\mathbb{I}_\theta^c})\), where the index set \(\mathbb{I}_\theta\) is determined by \(\theta\). For the corresponding text \(\mathbf{t}\sx{(\theta)}\), we define the latent variables as \(\mathbf{z}_{t\sx{(\theta)}} = (\tilde{\mathbf{s}}_{\mathbb{I}_\theta}, \mathbf{m}_t)\), where \(\tilde{\mathbf{s}}_{\mathbb{I}_\theta}\) represents the perturbed semantic variables under perturbation bias \(\rho\), and \(\mathbf{m}_t\) denotes the text-specific latent variables. The text generation process is then formalized as:
\begin{gather}
\tilde{\mathbf{s}}_{\mathbb{I}_\theta} \sim p_{\tilde{\mathbf{s}}_{\mathbb{I}_\theta} \mid \mathbf{s}, \theta, \rho},\quad
\mathbf{m}_t \sim p_{\mathbf{m}_t},\quad
\mathbf{z}_{t\sx{(\theta)}} = (\tilde{\mathbf{s}}_{\mathbb{I}_\theta}, \mathbf{m}_t),\quad  
\mathbf{t}\sx{(\theta)}  
= g_{t\sx{(\theta)}}(\mathbf{z}_{t\sx{(\theta)}}), \label{eq:caption_gen}
\end{gather}
where \(p_{\mathbf{m}_t}\) denotes the prior distribution over the latent subspace \(\mathcal{M}_t\), and \(g_{t\sx{(\theta)}}\) is the diffeomorphic mapping specified by the selection bias \(\theta\). The conditional \(p_{\tilde{\mathbf{s}}_{\mathbb{I}_\theta} \mid \mathbf{s}, \theta, \rho}\) explicitly characterizes misalignment in text generation arises through perturbations within selected semantic dimensions.

\subsection{Model Assumptions for Theoretical Analysis}
\label{sec:assumptions}

Based on the proposed LVM, we now present the assumptions underpinning our theoretical analysis:

\begin{restatable}[Continuous positive densities]{asm}{continuous} 
\label{asm:continuous}
The latent variables \(\mathbf{s}\), \(\mathbf{m}_x\), and \(\mathbf{m}_t\) are continuous and admit strictly positive densities, i.e., \(p_{\mathbf{s}} > 0\), \(p_{\mathbf{m}_x} > 0\), and \(p_{\mathbf{m}_t} > 0\), a.e. on their respective supports \(\mathcal{S}\), \(\mathcal{M}_x\), and \(\mathcal{M}_t\).
\end{restatable}

\begin{restatable}[Random perturbations]{asm}{smodify}
\label{asm:pairing}
Given a selection bias $\theta$ and a perturbation bias $\rho$, the conditional law $p_{\tilde{\mathbf s}_{\mathbb I_\theta}\mid\mathbf s,\theta,\rho}$ is generated by sampling $A\sim p_A$ independently of $(\mathbf s,\mathbf m_x,\mathbf m_t)$ and then sampling according to
\begin{equation}
\label{eq:perturbation}
\begin{aligned}
p_{\tilde{\mathbf s}_{\mathbb I_\theta}
\mid\mathbf s,A,\theta,\rho}
\bigl(\tilde{\mathbf s}_{\mathbb I_\theta}
\mid\mathbf s,A,\theta,\rho\bigr)
 =
\delta\bigl(
\tilde{\mathbf s}_{\mathbb I_\theta\setminus A}
-\mathbf s_{\mathbb I_\theta\setminus A}
\bigr)
\,p_{\tilde{\mathbf s}_A\mid\mathbf s_A}
\bigl(\tilde{\mathbf s}_A\mid\mathbf s_A\bigr).
\end{aligned}
\end{equation}
The law $p_{\tilde{\mathbf s}_{\mathbb I_\theta} \mid\mathbf s,\theta,\rho}$ is obtained by averaging over $A\sim p_A$. Here: \emph{\textbf{(i)}} $p_A$ is a probability distribution on $\mathcal P(\mathbb I_\rho)$. For every $l\in\mathbb I_\rho$, there exists a subset $A\subseteq\mathbb I_\rho$ such that $l\in A$ and $p_A(A)>0$; \emph{\textbf{(ii)}} $\delta(\cdot)$ denotes the Dirac delta, enforcing that coordinates outside $A$ remain unchanged; \emph{\textbf{(iii)}} for every nonempty $A$ with $p_A(A)>0$, $p_{\tilde{\mathbf s}_A\mid\mathbf s_A}$ is a smooth, strictly positive conditional density on $\mathcal S_A\times\mathcal S_A$, where $\mathcal S_A$ is the domain of $\mathbf s_A$; when $A=\emptyset$, $\tilde{\mathbf s}_{\mathbb I_\theta} =\mathbf s_{\mathbb I_\theta}$; and \emph{\textbf{(iv)}} the resulting vector satisfies $\tilde{\mathbf s}_{\mathbb I_\theta} \in\mathcal S_{\mathbb I_\theta}$ a.s., so that the text-generating map is well defined. 
\end{restatable}

\begin{wrapfigure}{r}{0.38\textwidth}
\centering
    \vspace{-2.5em}
    \includegraphics[width=0.38\textwidth]{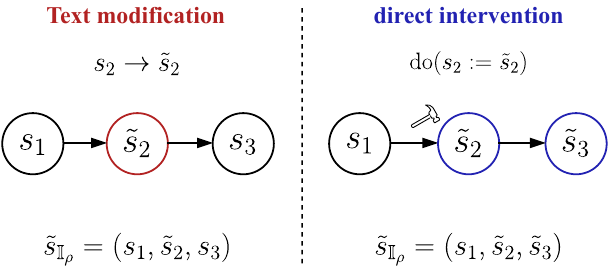}
    \caption{Textual modifications versus latent causal interventions.}
    \label{fig:textual_semantic_perturbation}
    \vspace{-1.5em}
\end{wrapfigure}

Note that \Cref{eq:perturbation} essentially implies that, in each \((\mathbf{x}, \mathbf{t}\sx{(\theta)})\), only a random subset \(A\) of \(\mathbb{I}_\rho\) undergoes perturbations, regardless of the underlying causal structure among latent semantic variables. The rationale is that latent semantic variables can only be modified indirectly by altering observations, rather than through direct intervention, unless the latent causal structure is fully identified. In the text modality, it occurs through the misassignment of certain content words to specific image semantics during the captioning process. Unlike a direct intervention, this cross-modal misalignment does not propagate to descendant semantic variables, as in \Cref{fig:textual_semantic_perturbation}.

%% file: Pages/4_theory.tex
\section{Local Identifiability under Cross-Modal Misalignment}
\label{sec:theory}


In this section, we analyze how cross-modal misalignment affects the recovery of latent semantic variables in the proposed LVM at population global minima of \Cref{eq:loss_thm} (\Cref{sec:identifiability}). Based on these results, we discuss implications for practical MMCL (\Cref{sec:practical_insights}). Detailed proofs are provided in \Cref{sec:proofs}. We adapt the notion of block-identifiability \citep{von2021self} to our setting by requiring local invertibility a.s.:

\begin{restatable}[Local block-identifiability]
{defn}{blockid}
\label{def:blockid}
Let $\mathbb{I}^*\subseteq\mathbb{I}_{\mathbf{s}}$ and $n=|\mathbb{I}^*|$. The semantic block $\mathbf{s}_{\mathbb{I}^*}$ is \emph{locally block-identified almost surely} by $f_x:\mathcal{X}\to\mathbb{R}^n$ and $f_t:\mathcal{T}\sx{(\theta)}\to\mathbb{R}^n$ if there exist smooth maps $h_x,h_t:\mathcal{S}_{\mathbb{I}^*}\to\mathbb{R}^n$ such that \(\hat{\mathbf{z}}_x :=f_x(\mathbf{x}) =h_x(\mathbf{s}_{\mathbb{I}^*})\), \(\hat{\mathbf{z}}_t:=f_t(\mathbf{t}\sx{(\theta)})=h_t(\mathbf{s}_{\mathbb{I}^*})\) a.s., and \(\det J(h_x)(\mathbf{s}_{\mathbb{I}^*})\ne0\), \(\det J(h_t)(\mathbf{s}_{\mathbb{I}^*})\ne0\) a.s. 
\end{restatable}

Equivalently, \Cref{def:blockid} means that, for almost every realization of the semantic block, each reduced map restricts to a diffeomorphism from an open neighborhood of that realization onto its image. This characterizes representations that depend only on the designated block and permit its local recovery around almost every latent realization. We discuss its relation to block-identifiability in \Cref{sec:recovery_scope}.

\subsection{Identifiability of Latent Semantic Variables under Cross-Modal Misalignment}
\label{sec:identifiability}
Building on the definition above, the formalization of the proposed LVM, and assumptions outlined in \Cref{sec:problemformulate}, we provide the following identifiability result:


\begin{restatable}[Local identifiability of latent semantic variables]
{thm}{idsemantics}
\label{thm:idsemantics}
Let $(\mathbf x,\mathbf t\sx{(\theta)})$ follow the data-generating process in \Cref{sec:problemformulate}, and suppose that \Cref{asm:continuous,asm:pairing} hold. Denote the retained semantic block by $\mathbf s_{\mathbb I_\rho^c}$, with dimension $n=|\mathbb I_\rho^c| =|\mathbb I_\theta|-|\mathbb I_\rho|$. Assume that the sequential conditional CDFs of $\mathbf s_{\mathbb I_\rho^c}$ used in \Cref{lem:globalminimum} admit jointly smooth, $(0,1)$-valued versions on their corresponding open projected domains. Let $f_x:\mathcal X\to(0,1)^n$ and $f_t:\mathcal T\sx{(\theta)}\to(0,1)^n$ be sufficiently flexible, smooth functions. Then every such pair $(f_x,f_t)$ attaining the global minimum of $\mathcal L_{\text{\tiny SymAlignMaxEnt}}$ in \Cref{eq:loss_thm} locally block-identifies $\mathbf s_{\mathbb I_\rho^c}$ a.s. in the sense of \Cref{def:blockid}.
\end{restatable}

\begin{formalbox}
\begin{remark}
\Cref{thm:idsemantics} concerns global minima of the population alignment--entropy objective. Under its assumptions, the learned representations depend only on \(\mathbf s_{\mathbb I_\rho^c}\) and are locally invertible with respect to this block almost surely. This excludes local dimensional collapse of the retained block, but distinct, separated latent states may still have the same representation. Thus unique global recovery is not guaranteed. Exclusion of \(\mathbf s_{\mathbb I_\rho}\), \(\mathbf s_{\mathbb I_\theta^c}\), \(\mathbf m_x\) and \(\mathbf m_t\) means the absence of additional functional dependence. Misaligned semantic variables may remain statistically predictable through their dependence on the retained block. A detailed proof is provided in \cref{sec:proofthm}.
\end{remark}
\end{formalbox}

\subsection{Insights into Cross-Modal Misalignment for Practice}
\label{sec:practical_insights}
We next consider two implications of the result within the stated model and population-surrogate setting: \textbf{(i)} pretraining with large-scale data and \textbf{(ii)} invariant representation learning. The former aims to capture comprehensive semantic information to support a wide range of downstream tasks, while the latter focuses on learning robust representations for out-of-distribution (OOD) generalization. In what follows, we present corollaries and insights for both scenarios.


\begin{restatable}[Local identifiability of full latent semantic variables]{cor}{allsemantics}
\label{cor:allsemantics}
Let $\theta=2^{n_s}-1$ and $\rho=1$, so that $\mathbb I_\theta=\mathbb I_{\mathbf s}$ and $\mathbb I_\rho=\emptyset$. Suppose that the conditions of \Cref{thm:idsemantics} hold for this choice of $(\theta,\rho)$. Then every pair of smooth encoders $f_x:\mathcal X\to(0,1)^{n_s}$ and $f_t:\mathcal T\sx{(\theta)}\to(0,1)^{n_s}$ attaining the global minimum of $\mathcal L_{\text{\tiny SymAlignMaxEnt}}$ in \Cref{eq:loss_thm} locally block-identifies $\mathbf s$ a.s. in the sense of \Cref{def:blockid}.
\end{restatable}

\begin{formalbox}
\textbf{Insight 4.1} (MMCL pretraining on large-scale data)\textbf{.} Large-scale multimodal datasets (e.g., COCO \citep{lin2014microsoft}, Conceptual Captions \citep{sharma-etal-2018-conceptual}, LAION-5B \citep{schuhmann2022laion}) often exhibit varying caption quality. Under the stated assumptions, our analysis shows that population global minimizers of the alignment--entropy surrogate depend only on the selected, unperturbed semantic block. Increasing sample size does not alter this population characterization when the pairing distribution is fixed. For foundation models intended to support diverse downstream tasks, these results motivate attention to semantic coverage and caption consistency alongside dataset size. In particular, \Cref{cor:allsemantics} provides sufficient conditions for almost-sure local recovery of the full semantic vector under complete semantic selection and no perturbation. These findings motivate improved caption control \citep{li2022blip, fang2015captions, ding2023image} as a practical direction for addressing gaps in semantic coverage.
\end{formalbox}


\begin{restatable}[Local identifiability of invariant semantic variables]{cor}{invsemantics}
\label{cor:invsemantics}
Suppose that the conditions of \Cref{thm:idsemantics} hold for the training data-generating process. In an OOD setting, let $\mathbb I_{inv}\subseteq\mathbb I_{\mathbf s}$ index a semantic block whose joint distribution remains invariant between training and testing environments, and let $\mathbb I_{var}=\mathbb I_{\mathbf s}\setminus\mathbb I_{inv}$. If
\(
\mathbb I_{var}=\mathbb I_\theta^c\cup\mathbb I_\rho
\),  then every pair of smooth encoders $f_x:\mathcal X\to(0,1)^{|\mathbb I_{inv}|}$ and $f_t:\mathcal T\sx{(\theta)}\to(0,1)^{|\mathbb I_{inv}|}$ attaining the global minimum of $\mathcal L_{\text{\tiny SymAlignMaxEnt}}$ in \Cref{eq:loss_thm} locally block-identifies $\mathbf s_{\mathbb I_{inv}}$ a.s. in the sense of \Cref{def:blockid} under the training distribution.
\end{restatable}

\begin{formalbox}
\textbf{Insight 4.2} (Invariant representation learning)\textbf{.} In tasks requiring robust OOD performance (e.g., domain generalization) \citep{rojas2018invariant, arjovsky2019invariant}, semantic variables whose distributions shift across environments can undermine generalization. Under the conditions of \Cref{cor:invsemantics}, population global minimizers of the alignment--entropy surrogate depend only on the invariant semantic block. This motivates controlled cross-modal misalignment as a possible approach to invariant representation learning \citep{peters2016causal, krueger2021out, eastwood2022probable}. Its benefits for OOD prediction depend on the task-relevant information retained by the learned representation and the stability of its relationship to the prediction target. Auditing and curating text offers a practical way to investigate this approach, since captions can be inspected and edited to specify which image attributes are described \cite{cai2025clap}.
\end{formalbox}

%% file: Pages/5_experiments.tex
\section{Experiments}
\label{sec:experiments}

We examine how cross-modal misalignment affects semantic predictability and downstream performance through numerical simulations (\Cref{sec:numeric_exp}), a real-world image-text dataset with independent semantic variables (\Cref{sec:MPIReal3D}), a synthetic dataset with dependent semantic variables (\Cref{sec:Causal3DIdent}), and a case study with OpenCLIP models (\Cref{sec:openclip}).

\subsection{Numerical Simulation}
\label{sec:numeric_exp}

\paragraph{Experimental setup.}
We simulate data following the generative process in \Cref{sec:problemformulate}. Semantic variables \(\mathbf{s}\sim \mathcal{N}(0, \Sigma_\mathbf{s})\) (dim. 10) and modality-specific variables \(\mathbf{m}_x\sim \mathcal{N}(0, \Sigma_{\mathbf{m}_x}), \mathbf{m}_t\sim \mathcal{N}(0, \Sigma_{\mathbf{m}_t})\) (dim. 5) encode statistical dependencies through their respective covariance matrices \(\Sigma_{(\cdot)}\). To study cross-modal misalignment, we independently vary: \textbf{(i)} Selection bias: controlling \(\mathbb{I}_\theta = \{1\}, \dots, [10]\), and \textbf{(ii)} Perturbation bias: defining \(\mathbb{I}_\rho = \emptyset, \dots, [9]\), with Gaussian noise added to each \(i\in \mathbb{I}_\rho\) with probability 0.75. In isolation, we fix \(\mathbb{I}_\rho = \emptyset\) when varying \(\mathbb{I}_\theta\), and \(\mathbb{I}_\theta=[10]\) when varying \(\mathbb{I}_\rho\). Image and text data are generated using randomly initialized invertible MLPs \(g_x\) and \(g_{t\sx{\theta}}\). We train two modality-specific MLP encoders for approximately $10^5$ steps using a symmetrized finite-batch contrastive loss with squared Euclidean distances, cross-modal positive pairs, and independently sampled same-modality negatives, with embedding dimension matching the number of unbiased semantic components. See \Cref{sec:num_set_details} for further details. 

We conduct two experiments. \textbf{(i)} Semantic identification: A lightweight MLP regressor is trained to recover each ground-truth semantic component from MMCL learned representations; we report mean \(R^2\) over three seeds. \textbf{(ii)} Downstream performance: We evaluate learned representations on four tasks with targets \(y_1, y_2, y_3, y_4\), defined as nonlinear functions over semantic subsets [3], [5], [7], [9]. We additionally obtain binary labels from $y_2$ using two-cluster $k$-means fitted separately within each generated batch, with labels ordered by the cluster centers. We evaluate OOD classification after applying a non-Gaussian shift with outliers to semantic dimensions $\{9,10\}$. See \Cref{sec:num_set_details} for full task design.

\paragraph{Identification of semantics.} The results in \Cref{fig:numeric_nonlinear} show that, under the independent latent variable scenario, unbiased semantic variables are accurately predicted from the learned representations (\(R^2 \approx 1\)), whereas semantic variables omitted by selection bias have low prediction scores (\(R^2 \approx 0\)). In the dependent latent variable scenario, some misaligned semantics become partially predictable, reflecting inherent mutual predictability among strongly dependent variables~\citep{von2021self, yao2023multi}. Modality-specific variables have low prediction scores under the evaluated probes across the reported settings. Similar effects are observed under perturbation bias, as illustrated in \Cref{fig:numeric_nonlinear_perturb}. Simple linear regression also achieves high $R^2$ scores, as demonstrated in \Cref{fig:numeric_linear}. These findings are qualitatively consistent with the dependence on the retained semantic block characterized in \Cref{thm:idsemantics}. Additional analyses on misassigned encoding dimensions and combined bias effects, are provided in \Cref{sec:num_ident_details}.

\begin{figure}[tbp]
    \centering
    \includegraphics[width=0.271\textwidth]{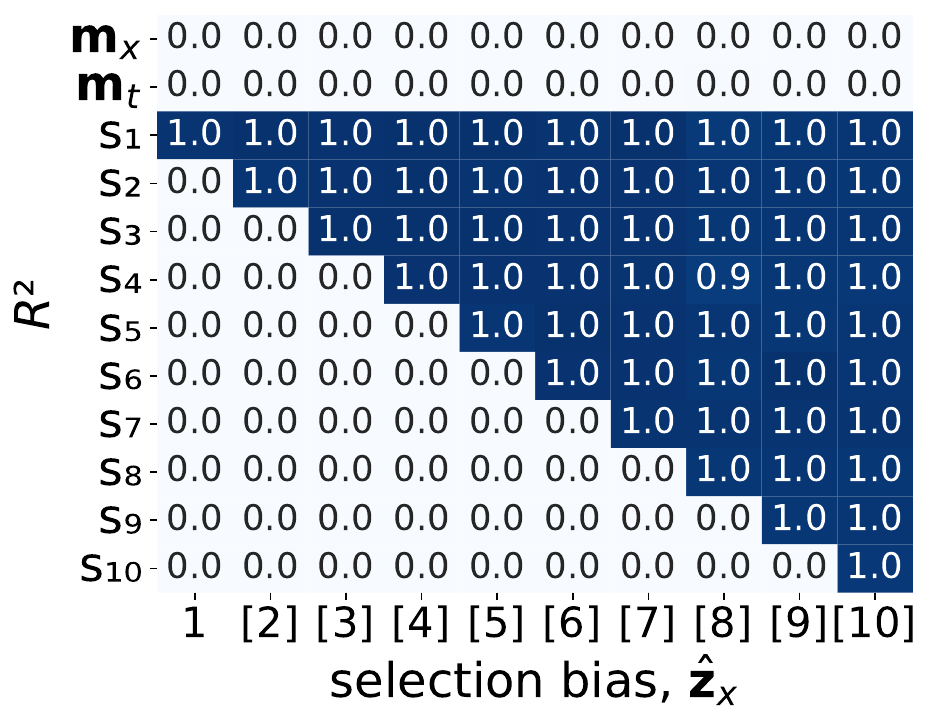}
    \hfill
    \includegraphics[width=0.228\textwidth]{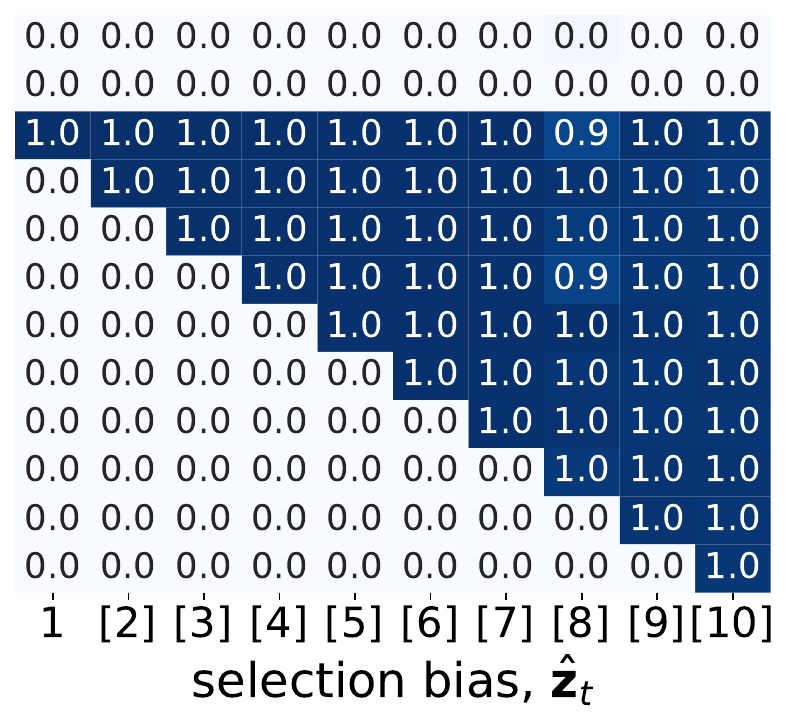}
    \hfill
    \includegraphics[width=0.228\textwidth]{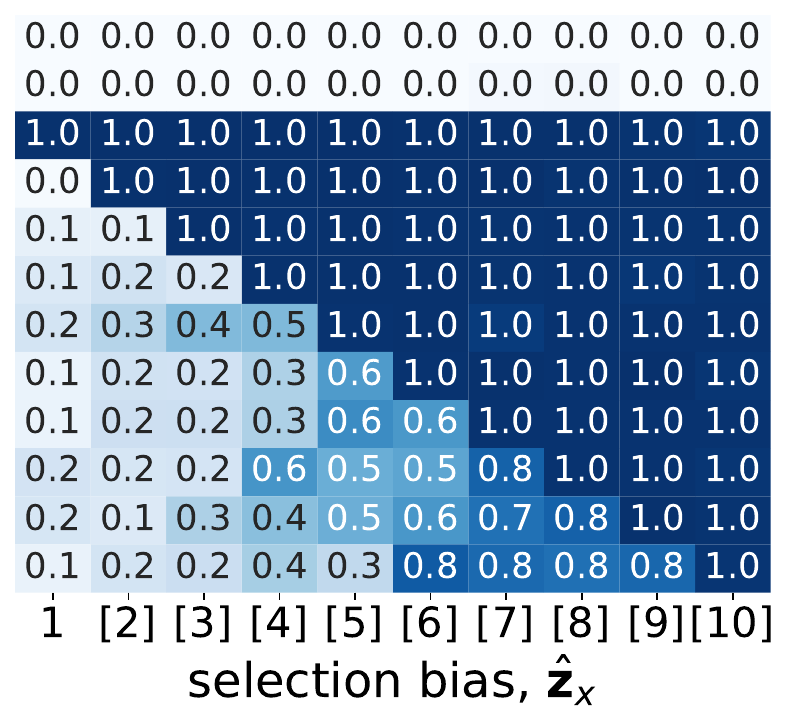}
    \hfill
    \includegraphics[width=0.228\textwidth]{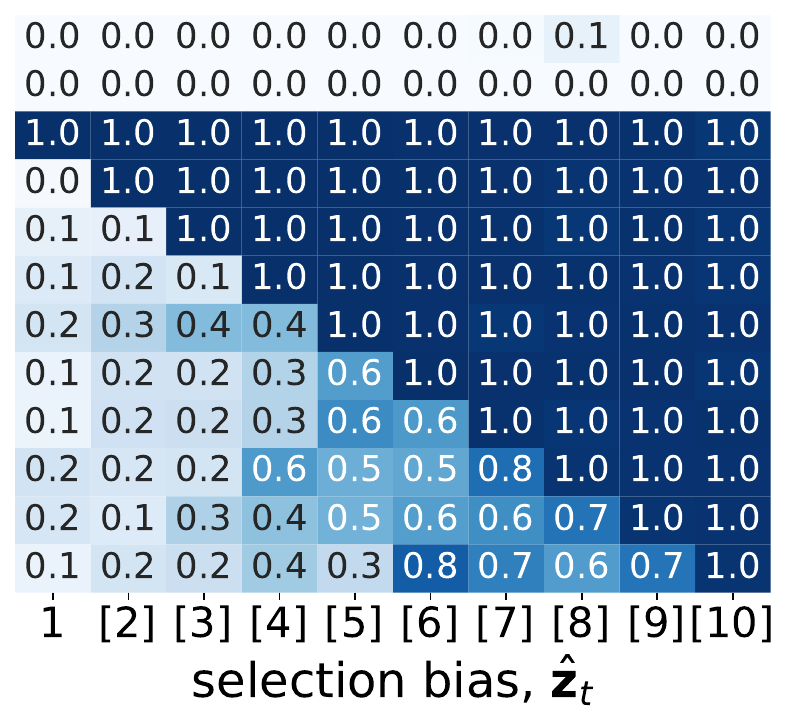}
    \caption{Mean \(R^2\) scores under selection bias settings. From left to right: predictions based on \(\hat{\mathbf{z}}_x\) and \(\hat{\mathbf{z}}_t\) with \textbf{independent} latent semantics, followed by those with \textbf{dependent} latent semantics. Negative mean $R^2$ scores are clipped to zero for visualization.}
    \label{fig:numeric_nonlinear}
\end{figure}

\begin{figure}[tbp]
    \centering
        \includegraphics[width=0.25\textwidth]{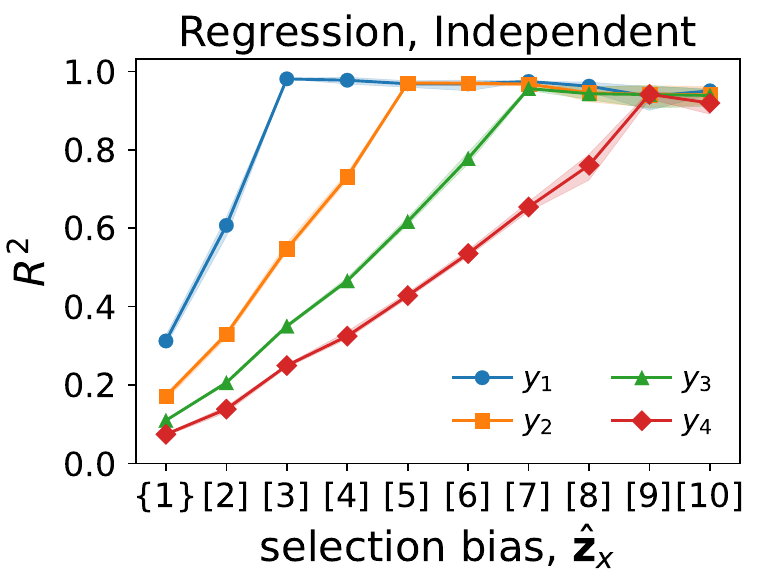}
        \includegraphics[width=0.234\textwidth]{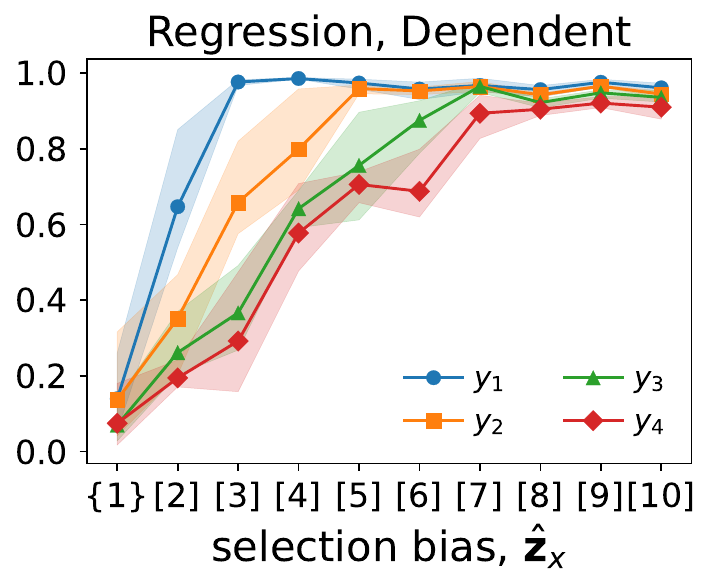}
        \hfill
        \includegraphics[width=0.25\textwidth]{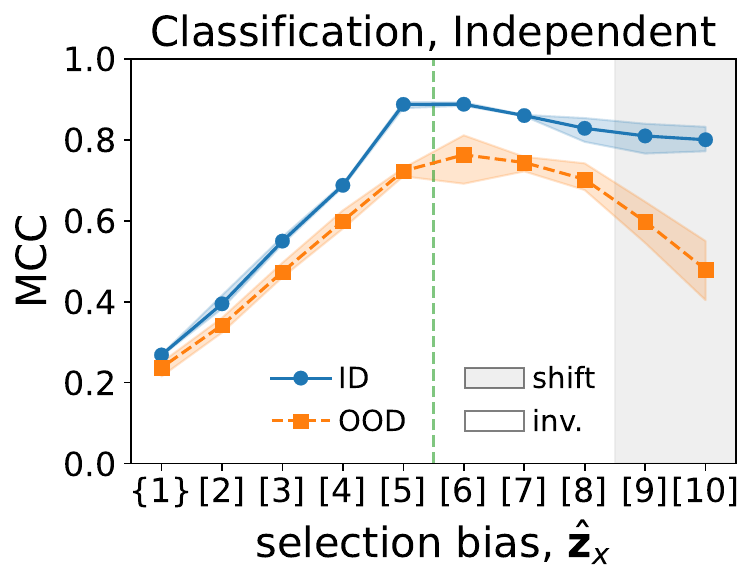}
        \includegraphics[width=0.234\textwidth]{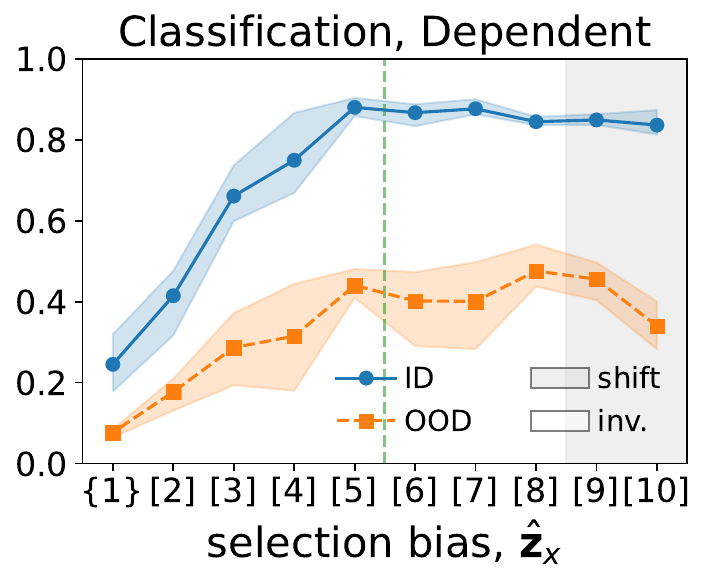}
    \caption{Downstream performance of pretrained \(\hat{\mathbf{z}}_x\) under selection bias. Left: in-distribution (ID) regression performance. Right: ID classification and out-of-distribution (OOD) generalization. Curves show means and shaded regions show ranges across seeds; negative regression $R^2$ scores are clipped to zero before aggregation.}
    \label{fig:drop_downstream}
\end{figure}

\paragraph{Downstream performance.}
As shown in \Cref{fig:drop_downstream}, including more semantic variables during pretraining improves in-distribution regression performance in these experiments. Under the evaluated distribution shift, appropriately chosen selection or perturbation biases improve OOD classification performance. These results illustrate the practical motivation behind \Cref{cor:allsemantics,cor:invsemantics}. Additional results on the effects of perturbation bias are provided in \Cref{sec:num_task_details}.


\subsection{MPI3D-Complex: Real-World Dataset with Factorized Latent Variables}
\label{sec:MPIReal3D}
\paragraph{Experimental setup.} 
\begin{wrapfigure}{r}{0.35\textwidth}
    \centering
    \vspace{-1em}
    \includegraphics[width=\linewidth]{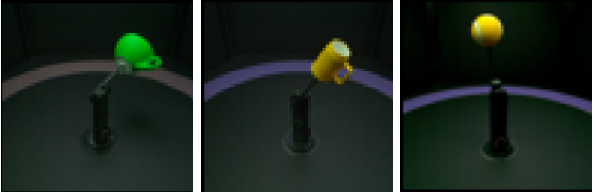}
    \caption{MPI3D-Comp. samples.}
    \vspace{-0.5em}
    \label{fig:mpi_samples}
\end{wrapfigure}
The MPI3D-Complex dataset~\citep{NEURIPS2019_d97d404b} contains real-world images annotated with seven mutually independent, discrete ground-truth factors: object color (\texttt{color}), shape (\texttt{shape}), size (\texttt{size}), camera height (\texttt{cam.}), background color (\texttt{back.}), horizontal position (\texttt{hori.}), and vertical position (\texttt{vert.}). We treat \texttt{hori.} and \texttt{vert.} as image-specific factors, and the remaining five as semantic variables. See \Cref{fig:mpi_samples} for example images from the MPI3D-Complex dataset. Text is generated from the ground-truth semantics using content-word mappings and three manually designed templates. To simulate misalignment, we introduce: \textbf{(i)} selection bias by progressively increasing the subset \(\mathbb{I}_\theta\) of included semantic indices: \circled{1}: {\texttt{color}},
\circled{2}: \{\texttt{color}, \texttt{shape}\},
\circled{3}: \{\texttt{color}, \texttt{shape}, \texttt{size}\},
\circled{4}: \{\texttt{color}, \texttt{shape}, \texttt{size}, \texttt{cam.}\},
\circled{5}: all five; \textbf{(ii)} perturbation bias by fixing \(\mathbb{I}_\theta = \{\texttt{color}, \texttt{shape}, \texttt{size}, \texttt{cam.}, \texttt{back.}\}\) and varying \(\mathbb{I}_\rho\) from \circled{1}: \(\emptyset\) to \circled{5}: \{\texttt{shape}, \texttt{size}, \texttt{cam.}, \texttt{back.}\} in reverse order, independently resampling each designated factor uniformly from its possible values with probability $0.9$, allowing the original value to be sampled again. 

Training uses a symmetric InfoNCE objective with cosine similarity and both within-modality and cross-modal negatives. Performance is evaluated using the average Matthews Correlation Coefficient (MCC) across three random seeds, based on predicting all seven image factors from each modality's learned representations using logistic regression and nonlinear MLP classifiers. Full details on dataset attributes, text generation, bias configurations, and architectures are provided in \Cref{sec:mpi_setup_details}.

\paragraph{Results.} \Cref{fig:mpi3d_nonlinear} presents the nonlinear results with learned representations. The findings demonstrate that, even with discrete latent variables, misaligned semantic variables have low nonlinear-probe MCC scores under both selection and perturbation bias. In contrast, unbiased semantics are accurately predicted, with generally high MCC scores, showing qualitatively similar patterns in a discrete latent setting beyond the assumptions of our continuous-variable theory. Further investigations into MCC using linear classifiers and ablation studies on encoder dimensionality are in \Cref{sec:ablation_mpi}.

\begin{figure}[tbp]
\centering
    \includegraphics[width=0.306\textwidth]{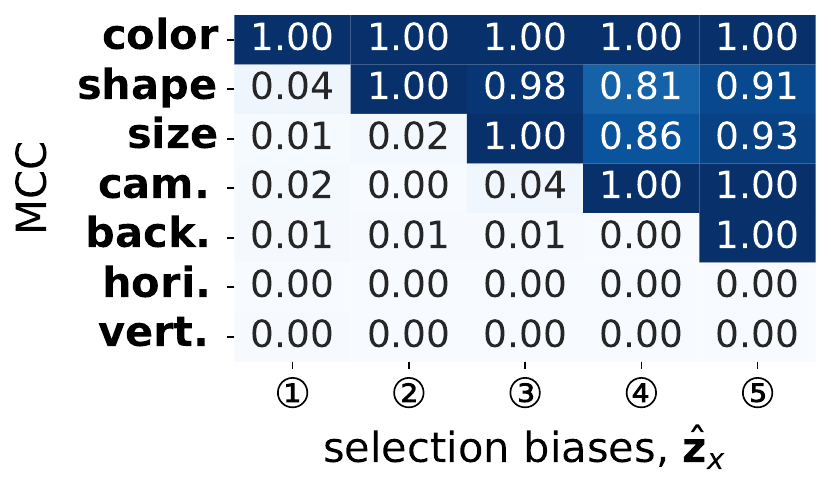}
    \hfill
    \includegraphics[width=0.225\textwidth]{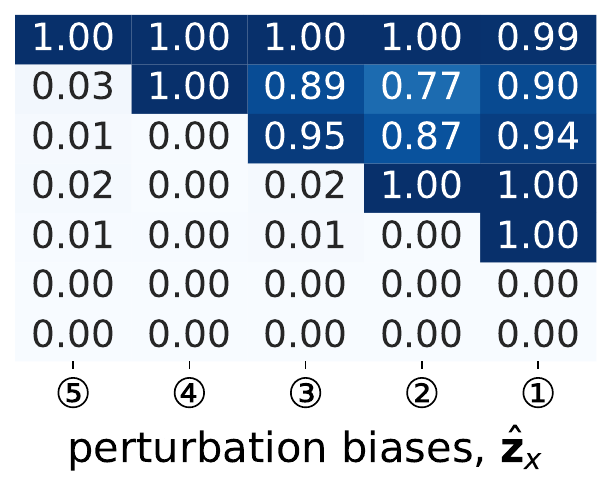}
    \hfill
    \includegraphics[width=0.225\textwidth]{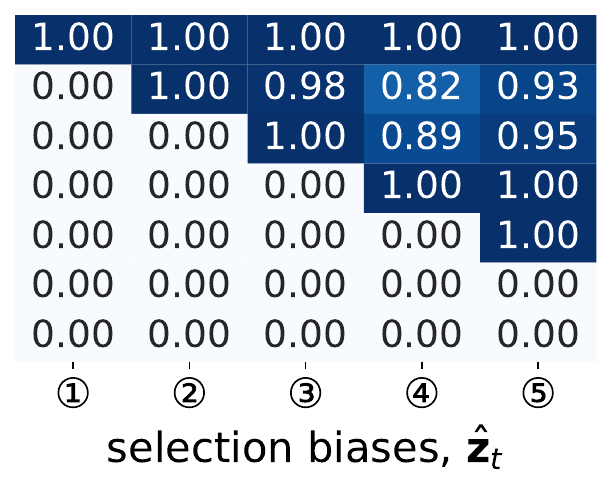}
    \hfill
    \includegraphics[width=0.225\textwidth]{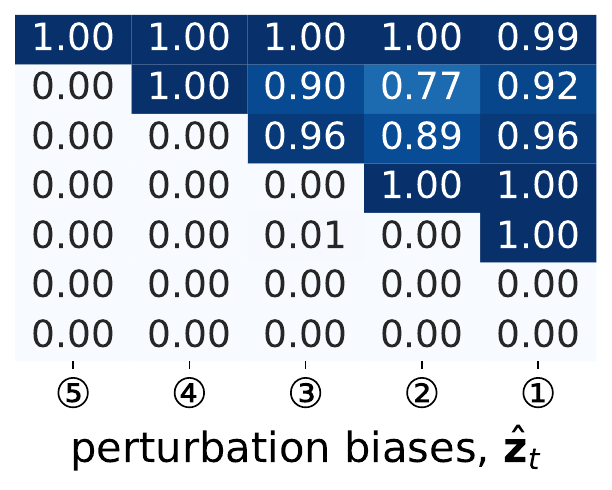}
    \caption{Mean MCC scores under misalignment settings. Left to right: Image features \(\hat{\mathbf{z}}_x\) with selection and perturbation bias settings, text features \(\hat{\mathbf{z}}_t\) under the same bias settings. Scores are clipped to $[0,1]$ before averaging across seeds.}
    \label{fig:mpi3d_nonlinear}
    \vspace{-0.5em}
\end{figure}

\subsection{Causal3DIdent: Semi-Synthetic Dataset with Structured Causal Latent Variables}
\label{sec:Causal3DIdent}

\paragraph{Experimental setup.}  
\begin{wrapfigure}{r}{0.35\textwidth}
    \centering
    \vspace{-1.2em}
    \includegraphics[width=\linewidth]{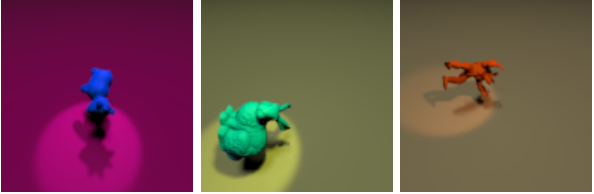}
    \caption{Causal3DIdent samples.}
    \label{fig:c3did_random_samples}
    \vspace{-0.5em}
\end{wrapfigure}
We further conduct experiments on the \emph{Causal3DIdent} dataset~\citep{zimmermann2021contrastive, von2021self, daunhawer2022identifiability, yao2023multi}, a semi-synthetic dataset allowing explicitly defined causal relations among latent variables. Images are generated from 10 latent variables: 3 discrete (\texttt{shape}, \texttt{x\_pos}, \texttt{y\_pos}) and 7 continuous (\texttt{color}, \texttt{s\_pos}, \texttt{s\_color}, \texttt{b\_color}, \texttt{alpha}, \texttt{beta}, \texttt{gamma}). We treat the rotation angles (\texttt{alpha}, \texttt{beta}, \texttt{gamma}) as image-specific, and the remaining as semantic variables following a causal graph among latent variables (\Cref{fig:c3did_gen}). See \Cref{fig:c3did_random_samples} for example images from the dataset. For text, we map the color-based attributes (\texttt{color}, \texttt{s\_color}, \texttt{b\_color}) to discrete color names and render continuous \texttt{s\_pos} using five directional descriptions. Text descriptions are rendered using five templates, based on latent semantic variables under selection and perturbation bias. Specifically, selection bias settings (\(\mathbb{I}_\theta\)) vary from \{\texttt{shape}\} to the set of all seven semantic indices, with \(\mathbb{I}_\rho = \emptyset\); perturbation bias settings progressively perturb subsets \(\mathbb{I}_\rho\) from \(\emptyset\) to \{\texttt{x\_pos}, \texttt{y\_pos}, \texttt{s\_pos}, \texttt{color}, \texttt{s\_color}, \texttt{b\_color}\} in reverse order, with full semantic selection.

\begin{figure}[tbp]
\centering
    \includegraphics[width=0.45\textwidth]{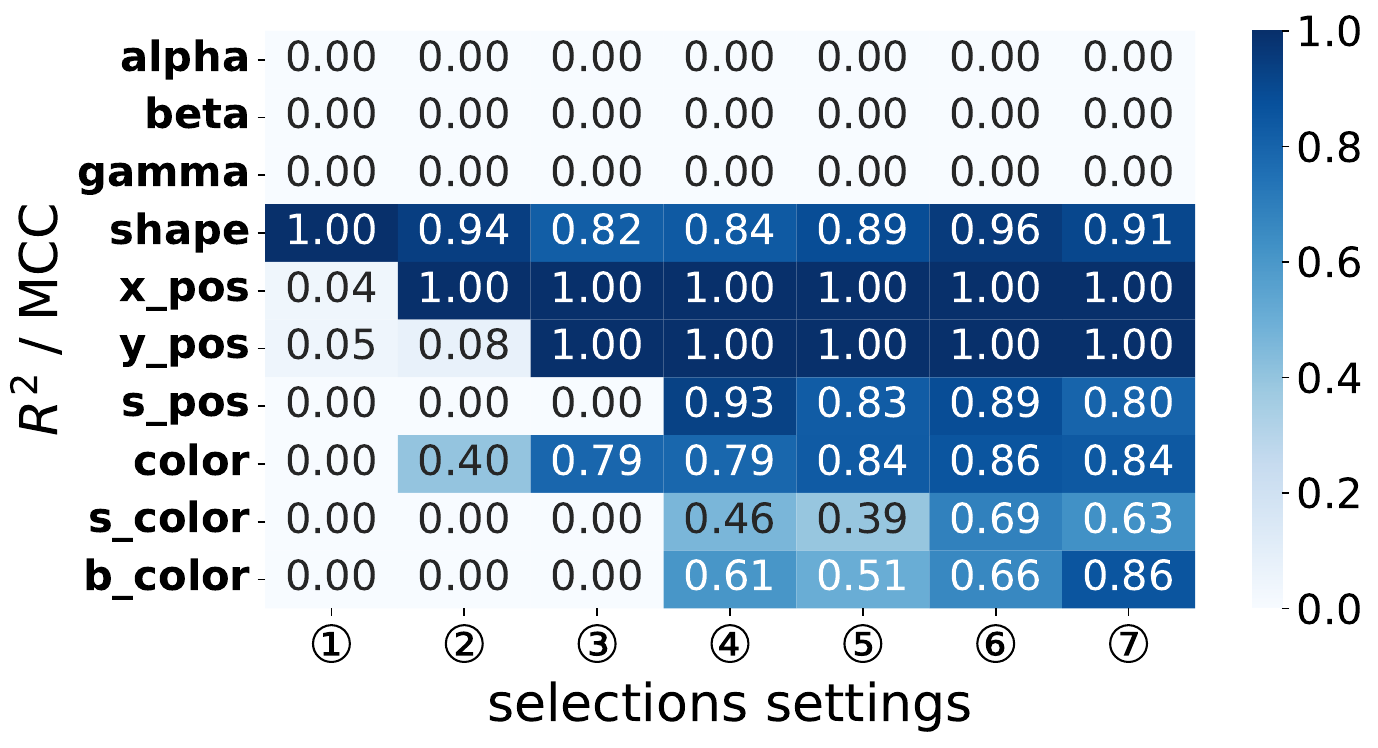}
    \hfill
    \includegraphics[width=0.43\textwidth]{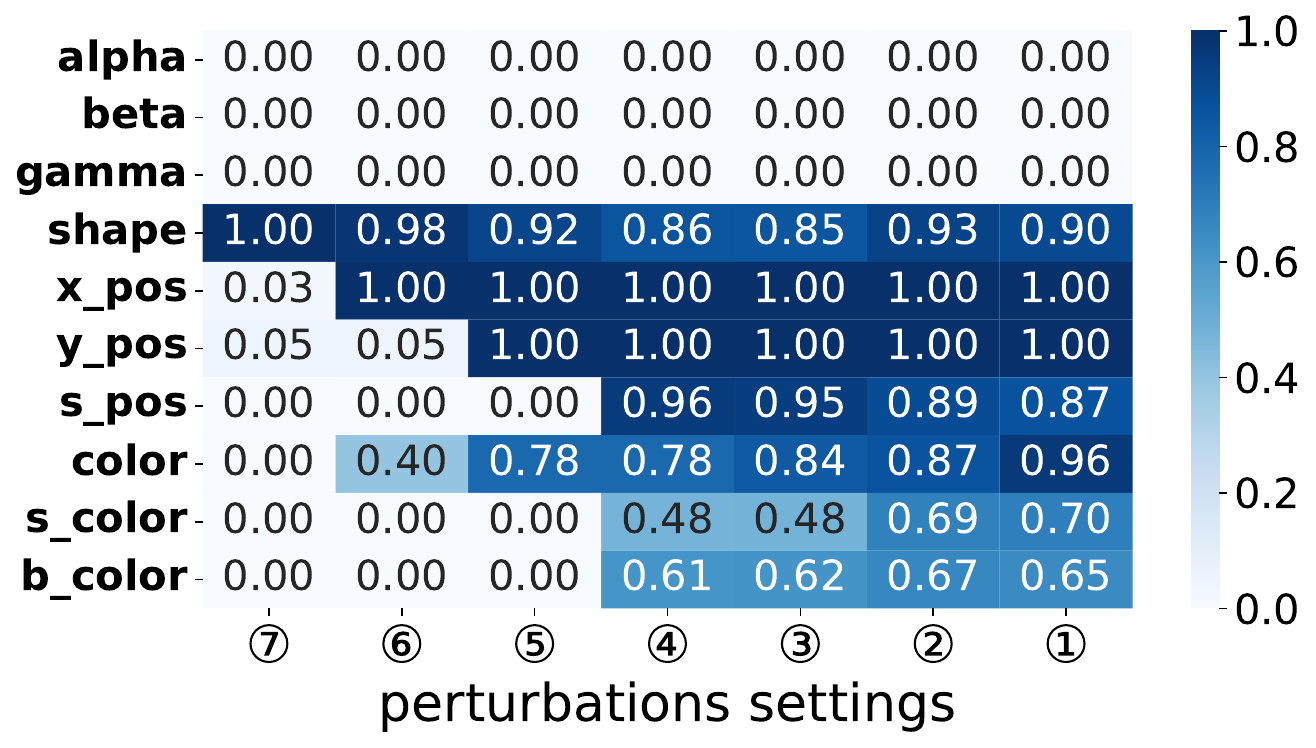}
\caption{Predicting semantic variables under misalignment using image features. \(R^2\) is reported for continuous factors and MCC for discrete factors. The prediction patterns align with our theoretical results, though additional factors may remain predictable through statistical correlations. Scores are clipped to $[0,1]$ before averaging across seeds.}
\label{fig:c3did_nonlinear}
\end{figure}

Training paradigm mirrors that of MPI3D-Complex. We evaluate using \(R^2\) for continuous and MCC for discrete latent variables, averaged over three seeds. See \Cref{sec:c3d_detail_setup} for full experimental details.

\vspace{-0.5em}
\paragraph{Results.}  
\Cref{fig:c3did_nonlinear} presents the prediction performance of a nonlinear MLP classifier or regressor trained on learned image representations. We observe high prediction scores for selected, unperturbed semantic variables, using $R^2$ for continuous factors and MCC for discrete factors. For semantic variables that are continuous in the image modality but discretized in the text modality, prediction performance shows some degradation (e.g., \texttt{s\_color} in selection setting \circled{6} or perturbation setting \circled{2}), yet still achieves relatively high \(R^2\) scores. Image-specific variables have low prediction scores under the evaluated probes. Likewise, semantic variables affected by selection or perturbation bias generally become less predictable. For instance, in selection setting \circled{1} or perturbation setting \circled{7}, only \texttt{shape} remains predictable. When factors such as \texttt{x\_pos} or \texttt{y\_pos} are included, other dependent semantic variables, such as \texttt{color}, become partially predictable (e.g., in selection setting \circled{2}). Similarly, including the discretized description of \texttt{s\_pos} is associated with greater predictability of \texttt{s\_color} and \texttt{b\_color}, reflecting the latent causal structure. Overall, these experiments show qualitatively similar patterns in a setting with mixed discrete and continuous factors and discretized text descriptions. Analyses of the text representations are provided in \Cref{sec:c3d_ablations}.

\subsection{Case Study: Zero-Shot Evaluation of OpenCLIP Model Representations}
\label{sec:openclip}
We further investigate caption coverage and zero-shot performance using OpenCLIP ViT-B/32 and ViT-L/14 models pretrained on LAION-400M \citep{ilharco_gabriel_2021,schuhmann2022laion}. To achieve this, we develop a structured taxonomy comprising 146 visual concepts, organized into 15 distinct concept groups with corresponding abbreviations for clarity in plots and analyses: \texttt{Animal}, \texttt{Clothing}, \texttt{Color}, \texttt{Food}, \texttt{Object}, \texttt{Role} (i.e., human roles), \texttt{Scene}, \texttt{Vehicle}, \texttt{Weather}, \texttt{Texture}, \texttt{POV} (i.e., point of view of the photo), \texttt{Emot.} (i.e., emotion or mental state), \texttt{Postproc.} (i.e., post-processing), \texttt{Trait} (i.e., trait judgment by appearance), and \texttt{Stere.} (i.e., stereotypes). For each concept, we collect a dataset of up to 200 images from Flickr\footnote{\url{https://www.flickr.com/}} by prompting with concept names, available under CC licenses. See the concept taxonomy and data collection details in \Cref{sec:app_data_collect}.

\begin{figure}[tbp]
    \centering
    \includegraphics[width=\linewidth]{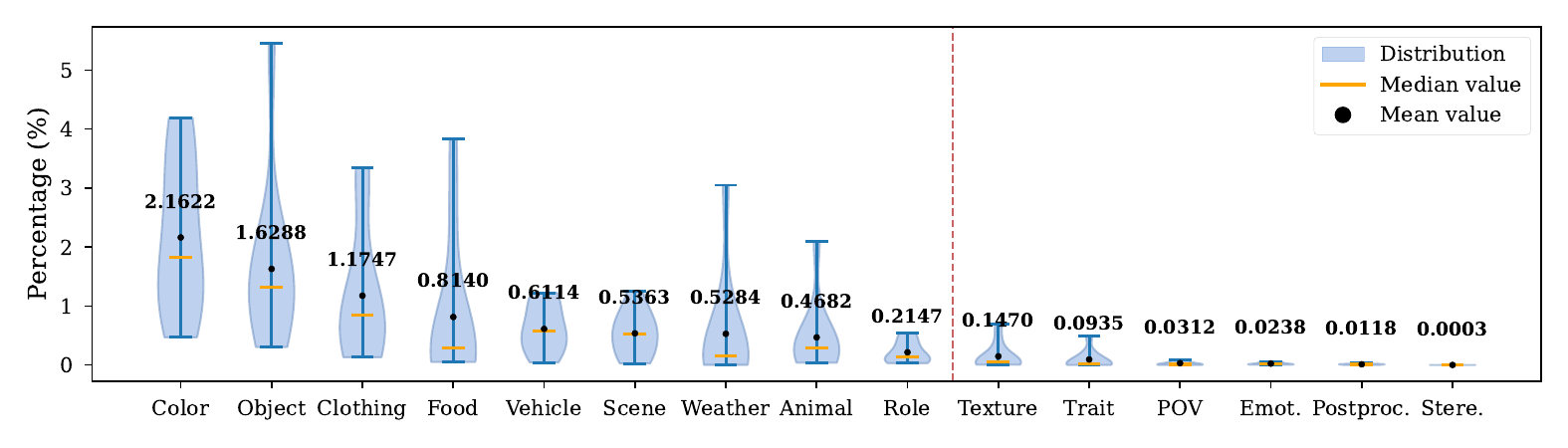}
    \vspace{-1.5em}
    \caption{Concept coverage by group in LAION-400M captions. See \Cref{sec:app_cap_stats} for details.}
    \label{fig:group_converage}
\end{figure}

\begin{figure}[tbp]
    \centering
    \includegraphics[width=\linewidth]{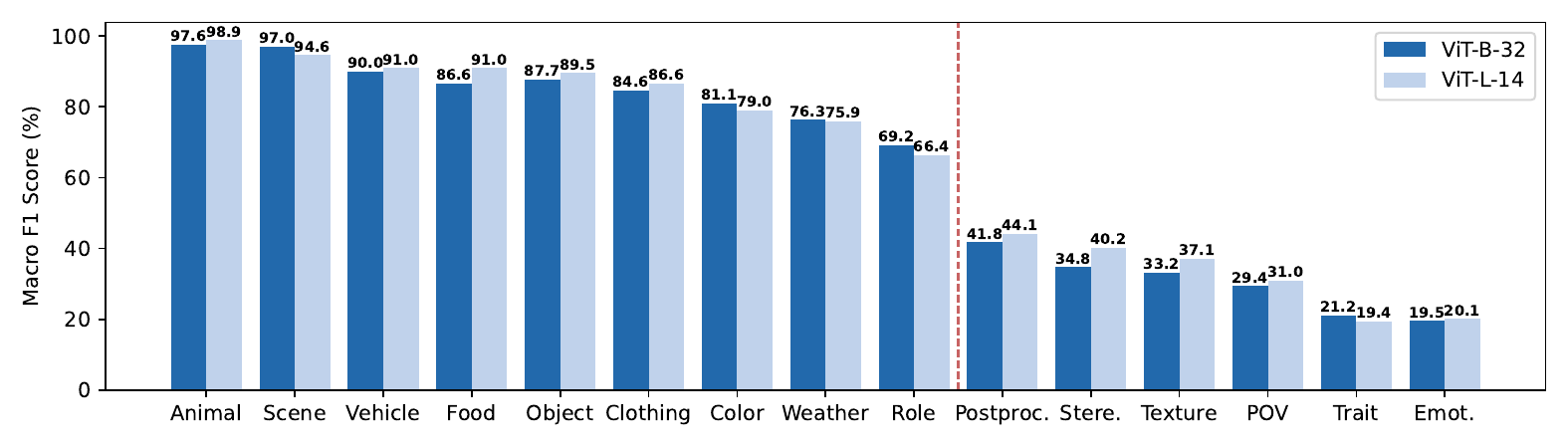}
    \vspace{-1.5em}
    \caption{Zero-shot F1 score of OpenCLIP model trained on LAION-400M in each concept group.}
    \label{fig:f1_bar}
\end{figure}

\paragraph{Concept coverage in LAION-400M captions.} We begin by analyzing the concept coverage of LAION-400M captions, measuring the average coverage rate, defined as the percentage of captions that explicitly mention each concept or its synonyms, averaged by concept group. As shown in \Cref{fig:group_converage}, concept groups display substantial variation in average coverage: common concepts exhibit relatively high coverage (e.g., \texttt{Object}: 1.63\%, \texttt{Color}: 2.16\%), while some concepts are rarely mentioned (e.g., \texttt{Stere.}: 0.0003\%, \texttt{Postproc.}: 0.0118\%). Nonzero group-average coverage does not rule out omission of individual concepts. Moreover, these rates measure lexical caption frequency; estimating the selection probabilities in our LVM would additionally require conditioning on whether each concept is present in the corresponding image.

\paragraph{Zero-shot evaluation of OpenCLIP models.} \Cref{fig:f1_bar} reports group-wise zero-shot macro-F1 scores for the two evaluated OpenCLIP models, using concept names as prompts and classifying among concepts within each group. Both evaluated models perform better on frequently captioned groups such as \texttt{Animal} and \texttt{Object} than on under-captioned groups such as \texttt{Trait} and \texttt{Emot.}. This comparison is descriptive and does not isolate caption coverage from differences in concept prevalence, visual difficulty, or evaluation prompts. Concept-level confusion matrices further reveal systematic misclassifications, particularly within low-coverage groups (see \Cref{sec:app_zs_details}). For applications requiring under-captioned concepts such as \texttt{Texture} and \texttt{Emot.}, these observations motivate evaluating more targeted captioning strategies.

%% file: Pages/7_conclusion.tex
\section{Conclusion}
\label{sec:conclusion}

In this work, we present a formal analysis of cross-modal misalignment in multimodal data pairs within the MMCL framework, examining its impact on the learned representations. Under the stated assumptions, population global minimizers of the alignment--entropy surrogate yield representations that depend only on the unbiased shared semantic variables and locally identify this subset almost surely. Misaligned variables are functionally excluded, so selection and perturbation biases determine which semantic block the representations depend on. This characterization highlights a task-dependent trade-off: detailed and consistent captions support broader semantic retention, whereas selective or perturbed text may promote invariance to distribution-sensitive factors. Our experiments across simulations and image-text datasets examine these effects on semantic predictability and downstream performance. These insights highlight the need for multimodal learning frameworks that mitigate harmful misalignment or exploit beneficial biases to enhance representation learning in real-world settings. More discussion of these implications is provided in \Cref{sec:discussion}.

%% file: Pages/appendix.tex
\begin{appendix}
\appendixpage
\crefalias{section}{appendix}      
\crefalias{subsection}{subappendix}   
\crefalias{subsubsection}{subsubappendix} 

\addtocontents{toc}{\protect\setcounter{tocdepth}{2}}
\tableofcontents
\vfill

\clearpage
\section{Notation and Terminology}
\label{sec:notation}

\Cref{tab:notation_summary} provides a summary of the notations and terminologies used throughout the paper.

\begin{table}[tbp]
\caption{Notations and terminologies used throughout the paper.}
\label{tab:notation_summary}
\centering
\small
\begin{tabularx}{\linewidth}{lX}
\toprule
\multicolumn{2}{l}{\textbf{Observation and Latent Spaces}} \\
\midrule
\(\mathcal{X}\) & Image observation space (\(\subseteq \mathbb{R}^{d_x}\)) \\
\(\mathcal{T}\sx{(\theta)}\) & Text observation subspace under selection bias \(\theta\) \\
\(\mathcal{S}\) & Latent semantic space (\(\subseteq \mathbb{R}^{n_s}\)) \\
\(\mathcal{M}_x\) & Image-specific non-semantic latent space \\
\(\mathcal{M}_t\) & Text-specific non-semantic latent space \\
\(\mathbb{I}_\mathbf{s}\) & Index set of semantic variables: \([n_s]\) \\
\(\mathbb{I}_{\text{inv}}\) & Index subset of semantic variables that remain invariant under distribution shift\\
\(\mathbb{I}_{\text{var}}\) & Index subset of semantic variables that vary under distribution shift \\
\midrule
\multicolumn{2}{l}{\textbf{Mappings and Functions}} \\
\midrule
\(g_x\) & Differomorphic generative mapping for images: \(\mathcal{S}\times\mathcal{M}_x \to \mathcal{X}\) \\
\(g_{t\sx{(\theta)}}\) & Differomorphic generative mapping for text under selection \(\theta\): \(\mathcal{S}_{\mathbb{I}_\theta}\times\mathcal{M}_t \to \mathcal{T}\sx{(\theta)}\) \\
\(\mathcal{G}_t\) & A diffeomorphism function class where \(g_{t\sx{(\theta)}}\) is selected from\\
\(f_x\) & Image encoder: \(\mathcal{X} \to (0,1)^{n}\) with specified \(n\) \\
\(f_t\) & Text encoder: \(\mathcal{T}\sx{(\theta)} \to (0,1)^{n}\) with specified \(n\) \\
\midrule
\multicolumn{2}{l}{\textbf{Loss Functions}} \\
\midrule
\(\mathcal{L}_{\text{\tiny MMCL}}(f_x, f_t)\) & Symmetric InfoNCE loss for MMCL (\Cref{eq:loss_exp}) \\
\(\mathcal{L}_{\text{\tiny SymAlignMaxEnt}}(f_x, f_t)\) & Population alignment--entropy analytical surrogate (\Cref{eq:loss_thm}) \\
\midrule
\multicolumn{2}{l}{\textbf{Notations for Cross-Modal Misalignment}} \\
\midrule
\(\theta\) & Selection bias, an integer realization in the range \([2^{n_s}-1]\) \\
\(\mathcal{P}^+(\mathbb{I}_\mathbf{s})\) & The set of all non-empty subsets of \(\mathbb{I}_{\mathbf{s}}\)\\
\(\mathbb{I}_{\theta}\) & A selected semantic subset indexed by \(\theta\),  \(\mathbb{I}_{\theta} \in \mathcal{P}^+(\mathbb{I}_\mathbf{s})\) \\
\(\mathbb{I}^c_{\theta}\) & Omitted semantic subset under \(\theta\), \(\mathbb{I}^c_{\theta} = \mathbb{I}_\mathbf{s}\setminus\mathbb{I}_\theta\) \\
\(\rho\) & Perturbation bias, an integer realization in the range \([2^{|\mathbb{I}_\theta|}-1]\)\\
\(\mathcal{P}_{\text{prop}}(\mathbb{I}_\theta)\) & The set of all proper subsets of \(\mathbb{I}_\theta\)\\
\(\mathbb{I}_\rho\) & A subset of \(\mathbb{I}_{\theta}\) subject to perturbation indexed by \(\rho\), \(\mathbb{I}_\rho \in \mathcal{P}_{\text{prop}}(\mathbb{I}_\theta)\) \\
\(\mathbb{I}^c_{\rho}\) & Subset of \(\mathbb{I}_{\theta}\) that always be unbiased under \(\rho\), \(\mathbb{I}^c_{\rho} = \mathbb{I}_{\theta}\setminus\mathbb{I}_\rho\) \\
\(p_{\tilde{\mathbf{s}}_{\mathbb{I}_\theta} \mid \mathbf{s}, \theta, \rho}\) & Perturbation conditional distribution, reflecting misalignment  (\Cref{eq:perturbation}) \\
\(A \subseteq \mathbb{I}_\rho\) & A random subset of semantic variables subject to perturbation, drawn from \(p_{A}\) \\
\midrule
\multicolumn{2}{l}{\textbf{Distributions and Operators}} \\
\midrule
\(p_\mathbf{s},\, p_{\mathbf{m}_x},\, p_{\mathbf{m}_t}\) & Distributions over semantic, image-specific and text-specific variables, respectively \\
\(H(\cdot)\) & Differential entropy \\
\(\delta(\cdot)\) & Dirac delta function \\
\midrule
\multicolumn{2}{l}{\textbf{Random Variables}} \\
\midrule
\(\mathbf{x}\) & Image observation sampled from \(\mathcal{X}\) \\
\(\mathbf{t}\sx{(\theta)}\) & Text observation under selection view \(\theta\), sampled from \(\mathcal{T}\sx{(\theta)}\) \\
\(\mathbf{s}\) & Latent semantic variables in \(\mathcal{S}\) \\
\(\mathbf{s}_{\mathbb{I}_{\theta}}\) & Selected latent semantic variables for generating text \\
\(\mathbf{s}_{\mathbb{I}^c_{\theta}}\) & Omitted latent semantic variables when generating text \\
\(\mathbf{s}_{\mathbb{I}_{\rho}}\) & Perturbable latent semantic variables for generating text \\
\(\mathbf{s}_{\mathbb{I}^c_{\rho}}\) & Unbiased latent semantic variables within the selected part \\
\(\mathbf{m}_x\) & Image-specific non-semantic variable in \(\mathcal{M}_x\) \\
\(\mathbf{m}_t\) & Text-specific non-semantic variable in \(\mathcal{M}_t\) \\
\(\mathbf{z}_x\) & Combined latent variables for images: \((\mathbf{s}, \mathbf{m}_x)\) \\
\(\mathbf{z}_{t\sx{(\theta)}}\) & Combined latent variables for text under \(\theta\): \((\mathbf{s}_{\mathbb{I}_\theta}, \mathbf{m}_t)\) \\
\bottomrule
\end{tabularx}
\end{table}

\newpage
\section{Related Work}
\label{sec:app_realtedwork}

\paragraph{Theoretical multimodal (multi-view) contrastive learning.}  
Recent work has sought to formalize the theoretical foundations of multimodal and multi-view representation learning, particularly under contrastive objectives. Wang et al.~\citep{wang2020understanding} decompose the InfoNCE loss~\citep{oord2018representation} into an alignment term, which pulls positive pairs together, and a uniformity term, which encourages dispersion over a hypersphere---laying the groundwork for subsequent analysis. Zimmermann et al.~\citep{zimmermann2021contrastive} show that contrastive objectives can invert the data-generating process, while Liu et al.~\citep{liu2024revealing} extend this result to multimodal settings. However, these approaches often rely on strong assumptions about latent distributions or manifold structures, which limits their practical applicability. A complementary line of work, such as~\citep{von2021self}, demonstrates that contrastive learning with data augmentation can recover shared content without requiring such strong parametric assumptions. This has been extended to MMCL by~\citep{lyu2022understanding, daunhawer2022identifiability}, and Yao et al.~\citep{yao2023multi} further investigate identifiability under partial observability in multi-view settings. Distinct from prior work, we do not assume fixed content-style decompositions, causal directions, or perfectly aligned pairs. Instead, we analyze an alignment--max-entropy surrogate under cross-modal misalignment and empirically examine the resulting predictions in MMCL settings.

\vspace{-0.5em}
\paragraph{Vision-language models and perspectives on misalignment.}  
Multimodal contrastive learning (MMCL) has achieved significant empirical success, particularly in aligning visual and textual modalities using models such as CLIP~\citep{radford2021learning} and ALIGN~\citep{jia2021scaling}. These successes are partly attributed to the use of massive training corpora, e.g., LAION-5B~\citep{schuhmann2022laion}, which are substantially larger than those used for vision-only foundation models~\citep{he2022masked, oquab2024dinov}. However, real-world multimodal datasets are often imperfectly aligned and noisy~\citep{miech2019howto100m}. Existing empirical methods~\citep{alayrac2020self, lavoie2024modeling} typically treat such misalignment as label noise, employing strategies such as multiple-instance learning or dataset refinement~\citep{alayrac2020self} to mitigate its impact. While some recent work suggests that contrastive models are robust to certain forms of structured misalignment~\citep{nakada2023understanding}, Others suggest augmenting text as a proxy for changes in visual content~\citep{cai2025clap}. Our work suggests that cross-modal misalignment can act as either a barrier or a bridge, depending on the application. Unlike~\citep{nakada2023understanding}, which assumes linear representations without modeling the generative process, our analysis is grounded in a realistic latent variable model that provides a deeper understanding of misalignment from a data-generating perspective.

\vspace{-0.5em}
\paragraph{Identifiability in latent variable models.}
Identifiability analysis addresses the fundamental question of whether the learning process can uniquely recover the latent generative structure or distribution underlying the observed data. This problem has been extensively studied in the context of nonlinear independent component analysis (ICA)~\citep{hyvarinen1999nonlinear, locatello2019challenging, khemakhem2020variational, sorrenson2020disentanglement} and causal representation learning (CRL)~\citep{liu2022identifying, liu2024iclr, pmlr-v235-zhang24br, yao2025unifying, liu2025latent}. In either setting, full identifiability, typically up to permutation, is rarely achievable in practice without strong assumptions, such as access to interventional data.
Consequently, recent works have focused on partial identifiability~\citep{gu2020partial, kong2023partial, gao2025domain} or relaxed equivalence classes, such as identifiability up to linear transformations~\citep{zimmermann2021contrastive, liu2024revealing} or up to group-wise/block-wise indeterminacy~\citep{von2021self, yao2023multi}, which can offer sufficient guarantees for specific tasks or applications. In the context of multimodal representation learning, several recent studies have explored identifiability results~\citep{lyu2022understanding, daunhawer2022identifiability, liu2024revealing}, but largely neglect the presence of cross-modal misalignment. Our work explicitly models misalignment and establishes local block recovery almost surely for population global minimizers of an alignment--max-entropy surrogate.

\vspace{-0.5em}
\paragraph{Invariant representation learning.}
Invariant representation learning (IRL) seeks to learn representations that remain robust under distributional shifts between environments~\citep{arjovsky2019invariant, eastwood2022probable, gao2025domain}, particularly in settings where empirical risk minimization (ERM)~\citep{NIPS1991_ff4d5fbb} fails to generalize out of distribution. In the absence of such challenge, ERM is sufficient for in-distribution prediction, rendering the objective of IRL unnecessary. From a causal perspective, learning invariant representations, or more ambitiously, autonomous mechanisms \cite{scholkopf2021toward}, requires observing sufficient variations in the latent factors underlying data \cite{liu2022identifying}. Such variability can be introduced through interventional data~\citep{lippe2022citris, lachapelle2022disentanglement}, exchangeability assumptions~\citep{reizinger2025identifiable}, or the use of auxiliary variables such as domain indices~\citep{pmlr-v235-zhang24br, liu2024identifiable}. However, direct interventions on latent variables are typically infeasible in real-world observational data, and auxiliary variable methods often require access to a large number of diverse environments to ensure identifiability---an assumption that is often challenging to satisfy in practice. Our work sheds light on an alternative approach: by leveraging the inherent flexibility of text supervision, we show that manipulating biases, specifically through selective omission or semantic perturbation in text, can serve as a controllable proxy for environmental variation.

\section{Proofs}
\label{sec:proofs}

\subsection{Lemmas}
Before proceeding with the proof, we first establish the following lemmas.

\paragraph{Entropy convention.}
For a random vector $\mathbf z$ taking values in $(0,1)^n$, we use the extended-entropy convention
\(
H(\mathbf z)
:=-D_{\mathrm{KL}}\!\left(
p_{\mathbf z}\,\middle\|\,
\operatorname{Unif}((0,1)^n)
\right)
\), 
where $p_{\mathbf z}$ denotes its probability law. This agrees with differential entropy for absolutely continuous laws on the unit cube, including the value $-\infty$, and equals $-\infty$ whenever the law is not absolutely continuous with respect to Lebesgue measure. Under this convention, \Cref{eq:loss_thm} is well defined with values in $[0,+\infty]$. It equals $+\infty$ whenever either output marginal is not absolutely continuous, and equals zero if and only if the paired representations agree a.s. and both marginals are uniform on $(0,1)^n$.

\smallskip
\begin{restatable}[Global Minimum of \(\mathcal{L}_{\text{\tiny SymAlignMaxEnt}}\)]{lemma}{globalminimum}
\label{lem:globalminimum}
Under the assumptions of \Cref{thm:idsemantics}, the global minimum of
\begin{equation}
\label{eq:loss_thm_app}
\mathcal{L}_{\text{\tiny SymAlignMaxEnt}}(f_x, f_t) = \E_{(\mathbf{x}, \mathbf{t}\sx{(\theta)}) \sim p_{\mathbf{x}, \mathbf{t}\sx{(\theta)}}} \left[\|f_x(\mathbf{x}) - f_t(\mathbf{t}\sx{(\theta)})\|_2 \right] - \frac{1}{2} \Bigl( H\bigl(f_x(\mathbf{x})\bigr) + H\bigl(f_t(\mathbf{t}\sx{(\theta)})\bigr) \Bigr),  
\end{equation}
over the smooth encoders considered in that theorem is $0$. This minimum is attained by the following pair of smooth functions:
\begin{gather}
f_x^* = \mathbf{d} \circ (g^{-1}_{x})_{\mathbb{I}^c_\rho}: \mathcal{X} \to (0, 1)^{n}, \label{eq:fx_ast}   \\
f_t^* = \mathbf{d} \circ (g^{-1}_{t\sx{(\theta)}})_{\mathbb{I}^c_\rho}: \mathcal{T}\sx{(\theta)} \to (0, 1)^{n}, \label{eq:ft_ast} 
\end{gather}
where:
\begin{itemize}
    \item \(g_x\) and \(g_{t\sx{(\theta)}}\) denote the true underlying generative mappings for images and paired text, respectively, as described in \Cref{sec:problemformulate}.
    \item The operator \((\cdot)_{\mathbb{I}^c_\rho}\) extracts the components corresponding to the unbiased semantic variables (i.e., unaffected by the selection bias nor the perturbation bias), with 
    \(
    n = \left|\mathbb{I}^c_\rho\right|
    \)
    being their dimensionality.
    \item \(\mathbf{d} = (d_1, \dots, d_n)\) is defined via the Darmois construction \citep{darmois1951analyse, hyvarinen1999nonlinear, von2021self}, where for each \(i\in [n]\),
    \[
    d_i(\mathbf{s}_{\mathbb{I}^c_\rho}) = \mathrm{CDF}_i\bigl(s_{\mathbb{I}^c_\rho,i} \mid \mathbf{s}_{\mathbb{I}^c_\rho, [i-1]}\bigr) = \mathbb{P}\Bigl(S_{\mathbb{I}^c_\rho,i} \leq s_{\mathbb{I}^c_\rho,i} \,\Big|\, \mathbf{s}_{\mathbb{I}^c_\rho, [i-1]}\Bigr),
    \]
    with \(\mathrm{CDF}_i\) denoting the conditional cumulative distribution of \(s_{\mathbb{I}^c_\rho,i}\) given \(\mathbf{s}_{\mathbb{I}^c_\rho, [i-1]}\).
\end{itemize}
\end{restatable}

\smallskip
\begin{proof}[\textbf{Proof of \Cref{lem:globalminimum}}]
Choose the jointly smooth conditional-CDF versions assumed in \Cref{thm:idsemantics} in the definition of $\mathbf d$. Then $\mathbf d$ is smooth and takes values in $(0,1)^n$. Since $g_x$ and $g_{t\sx{(\theta)}}$ are diffeomorphisms, the candidate functions  \(f_x^*\) and \(f_t^*\) in \cref{eq:fx_ast,eq:ft_ast} are smooth. We prove that the candidate functions yield \(\mathcal{L}_{\text{\tiny SymAlignMaxEnt}}(f_x^*, f_t^*)=0\). 

Substituting these candidate functions into the loss in \cref{eq:loss_thm_app}, we have
\begin{align*}
\mathcal{L}_{\text{\tiny SymAlignMaxEnt}}(f_x^*, f_t^*) 
&= \E_{(\mathbf{x}, \mathbf{t}\sx{(\theta)}) \sim p_{\mathbf{x},\mathbf{t}\sx{(\theta)}}} \Bigl[ \|f_x^*(\mathbf{x}) - f_t^*(\mathbf{t}\sx{(\theta)})\|_2 \Bigr] \\
&\quad - \frac{1}{2}\Bigl( H\bigl(f_x^*(\mathbf{x})\bigr) + H\bigl(f_t^*(\mathbf{t}\sx{(\theta)})\bigr) \Bigr).
\end{align*}
By the invertibility of the generative processes \(g_x\) and \(g_{t\sx{(\theta)}}\) (see \Cref{sec:problemformulate}), we may change variables to express the expectation over the latent variables:
\begin{align*}
\mathcal{L}_{\text{\tiny SymAlignMaxEnt}}(f_x^*, f_t^*) 
&= \E_{(\mathbf{z}_x, \mathbf{z}_{t\sx{(\theta)}}) \sim p_{\mathbf{z}_x,\mathbf{z}_{t\sx{(\theta)}}}} \Bigl[ \|\mathbf{d}(\mathbf{s}_{\mathbb{I}^c_\rho}) - \mathbf{d}(\tilde{\mathbf{s}}_{\mathbb{I}^c_\rho})\|_2 \Bigr] \\
&\quad - \frac{1}{2}\Bigl( H\bigl(\mathbf{d}(\mathbf{s}_{\mathbb{I}^c_\rho})\bigr) + H\bigl(\mathbf{d}(\tilde{\mathbf{s}}_{\mathbb{I}^c_\rho})\bigr) \Bigr),
\end{align*}
where \(\mathbf{s}_{\mathbb{I}^c_\rho}\) and \(\tilde{\mathbf{s}}_{\mathbb{I}^c_\rho}\) denote the preserved unbiased components of the semantic variables across image-text pairs, respectively.

We now show that these unbiased semantic components are identical across modalities a.s. By \Cref{asm:pairing}, for any image-text pair the text is generated via a random perturbation process that modifies only a subset \(A \subseteq \mathbb{I}_\rho\) of the activated semantic variables. Specifically, recall \cref{eq:perturbation}, the perturbation density is defined as
\[
p_{\tilde{\mathbf{s}}_{\mathbb{I}_\theta}
  \mid \mathbf{s}, A, \theta,\rho}
\bigl(\tilde{\mathbf{s}}_{\mathbb{I}_\theta}
  \mid \mathbf{s},A,\theta,\rho\bigr)
= \delta\bigl(\tilde{\mathbf{s}}_{\mathbb{I}_{\theta} \setminus A} - \mathbf{s}_{\mathbb{I}_{\theta} \setminus A}\bigr)  
\, p_{\tilde{\mathbf{s}}_{A} \mid \mathbf{s}_{A}}\bigl(\tilde{\mathbf{s}}_{A} \mid \mathbf{s}_{A}\bigr).
\]
Since \(A\subseteq \mathbb{I}_\rho\), it follows that the indices in \(\mathbb{I}^c_\rho\) are a subset of those in \(\mathbb{I}_{\theta}\setminus A\); that is,
\(
\mathbb{I}^c_\rho \subseteq \mathbb{I}_{\theta}\setminus A.
\)
Thus, the Dirac delta in the above expression enforces that
\[
\tilde{\mathbf{s}}_{\mathbb{I}^c_\rho}
=
\mathbf{s}_{\mathbb{I}^c_\rho}
\qquad
p_{\tilde{\mathbf{s}}_{\mathbb{I}_\theta}
\mid\mathbf{s},\theta,\rho}\text{-a.s.},
\quad
\text{for }p_{\mathbf{s}}\text{-a.e. }\mathbf{s},
\]
under biases \(\theta\) and \(\rho\), regardless of the particular perturbation set \(A\) at each time. 

By the Darmois construction \citep{darmois1951analyse}, $\mathbf d(\mathbf s_{\mathbb I_\rho^c})$ is uniform on $(0,1)^n$. Since $\tilde{\mathbf s}_{\mathbb I_\rho^c}=\mathbf s_{\mathbb I_\rho^c}$ a.s., the candidate representations agree a.s. and both marginals are uniform. Their alignment term and marginal entropies therefore equal zero, giving
\[
\mathcal L_{\text{\tiny SymAlignMaxEnt}}(f_x^*,f_t^*)=0.
\]
Under the entropy convention above, every distribution on $(0,1)^n$ has entropy at most zero. Hence the alignment term and both negative-entropy terms are nonnegative for every admissible encoder pair. Thus zero is the global minimum.
\end{proof}

\medskip
\begin{restatable}[Local invertibility a.s.]{lemma}{uniformizing}
\label{lem:uniformizing}
Let $h:\mathcal U\to\mathcal V$ be smooth, where $\mathcal U,\mathcal V\subseteq\mathbb R^n$ are open, and let $\mathbf u$ be a random vector taking values in $\mathcal U$. If the law of $\mathbf v=h(\mathbf u)$ is absolutely continuous with respect to $n$-dimensional Lebesgue measure, then 
\[
\Pr\!\left(\det J(h)(\mathbf u)=0\right)=0.
\]
Thus $h$ has a smooth local inverse around almost every realization of $\mathbf u$. In particular, this holds when $h(\mathbf u)\sim\operatorname{Unif}((0,1)^n)$. If, additionally, $\mathbf u$ has a density strictly positive a.e. on $\mathcal U$, then the critical set of $h$ is also Lebesgue-null.
\end{restatable}

\smallskip
\begin{proof}[\textbf{Proof of \Cref{lem:uniformizing}}]
Let \(K=\{\mathbf a\in\mathcal U:\det J(h)(\mathbf a)=0\}\) be the critical set of $h$. By Sard's theorem \citep{sard1942measure}, its image $h(K)$ has Lebesgue measure zero. Absolute continuity of the law of $\mathbf v$ therefore gives
\[
\Pr(\mathbf u\in K)
\leq
\Pr\bigl(h(\mathbf u)\in h(K)\bigr)
=0.
\]
At every point outside $K$, the inverse function theorem gives open neighborhoods on which $h$ restricts to a diffeomorphism.

Finally, if the density of $\mathbf u$ is strictly positive a.e. on $\mathcal U$, every measurable subset of $\mathcal U$ with positive Lebesgue measure has positive probability. Since $\Pr(\mathbf u\in K)=0$, the set $K$ is Lebesgue-null.
\end{proof}

\subsection{Proof of Theorem 4.1}
\label{sec:proofthm}
We now proceed to prove \Cref{thm:idsemantics}. To begin, we restate the theorem for clarity:

\smallskip
\idsemantics*

\smallskip
\begin{proof}[\textbf{Proof of \Cref{thm:idsemantics}}]
We first establish almost-sure cross-modal alignment and uniform representation marginals at every global minimizer. Steps~2--4 then show that both representations factor through a common smooth function of the retained semantic block $\mathbf s_{\mathbb I_\rho^c}$. Finally, \Cref{lem:uniformizing} establishes local invertibility a.s.

\smallskip
\textbf{Step 1} (Global minimum and invariance condition)\textbf{.} 
Let \(f_x: \mathcal{X} \to (0, 1)^{n}\) and \(f_t: \mathcal{T}\sx{(\theta)} \to (0, 1)^{n}\) be any smooth functions attaining the global minimum. Define the smooth mappings:
\[
 h_x = f_x \circ g_x, \quad  h_t = f_t \circ g_{t\sx{(\theta)}}.
\]
By \Cref{lem:globalminimum}, the global minimum is zero. Under the stated entropy convention, the alignment term and both negative-entropy terms are nonnegative. Hence each term equals zero at every global minimizer, giving:
\begin{gather}
    \E_{(\mathbf{x}, \mathbf{t}\sx{(\theta)}) \sim p_{\mathbf{x}}\,p_{\mathbf{t}\sx{(\theta)}|\mathbf{x}}} \left[\|f_x(\mathbf{x}) - f_t(\mathbf{t}\sx{(\theta)})\|_2\right] 
    = \E_{(\mathbf{z}_x, \mathbf{z}_{t\sx{(\theta)}}) \sim p_{\mathbf{z}_x}p_{\mathbf{z}_{t\sx{(\theta)}}|\mathbf{z}_x}} \left[\| h_x(\mathbf{z}_x) -  h_t(\mathbf{z}_{t\sx{(\theta)}})\|_2\right] = 0, \label{eq:loss_align} \\
    H\bigl(f_x(\mathbf{x})\bigr) = H\bigl( h_x(\mathbf{z}_x)\bigr) = 0, \label{eq:img_uniform} \\
    H\bigl(f_t(\mathbf{t}\sx{(\theta)})\bigr) = H\bigl( h_t(\mathbf{z}_{t\sx{(\theta)}})\bigr) = 0. \label{eq:text_uniform}
\end{gather}


From \cref{eq:loss_align}, it follows that
\begin{equation}
\label{eq:inv_condition}
\begin{gathered}
h_t(\mathbf z_{t\sx{(\theta)}})
=
h_x(\mathbf z_x)
\qquad
\bigl(
p_{\tilde{\mathbf s}_{\mathbb I_\theta}
\mid\mathbf s,\theta,\rho}\,p_{\mathbf m_t}
\bigr)\text{-a.s.},
\qquad
\text{for }
\bigl(p_{\mathbf s}\,p_{\mathbf m_x}\bigr)
\text{-a.e. }
\mathbf z_x=(\mathbf s,\mathbf m_x).
\end{gathered}
\end{equation}
Thus, the representations align a.s. under the joint pairing law. The equalities of expectations in \Cref{eq:loss_align} follow because $\mathbf x$ and $\mathbf t\sx{(\theta)}$ have the pushforward distributions induced by $g_x$ and $g_{t\sx{(\theta)}}$, respectively (\Cref{eq:img_gen,eq:caption_gen}).

\Cref{eq:img_uniform,eq:text_uniform} imply that \( h_x\) and \( h_t\) map the latent variables \(\mathbf{z}_x\) and \(\mathbf{z}_{t\sx{(\theta)}}\) onto uniform distributions over \((0, 1)^{n}\) (with \(n=|\mathbb{I}^c_\rho|\)), since their differential entropies equal zero.

\smallskip
\textbf{Step 2}
(Exclusion of modality-specific variables)\textbf{.}
Fix any $A$ with $p_A(A)>0$ and condition \Cref{eq:inv_condition} on this event. First suppose that $A\ne\emptyset$. Define the open set
\[
\begin{aligned}
\Omega_A
:=\bigl\{
(\mathbf s,\mathbf m_x,\mathbf m_t,\mathbf w_A)
\in
\mathcal S\times\mathcal M_x\times\mathcal M_t
\times\mathcal S_A:
(\mathbf s_{\mathbb I_\theta\setminus A},\mathbf w_A)
\in\mathcal S_{\mathbb I_\theta}
\bigr\}.
\end{aligned}
\]
Here the assembled semantic vector has coordinates $\mathbf w_A$ on $A$ and agrees with $\mathbf s$ on $\mathbb I_\theta\setminus A$. 

By the independence assumptions in the LVM and \Cref{asm:pairing}, the conditional law of the free coordinates given $A$ has density
\[
p_{\mathbf s}(\mathbf s)\,
p_{\mathbf m_x}(\mathbf m_x)\,
p_{\mathbf m_t}(\mathbf m_t)\,
p_{\tilde{\mathbf s}_A\mid\mathbf s_A}
(\mathbf w_A\mid\mathbf s_A),
\]
which is positive a.e. on $\Omega_A$ by \Cref{asm:continuous,asm:pairing}. Define the continuous function
\[
\begin{aligned}
\Delta_A(\mathbf s,\mathbf m_x,\mathbf m_t,\mathbf w_A)
:={}h_x(\mathbf s,\mathbf m_x)
-h_t\bigl(
(\mathbf s_{\mathbb I_\theta\setminus A},\mathbf w_A),
\mathbf m_t
\bigr).
\end{aligned}
\]
By \Cref{eq:inv_condition}, $\Delta_A=0$ a.s. under this conditional law. Consequently, $\Delta_A=0$ everywhere on $\Omega_A$. Indeed, if it were nonzero at some point, continuity would give a nonempty open neighborhood within $\Omega_A$ on which its norm is bounded away from zero. Positivity of the density would give that neighborhood positive conditional probability, contradicting almost-sure alignment.

For every $(\mathbf s,\mathbf m_x,\mathbf m_t)\in\mathcal S\times\mathcal M_x\times\mathcal M_t$, the point $(\mathbf s,\mathbf m_x,\mathbf m_t,\mathbf s_A)$ belongs to $\Omega_A$. Evaluating the pointwise identity there yields
\[
h_x(\mathbf s,\mathbf m_x)
=
h_t(\mathbf s_{\mathbb I_\theta},\mathbf m_t).
\]
If $A=\emptyset$, the same identity follows directly from conditional alignment and the same positivity-and-continuity argument on $\mathcal S\times\mathcal M_x\times\mathcal M_t$.

Fix arbitrary reference points $\mathbf m_x^0\in\mathcal M_x$ and $\mathbf m_t^0\in\mathcal M_t$, and define
\[
\bar h_x(\mathbf s)
:=h_x(\mathbf s,\mathbf m_x^0),
\qquad
\bar h_t(\mathbf a)
:=h_t(\mathbf a,\mathbf m_t^0).
\]
These maps are smooth. The preceding identity, holding for every pair of modality-specific arguments, implies $h_x(\mathbf s,\mathbf m_x)=\bar h_x(\mathbf s)$. Moreover, every $\mathbf a\in\mathcal S_{\mathbb I_\theta}$ is the projection of some $\mathbf s\in\mathcal S$, so it also implies $h_t(\mathbf a,\mathbf m_t)=\bar h_t(\mathbf a)$. Therefore,
\begin{align}
h_x(\mathbf z_x)
&=\bar h_x(\mathbf s),
\label{eq:no_mx_in_x}\\
h_t(\mathbf z_{t\sx{(\theta)}})
&=\bar h_t(\tilde{\mathbf s}_{\mathbb I_\theta}).
\label{eq:no_mt_in_t}
\end{align}
These identities hold throughout the respective latent domains.

\smallskip
\textbf{Step 3}
(Exclusion of omitted semantic variables)\textbf{.} The identity established in Step~2 gives 
\[
\bar h_x(\mathbf s)
=
\bar h_t(\mathbf s_{\mathbb I_\theta})
\qquad\text{for every }\mathbf s\in\mathcal S.
\]
Define the smooth map $F=\bar h_t:\mathcal S_{\mathbb I_\theta}\to(0,1)^n$. Then, throughout the image latent domain,
\begin{equation}
h_x(\mathbf z_x)
=
F(\mathbf s_{\mathbb I_\theta}).
\label{eq:no_omit_x}
\end{equation}
In particular, for any $\mathbf s,\mathbf s'\in\mathcal S$ with $\mathbf s_{\mathbb I_\theta}
=\mathbf s'_{\mathbb I_\theta}$,
\[
\bar h_x(\mathbf s)
=
F(\mathbf s_{\mathbb I_\theta})
=
F(\mathbf s'_{\mathbb I_\theta})
=
\bar h_x(\mathbf s').
\]
Thus the image representation depends on the semantic variables only through the selected block. Substituting the factorizations from Step~2 and \Cref{eq:no_omit_x} into \Cref{eq:inv_condition} yields
\begin{equation}
\label{eq:inv_selected}
\begin{gathered}
F(\mathbf s_{\mathbb I_\theta})
=
F(\tilde{\mathbf s}_{\mathbb I_\theta})
\qquad
p_{\tilde{\mathbf s}_{\mathbb I_\theta}
\mid\mathbf s,\theta,\rho}\text{-a.s.},
\qquad
\text{for }p_{\mathbf s}\text{-a.e. }\mathbf s.
\end{gathered}
\end{equation}

\smallskip
\textbf{Step 4}
(Exclusion of perturbed semantic variables)\textbf{.}
The selected semantic vector has law $p_{\mathbf s_{\mathbb I_\theta}}$, whose density is positive a.e. on $\mathcal S_{\mathbb I_\theta}$. For each nonempty $A$ with $p_A(A)>0$, conditioning \Cref{eq:inv_selected} and the support-compatibility statement in \Cref{asm:pairing} on this event preserves both statements a.s. Since $A$ is sampled independently of the original latent variables, the original selected vector still has law $p_{\mathbf s_{\mathbb I_\theta}}$ under this conditioning.

For each such $A\subseteq\mathbb I_\rho$, choose a fixed density $q_A$ positive a.e. on $\mathcal S_A$. By the full-support condition in \Cref{asm:pairing}, $q_A$ and the original conditional perturbation density have the same null sets on $\mathcal S_A$ for each source value $\mathbf s_A$. Replacing the original kernel by $q_A$ therefore preserves both the alignment identity and the almost-sure requirement that the perturbed vector belongs to $\mathcal S_{\mathbb I_\theta}$, for $p_{\mathbf s_{\mathbb I_\theta}}$-a.e. original selected vector.

To verify that successive resamplings remain valid, let $\mu\ll p_{\mathbf s_{\mathbb I_\theta}}$ be an input law. Since $\mu$ has a Lebesgue density, its marginal on the unresampled coordinates has a density $r_{\mathbb I_\theta\setminus A}$. Independent resampling of the coordinates in $A$ with density $q_A$ produces the density
\[
r_{\mathbb I_\theta\setminus A}(\mathbf a_{\mathbb I_\theta\setminus A})\,q_A(\mathbf a_A),
\]
with these densities understood to be zero outside their respective domains. The output belongs to $\mathcal S_{\mathbb I_\theta}$ a.s. by support compatibility, which also holds for $\mu\ll p_{\mathbf s_{\mathbb I_\theta}}$. Since $p_{\mathbf s_{\mathbb I_\theta}}$ has positive density a.e. on $\mathcal S_{\mathbb I_\theta}$, the output law is again absolutely continuous with respect to $p_{\mathbf s_{\mathbb I_\theta}}$.

We may therefore apply these resamplings successively along a finite sequence $(A_1,\ldots,A_m)$ of positive-probability sets whose union is $\mathbb I_\rho$, using independent fresh draws at each step. Such a list exists by the coverage condition in \Cref{asm:pairing}. Every intermediate law remains absolutely continuous with respect to $p_{\mathbf s_{\mathbb I_\theta}}$, and the representation remains unchanged a.s. at every step.

Write $\mathbf c=\mathbf s_{\mathbb I_\rho^c}$ and $\mathbf v=\mathbf s_{\mathbb I_\rho}$. After these resamplings, every coordinate of $\mathbf v$ has been overwritten. Each final coordinate is determined by the last fresh draw that updates it, so the resulting vector $\mathbf v'$ depends only on the fresh draws and is independent of $(\mathbf c,\mathbf v)$. Consequently,
\[
F(\mathbf c,\mathbf v)=F(\mathbf c,\mathbf v')\quad\text{a.s.}
\]
Since $F$ is bounded, taking conditional expectations given $(\mathbf c,\mathbf v)$ and using this independence gives
\[
\begin{aligned}
F(\mathbf c,\mathbf v)
=\E\!\left[F(\mathbf c,\mathbf v')\mid\mathbf c,\mathbf v\right]=\E\!\left[F(\mathbf c,\mathbf v')\mid\mathbf c\right]
=:h_0(\mathbf c)
\quad\text{a.s.}
\end{aligned}
\]
Here $h_0$ is a measurable function of the retained block. If $\mathbb I_\rho=\emptyset$, this factorization is immediate with $h_0=F$.

The reduced map has a smooth version. Since $p_{\mathbf s_{\mathbb I_\theta}}$ has a density positive a.e. on $\mathcal S_{\mathbb I_\theta}$, the equality $F(\mathbf c,\mathbf v)=h_0(\mathbf c)$ holds Lebesgue-a.e. on $\mathcal S_{\mathbb I_\theta}$. On any product neighborhood $\mathcal C\times\mathcal V\subseteq\mathcal S_{\mathbb I_\theta}$, Fubini's theorem gives this equality for a.e. $\mathbf v\in\mathcal V$, for a.e. $\mathbf c\in\mathcal C$. For each such $\mathbf c$, continuity in $\mathbf v$ extends the equality to every $\mathbf v\in\mathcal V$. Consequently, for any fixed $\mathbf v_0\in\mathcal V$, the smooth section $F(\cdot,\mathbf v_0)$ equals $h_0$ a.e. on $\mathcal C$.

Any two such sections agree a.e. on their overlapping domains and hence everywhere there by continuity. They therefore define a smooth map $h:\mathcal S_{\mathbb I_\rho^c}\to(0,1)^n$. Choosing a section through each point of $\mathcal S_{\mathbb I_\theta}$ shows that $F(\mathbf c,\mathbf v)=h(\mathbf c)$ throughout $\mathcal S_{\mathbb I_\theta}$.

Combining this factorization with \cref{eq:no_omit_x,eq:no_mt_in_t,eq:inv_selected} gives
\begin{equation}
\label{eq:retained_x_reps}
h_x(\mathbf z_x)
=
h(\mathbf s_{\mathbb I_\rho^c})
\quad
\bigl(p_{\mathbf s}p_{\mathbf m_x}\bigr)\text{-a.s.},
\end{equation}
and
\begin{equation}
\label{eq:retained_t_reps}
\begin{gathered}
h_t(\mathbf z_{t\sx{(\theta)}})
=
h(\mathbf s_{\mathbb I_\rho^c})
\qquad
\bigl(
p_{\tilde{\mathbf s}_{\mathbb I_\theta}\mid\mathbf s,\theta,\rho}
p_{\mathbf m_t}
\bigr)\text{-a.s.},
\qquad
\text{for }p_{\mathbf s}\text{-a.e. }\mathbf s.
\end{gathered}
\end{equation}

\smallskip
\textbf{Step 5}
(Local recovery of the retained semantic block)\textbf{.}
By \cref{eq:retained_x_reps,eq:retained_t_reps}, both representations equal $h(\mathbf s_{\mathbb I_\rho^c})$ a.s. The domain $\mathcal S_{\mathbb I_\rho^c}$ is open as a coordinate projection of the open set $\mathcal S$, and its dimension is $n$. By \cref{eq:img_uniform,eq:text_uniform}, $h(\mathbf s_{\mathbb I_\rho^c})$ is uniform on $(0,1)^n$. Applying \Cref{lem:uniformizing} therefore yields
\[
\Pr\!\left(
\det J(h)(\mathbf s_{\mathbb I_\rho^c})=0
\right)=0.
\]
Taking $h$ as the reduced map for both modalities in \Cref{def:blockid}, we conclude that $f_x$ and $f_t$ locally block-identify $\mathbf s_{\mathbb I_\rho^c}$ a.s., completing the proof.
\end{proof}

\begin{formalbox}
\begin{restatable}[Causal interpretations]{remark}{causal}
\label{clar:causal}
The exclusion established above is functional: both representations depend on the latent variables only through $\mathbf s_{\mathbb I_\rho^c}$. This does not imply statistical independence from omitted or perturbed semantic variables, which may remain correlated with the retained block. The proof permits dependence among semantic variables and uses variation within the support of the observational pairing law, without requiring identification of the underlying causal structure.
\end{restatable}
\end{formalbox}


\bigskip
\subsection{Proof of Corollary 4.1}
We now proceed to prove \Cref{cor:allsemantics}. To begin, we restate the corollary for clarity:
\allsemantics*
\begin{proof}
Under the graded lexicographic orders specified in \Cref{defn:selection,defn:perturbation}, setting $\theta=2^{n_s}-1$ selects the full semantic index set $\mathbb I_\theta=\mathbb I_{\mathbf s}$, while setting $\rho=1$ gives $\mathbb I_\rho=\emptyset$. Consequently,
\[
\mathbb I_\rho^c
=
\mathbb I_\theta\setminus\mathbb I_\rho
=
\mathbb I_{\mathbf s},
\qquad
n=|\mathbb I_\rho^c|=n_s.
\]
The conditions of \Cref{thm:idsemantics} hold by assumption. Applying that theorem with these choices shows that every pair of smooth encoders $f_x:\mathcal X\to(0,1)^{n_s}$ and $f_t:\mathcal T\sx{(\theta)}\to(0,1)^{n_s}$ attaining the global minimum of the population surrogate $\mathcal L_{\text{\tiny SymAlignMaxEnt}}$ locally block-identifies $\mathbf s_{\mathbb I_\rho^c}=\mathbf s$ a.s.
\end{proof}

\begin{restatable}[Graded lexicographic order]{defn}{lexico_order}
\label{defn:order}
Let $S$ be a finite set with a total order (e.g., $S = \{1, 2, \dots, n\}$). The \emph{graded lexicographic order} on the non-empty subsets of $S$ is defined as follows: \textbf{(i)} Subsets are ordered first by increasing cardinality (i.e., all subsets of size $k$ precede those of size $k+1$). \textbf{(ii)} Within each group of equal cardinality, subsets are ordered lexicographically: that is, $\{i_1, i_2, \dots, i_k\} < \{j_1, j_2, \dots, j_k\}$ if there exists an $\ell$ such that $i_m = j_m$ for all $m < \ell$ and $i_\ell < j_\ell$.
\end{restatable}

\begin{formalbox}
\begin{restatable}[Fixed graded lexicographic order]{remark}{fixedorder}
\label{clar:fixedorder}
The \emph{graded lexicographic order} over subsets of semantic indices is fixed solely for notational consistency and clarity. For example, given semantic variables \(\mathbf{s} = (s_1, s_2, s_3)\) with index set \(\mathbb{I}_\mathbf{s} = \{1, 2, 3\}\), the graded lexicographic order over the non-empty powerset \(\mathcal{P}^+(\mathbb{I}_\mathbf{s})\) yields:
\[
\mathcal{P}^+(\mathbb{I}_\mathbf{s}) = \{\{1\}, \{2\}, \{3\}, \{1,2\}, \{1,3\}, \{2,3\}, \{1,2,3\}\}.
\]
This ordering imposes no constraints on the latent structure and is adopted without loss of generality. Since the true permutation of latent variables is unidentifiable, any consistent order can be used to index subset-valued parameters such as \(\theta\) (selection) and \(\rho\) (perturbation). Importantly, both \(\theta\) and \(\rho\) act indirectly through text observations rather than directly operating on the latent space. The fixed order simply determines how each value of \(\theta\) or \(\rho\) maps to a subset of semantic indices. For instance, omitting color in a caption corresponds to removing the associated latent variable---\emph{which} index of \(\theta\) encodes this depends on the fixed ordering. Thus, adopting a graded lexicographic order ensures a reproducible indexing scheme without limiting model generality.
\end{restatable}
\end{formalbox}

\bigskip
\subsection{Proof of Corollary 4.2}
We now proceed to prove \Cref{cor:invsemantics}. For clarity, we restate the corollary below:
\invsemantics*

\begin{proof}
By assumption, $\mathbb I_{var}=\mathbb I_\theta^c\cup\mathbb I_\rho$. Hence the selected, unperturbed semantic index set satisfies
\[
\begin{aligned}
\mathbb I_\rho^c
&=\mathbb I_\theta\setminus\mathbb I_\rho\\
&=\mathbb I_{\mathbf s}\setminus
  (\mathbb I_\theta^c\cup\mathbb I_\rho)\\
&=\mathbb I_{\mathbf s}\setminus\mathbb I_{var}
=\mathbb I_{inv}.
\end{aligned}
\]
The conditions of \Cref{thm:idsemantics} hold for the training data-generating process by assumption, and the retained block has dimension $n=|\mathbb I_{inv}|$. Applying that theorem shows that every pair of smooth encoders $f_x:\mathcal X\to(0,1)^{|\mathbb I_{inv}|}$ and $f_t:\mathcal T\sx{(\theta)}\to(0,1)^{|\mathbb I_{inv}|}$ attaining the global minimum of the population surrogate $\mathcal L_{\text{\tiny SymAlignMaxEnt}}$ locally block-identifies $\mathbf s_{\mathbb I_\rho^c}=\mathbf s_{\mathbb I_{inv}}$ a.s. under the training distribution.
\end{proof}



\section{Scope of the Recovery and Objective Claims}
\label{sec:recovery_objective_scope}

\subsection{Large-Batch Contrastive Objective}
\label{sec:objective_scope}

For fixed encoders with bounded outputs, i.i.d.\ paired batches, $\tau=1$, and $\kappa(\mathbf a,\mathbf b)=-\|\mathbf a-\mathbf b\|_2^2$, the expected objective in \Cref{eq:loss_exp} satisfies
\begin{align}
\lim_{K\to\infty}
\left(
\E[\mathcal L_{\text{\tiny MMCL}}^{(K)}(f_x,f_t)]
-\log K
\right)
\nonumber
&=
\E_{(\mathbf x,\mathbf{t})}
\bigl[\|f_x(\mathbf x)-f_t(\mathbf t)\|_2^2\bigr]
\nonumber\\
&\quad+
\frac12\E_{\mathbf x}
\log\E_{\mathbf t'\sim p_{\mathbf t}}
\exp\!\left\{
-\|f_x(\mathbf x)-f_t(\mathbf t')\|_2^2
\right\}
\nonumber\\
&\quad+
\frac12\E_{\mathbf t}
\log\E_{\mathbf x'\sim p_{\mathbf x}}
\exp\!\left\{
-\|f_x(\mathbf x')-f_t(\mathbf t)\|_2^2
\right\}.
\label{eq:mmcl_population_limit}
\end{align}
Here, $\mathcal L_{\text{\tiny MMCL}}^{(K)}$ denotes the objective evaluated on a batch of $K$ pairs. The expectation on the left is over batches, and primes denote independent draws from the indicated marginals. Subtracting $\log K$ replaces each denominator sum by its average. Conditional on the anchor, the law of large numbers gives the corresponding kernel average, with the paired positive contributing a term of order $K^{-1}$. Because the encoder outputs are bounded, the exponential kernels remain bounded away from zero, justifying passage of the limit through the logarithm and outer expectation.

Unlike the marginal entropies in \Cref{eq:loss_thm}, the last two terms involve fixed-bandwidth kernel averages. Further, the limiting alignment term uses a squared Euclidean distance rather than the unsquared norm in \Cref{eq:loss_thm}---a distinction consistent with the contrastive limit and kernel-density perspective analyzed by \cite{wang2020understanding}. Because neither objective equality nor minimizer equivalence follows from this large-batch limit, we adopt \Cref{eq:loss_thm} as a dedicated population surrogate for our theoretical analysis.

\subsection{Local versus Global Recovery}
\label{sec:recovery_scope}

Our proof adapts the entropy-based strategy developed in prior identifiability analyses \citep{zimmermann2021contrastive,von2021self,lyu2022understanding,daunhawer2022identifiability,eastwood2023self,yao2023multi}. We note a distinction regarding the scope of this technique: without additional topological or global assumptions, smoothness of the induced map on the retained block and a uniform pushforward distribution alone do not guarantee global invertibility (nor injectivity on a set of full probability). Instead, as established in the revised \Cref{lem:uniformizing}, these conditions certify local invertibility around almost every latent point. We therefore state the recovery guarantees in \Cref{thm:idsemantics} and \Cref{cor:allsemantics,cor:invsemantics} in terms of almost-sure local block recovery, updating an earlier version of this manuscript that adopted the global-recovery formulation used in this line of identifiability research.

The distinction between local and global recovery persists even for smooth maps with everywhere nonzero Jacobian determinants and continuous input and output variables. Let $\mathbf{u}=(u_1,u_2)$ be uniformly distributed on $\mathcal{Q}=(0,1)^2$ and, for an integer $k\geq2$, define
\[
w_k(\mathbf{u}) = \bigl(\sqrt{u_1}\cos(2\pi k u_2), \sqrt{u_1}\sin(2\pi k u_2)\bigr).
\]
This map is smooth on $\mathcal{Q}$, satisfies $\det J(w_k)(\mathbf{u})=\pi k$ everywhere, and maps $\mathcal{Q}$ onto the punctured open unit disk $\mathcal{D}\setminus\{\mathbf{0}\}$, where \(\mathcal{D} = \left\{\mathbf{v}=(v_1,v_2)\in\mathbb{R}^2: v_1^2+v_2^2<1\right\}\). On each open strip
\[
(0,1)\times(j/k,(j+1)/k), \qquad j=0,\ldots,k-1,
\]
$w_k$ is a diffeomorphism onto $\mathcal{D}$ minus the nonnegative horizontal radius. Each branch contributes density $1/(\pi k)$, so the $k$ branches sum to $1/\pi$ a.e. on $\mathcal{D}$. The disk center and the branch seams are null sets that do not affect the pushforward distribution. Thus, $w_k(\mathbf{u})$ is uniformly distributed on $\mathcal{D}$.

Next, apply the Rosenblatt transformation \citep{rosenblatt1952remarks} to the uniform distribution on $\mathcal{D}$. In this case, the conditional-CDF diffeomorphism $r:\mathcal{D}\to\mathcal{Q}$ is given by
\[
F(v_1) = \frac12+\frac{v_1\sqrt{1-v_1^2}+\arcsin v_1}{\pi}, \qquad r(\mathbf{v}) = \left(F(v_1), \frac12+\frac{v_2}{2\sqrt{1-v_1^2}}\right).
\]
The map $r$ is smooth, satisfies $\det J(r)(\mathbf{v})=1/\pi$ everywhere, and pushes the uniform disk law forward to the uniform square law. Consequently, the composite map
\(
h_k=r\circ w_k:\mathcal{Q}\to\mathcal{Q}
\)
is smooth on all of $\mathcal{Q}$, preserves the uniform distribution, and satisfies $\det J(h_k)(\mathbf{u})=k$ everywhere. Although
\(
h_k(\mathcal{Q}) = \mathcal{Q}\setminus\{(1/2,1/2)\},
\)
the omitted point has zero measure and does not affect the uniform pushforward law. Nevertheless,
\[
h_k(u_1,u_2) = h_k(u_1,u_2+1/k) \qquad \text{whenever }0<u_2<u_2+1/k<1,
\]
and $h_k$ is $k$-to-one a.e. Moreover, $h_k$ cannot be injective on any full-measure subset of $\mathcal{Q}$: translation by $1/k$ in the second coordinate preserves null sets, so almost every point $(u_1,u_2)$ in the first strip and its translated partner $(u_1,u_2+1/k)$ must both belong to that subset.

To instantiate this counterexample in our LVM, take $n_s=n=2$, $\mathcal{S}=\mathcal{Q}$, and $\mathbf{s}=\mathbf{u}$. Select all semantic coordinates and introduce no perturbation:
\[
\mathbb{I}_{\theta} = \mathbb{I}_{\mathbf{s}} = [2], \qquad \mathbb{I}_{\rho}=\emptyset, \qquad \mathbb{I}^c_{\rho}=\mathbb{I}_{\mathbf{s}}.
\]
Choose independent, continuous modality-specific variables $\mathbf{m}_x$ and $\mathbf{m}_t$ with smooth, strictly positive densities, independent of $\mathbf{s}$. Let $g_x$ and $g_{t\sx{(\theta)}}$ be the identity maps, so that
\[
\mathbf{x}=(\mathbf{s},\mathbf{m}_x), \qquad \mathbf{t}\sx{(\theta)}=(\mathbf{s},\mathbf{m}_t).
\]
Define the encoders to apply $h_k$ to the semantic coordinates while ignoring the modality-specific variables:
\[
f_x(\mathbf{x})=h_k(\mathbf{s}), \qquad f_t(\mathbf{t}\sx{(\theta)})=h_k(\mathbf{s}).
\]
Because the paired representations agree and both output marginals are uniform on $\mathcal{Q}$, we have
\[
\mathcal{L}_{\text{\tiny SymAlignMaxEnt}}(f_x,f_t)=0.
\]
These encoders therefore attain the population global minimum of \Cref{eq:loss_thm} without globally identifying the retained block $\mathbf{s}_{\mathbb{I}^c_{\rho}}$. All latent variables remain continuous, and the uniform distribution on the retained block satisfies the conditional-CDF regularity condition in \Cref{thm:idsemantics}. For $n>2$, apply $h_k$ to the first two coordinates of a uniform vector on $(0,1)^n$ and the identity map to the remaining coordinates. Accordingly, our revised results establish functional exclusion alongside almost-sure local recovery. Upgrading to global recovery requires additional structural restrictions to rule out multiple preimages. Uniformity and simple connectedness alone do not provide such restrictions.

\begin{formalbox}
\begin{remark}[A sufficient condition for global recovery]
Let $h:\mathcal{S}_{\mathbb{I}^c_{\rho}}\to(0,1)^n$ denote the smooth retained-block map, and suppose that $\mathcal{S}_{\mathbb{I}^c_{\rho}}$ is a connected open subset of $\mathbb{R}^n$. If $h$ is proper and its Jacobian determinant is nonzero everywhere, then $h$ is a global diffeomorphism, since $(0,1)^n$ is simply connected \citep[Theorem~2.1]{ruzhansky2015global}. Here, properness means that $h^{-1}(A)$ is compact for every compact set $A\subset(0,1)^n$. These conditions provide one route to global block recovery.  
\end{remark}
\end{formalbox}

\begin{formalbox}
\begin{remark}[Compactness and the choice of latent space]
A continuous map from a compact domain into a Hausdorff space is proper. However, compactness of the retained latent domain is incompatible with the exact uniformity required in our formulation. Indeed, if $\mathcal{S}_{\mathbb{I}^c_{\rho}}$ were compact, then its image under the continuous map $h:\mathcal{S}_{\mathbb{I}^c_{\rho}}\to(0,1)^n$ would also be compact. Hence, there would exist $\varepsilon>0$ such that
\(
h\bigl(\mathcal{S}_{\mathbb{I}^c_{\rho}}\bigr)
\subseteq [\varepsilon,1-\varepsilon]^n.
\)
The pushforward distribution could therefore not be uniform on the entire open cube. Assuming only that the latent domain is bounded, or that its closure is compact, does not ensure properness of a map defined on that open domain. Changing the latent space to a closed cube would require a separate treatment of boundary behavior, since the global inverse theorem above concerns open manifolds.
\end{remark}
\end{formalbox}

\begin{formalbox}
\begin{remark}[Relation to hyperspherical results]
In the hyperspherical setting of \cite{zimmermann2021contrastive}, both the latent variable and its representation lie on spheres. Compactness of the source sphere ensures properness of any continuous latent-to-representation map. If the spheres have equal intrinsic dimension of at least two, a smooth map with an everywhere nonsingular differential is therefore globally invertible, since the target is simply connected \citep[Theorem~2.1]{ruzhansky2015global}. Their main recovery result \citep[Theorem~2]{zimmermann2021contrastive} uses a more specific argument: in the infinite-negative-sample limit, the specified positive-pair conditional distribution and matching model family yield inner-product preservation up to a fixed scale, and hence orthogonal recovery after rescaling. This relies on structure beyond marginal uniformity and does not establish global invertibility from smoothness and marginal uniformity alone in our open-cube setting.
\end{remark}
\end{formalbox}

\section{Numerical Simulation Details}
\label{sec:numeric_details}

We provide additional details on the numerical simulations that are not fully covered in \Cref{sec:numeric_exp}. Specifically:

\begin{itemize}[left=1.8em]
    \item In \Cref{sec:num_set_details}, we outline the experimental setup, including hyperparameters, model architectures, and the construction of downstream tasks.
    \item In \Cref{sec:num_ident_details}, we present additional experiments that further validate our theoretical findings.
    \item In \Cref{sec:num_task_details}, we analyze downstream performance under various perturbation bias settings.
\end{itemize}

\subsection{Detailed Experimental Setup}
\label{sec:num_set_details}

\paragraph{Latent space construction.} 
We define a latent space comprising a semantic subspace \(\mathcal{S} \subseteq \mathbb{R}^{10}\) and two modality-specific subspaces \(\mathcal{M}_x, \mathcal{M}_t \subseteq \mathbb{R}^5\). Variables are sampled from Gaussian priors: 
\[
\mathbf{s}\sim \mathcal{N}(0, \Sigma_{\mathbf{s}}),\quad \mathbf{m}_x \sim \mathcal{N}(0, \Sigma_{\mathbf{m}_x}), \quad \mathbf{m}_t \sim \mathcal{N}(0, \Sigma_{\mathbf{m}_t}).
\]
For \emph{independent} semantic variables, \(\Sigma_{\mathbf{s}} = I_{10}\) yields a factorized standard Gaussian, aligning with nonlinear ICA assumptions \citep{khemakhem2020variational, sorrenson2020disentanglement}. For \emph{dependent} semantics, we sample \(\Sigma_{\mathbf{s}}\) from Wishart distribution \(\mathcal{W}_{10}(I, 10)\) (identity scale, 10 degrees of freedom), introducing latent dependencies as in \citep{von2021self, daunhawer2022identifiability, yao2023multi}. Modality-specific covariances \(\Sigma_{\mathbf{m}_x}\) and \(\Sigma_{\mathbf{m}_t}\) are sampled similarly from \(\mathcal{W}_{5}(I, 5)\), the Wishart distribution with identity scale and 5 degrees of freedom.

\paragraph{Selection and perturbation biases.} 
Given the 10-dimensional semantic space, the total number of nonempty subsets is \(2^{10} - 1 = 1023\), defining the full range of \textit{selection bias} configurations \(\theta\). For tractability, we choose 10 representative selection subsets \(\mathbb{I}_\theta\), detailed in \Cref{tab:selection_bias}.
  
To isolate \textit{perturbation bias}, we fix full selection (\(\theta=1023\)) and evaluate 10 representative perturbation subsets \(\mathbb{I}_\rho\), as listed in \Cref{tab:perturbation_bias}. Perturbations are introduced by independently modifying each semantic variable \(s_i\) in \(\mathbb{I}_\rho\) with probability 0.75 using additive Gaussian noise:
\[
\tilde{s}_i = s_i + \mathcal{N}(0, \Sigma_{\epsilon, i}),
\]
where the joint noise covariance \(\Sigma_{\epsilon}\) is sampled from \(\mathcal{W}_{10}(I, 10)\). Perturbations are applied solely to the text modality, aligning the proposed LVM and assumptions stated in \Cref{sec:problemformulate}.

In addition to the isolated bias settings, we consider a \textit{joint bias} scenario to investigate compounding effects. Specifically, we select \(\mathbb{I}_\theta = [8]\) (corresponding \(\theta=968\), excluding the last two semantic indices), and apply perturbations to the first two indices (\(\mathbb{I}_\rho=[2]\), correspnding to \(\rho=12\)). 

Results across these 20 representative configurations (10 selection settings, 10 perturbation settings) and the joint bias case serve to validate the theoretical findings under a range of misalignment settings.

\begin{table}[tbp]
\caption{Selection bias settings and selected semantic indices.}
\label{tab:selection_bias}
\centering
    \begin{tabular}{@{}c|cccccccccc@{}}
    \toprule
       \(\theta\), \(\rho=1\)  & \(1\)   & \(11\)   & \(56\)   & \(176\)   & \(386\)   & \(638\)   & \(848\)   & \(968\)   & \(1013\)   & \(1023\) \\
    \midrule
       \(\mathbb{I}_{\theta}\)  & \(\{1\}\)   & \([2]\)   & \([3]\)   & \([4]\)  & \([5]\)   & \([6]\)    & \([7]\)   & \([8]\)   & \([9]\)   & \([10]\) \\
    \bottomrule
    \end{tabular}%
\end{table}

\begin{table}[tbp]
\caption{Perturbation bias settings and perturbable semantic indices.}
\label{tab:perturbation_bias}
\centering
    \begin{tabular}{@{}c|cccccccccc@{}}
    \toprule
       \(\rho\mid\theta=1023\)  & \(1\) & \(2\)  & \(12\)   & \(57\)   & \(177\)   & \(387\)   & \(639\)   & \(849\)   & \(969\)   & \(1014\)  \\
    \midrule
       \(\mathbb{I}_{\rho}\)  & \(\emptyset\) & \(\{1\}\)   & \([2]\)   & \([3]\)   & \([4]\)  & \([5]\)   & \([6]\)    & \([7]\)   & \([8]\)   & \([9]\)  \\
    \bottomrule
    \end{tabular}%
\end{table}

\paragraph{Parameter settings.} 
We parameterize the generative functions \(g_x\) and \(g_{t\sx{(\theta)}}\) using randomly initialized 3-layer invertible MLPs, following prior work~\citep{daunhawer2022identifiability}. Invertibility is enforced by maintaining a condition number threshold of \(1\mathrm{e}{-}3\) for each layer.

The encoding functions \(f_x\) and \(f_t\) are implemented as 7-layer MLPs and optimized using the Adam optimizer. For MMCL training, we use a batch size of \(6144\), a learning rate of \(1\mathrm{e}{-}4\), and train for \(100{,}000\) iterations. We use the symmetrized finite-batch contrastive loss described in \Cref{sec:numeric_exp}, with squared Euclidean distances, cross-modal positive pairs, independently sampled same-modality negatives, and temperature \(\tau=1\). To ensure training stability, gradients are clipped using a maximum 2-norm of \(2\).

All experiments are run with three distinct random seeds, and we report averaged \(R^2\) values, clipped to the interval \([0,1]\) for interpretability. A displayed zero can therefore include negative unclipped $R^2$ and should be interpreted as poor performance of the evaluated decoder.

\paragraph{Downstream tasks design.}
To evaluate pretrained representations under various bias conditions, we construct several downstream tasks. Specifically, four regression tasks are created by generating labels using complex nonlinear functions \(f_{y_i}\), each applied to different subsets of the true semantic variables:
\[
y_1 = f_{y_1}(\mathbf{s}_{[3]}), \quad
y_2 = f_{y_2}(\mathbf{s}_{[5]}), \quad
y_3 = f_{y_3}(\mathbf{s}_{[7]}), \quad
y_4 = f_{y_4}(\mathbf{s}_{[9]}).
\]
Each function \(f_{y_{(\cdot)}}: \mathbb{R}^d \rightarrow \mathbb{R}\) includes quadratic, cubic, pairwise, and triple-wise interaction terms, as well as sinusoidal, logarithmic, and exponential transformations. The full formulation for a semantic vector \(\mathbf{s}_{[d]}\) is:
\begin{gather*}
f_{y_{(\cdot)}}(\mathbf{s}_{[d]}) = \sum_{i=1}^d s_i^2 
+ 0.3 \sum_{i=1}^d s_i^3
+ 0.5 \sum_{1 \leq i < j \leq d} s_i s_j
+ 0.2 \sum_{1 \leq i < j < k \leq d} s_i s_j s_k \\
\quad + 0.7 \sum_{i=1}^d \left( \sin(s_i) + \cos(s_i) \right)
+ 0.4 \sum_{i=1}^d \log(1 + |s_i|)
+ 0.6 \sum_{i=1}^d e^{-|s_i|}.
\end{gather*}

For the classification, binary labels are obtained by fitting two-cluster k-means to \(y_2\) separately within each generated batch, with labels ordered by the cluster centers. To simulate distribution shifts in the observations \(\mathbf{x}\), we apply a skewed, heavy-tailed transformation to semantic dimensions \(9\) and \(10\):
\[
s_i^{\text{ood}} = 2\, \mathrm{sign}(s_i^{\text{id}}) \cdot |s_i^{\text{id}}|^2, \quad \text{for } i \in \{9, 10\}.
\]
After this transformation, each affected coordinate is independently replaced, with probability \(0.1\), by \(10\) times its sign. For both downstream tasks, we \emph{fix} the pretrained encoders and evaluate the quality of the learned representations using a two-layer MLP as a probing model. We generate \(20{,}480\) samples as the evaluation set for training the regressors and classifiers, along with an additional \(20{,}480\) samples as the in-distribution test set. To assess OOD generalization, we generate another \(20{,}480\) samples from the shifted latent space as the OOD test set.

The regressors are trained using Mean Squared Error (MSE) loss, and the classifiers use Cross-Entropy loss, both with a learning rate of \(10^{-3}\) and trained for \(10{,}000\) steps. We report classification performance using the average Matthews Correlation Coefficient (MCC). Importantly, in the OOD setting, we perform \emph{no adaptation or fine-tuning}; the classifier is evaluated directly to test the generalization capability of the pretrained representations.

\subsection{Additional Identification Results}
\label{sec:num_ident_details}

\paragraph{Results under perturbation bias.} 
As shown in \Cref{fig:numeric_nonlinear_perturb}, the results of predicting true latent semantic variables under various perturbation bias settings exhibit similar trends to those observed under selection bias. The evaluated decoders obtain low scores for modality-specific variables, high scores for unbiased semantic variables, and intermediate scores for some misaligned semantic variables in the dependent setting, with similar patterns across image and text representations. These results provide empirical evidence for the qualitative dependence of semantic predictability on the pairing mechanism.

\begin{figure}[tbp]
    \centering
    \includegraphics[width=0.271\textwidth]{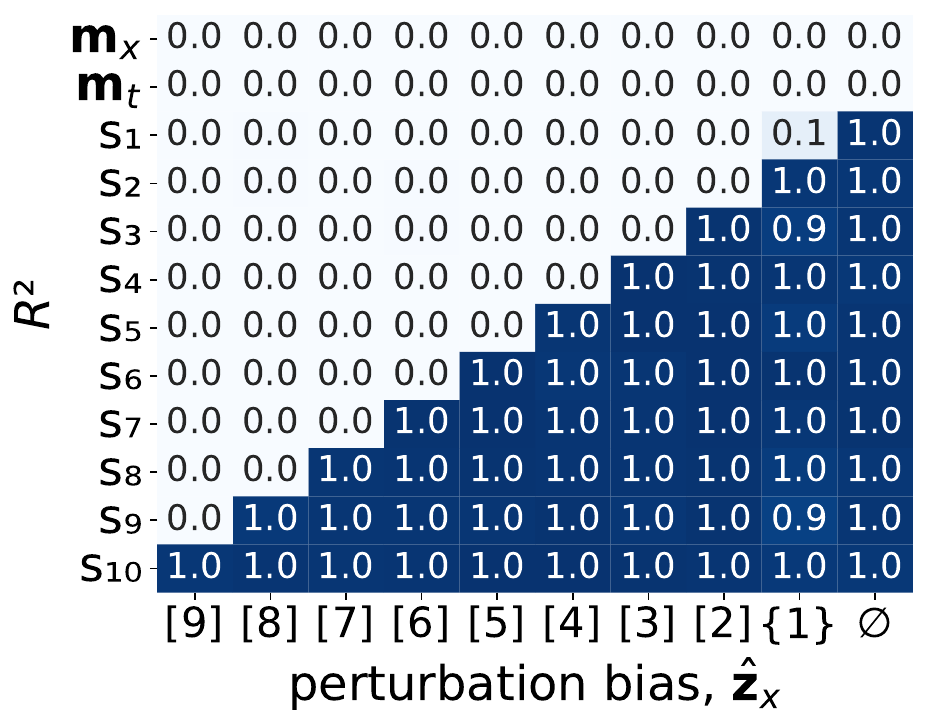}
    \hfill
    \includegraphics[width=0.228\textwidth]{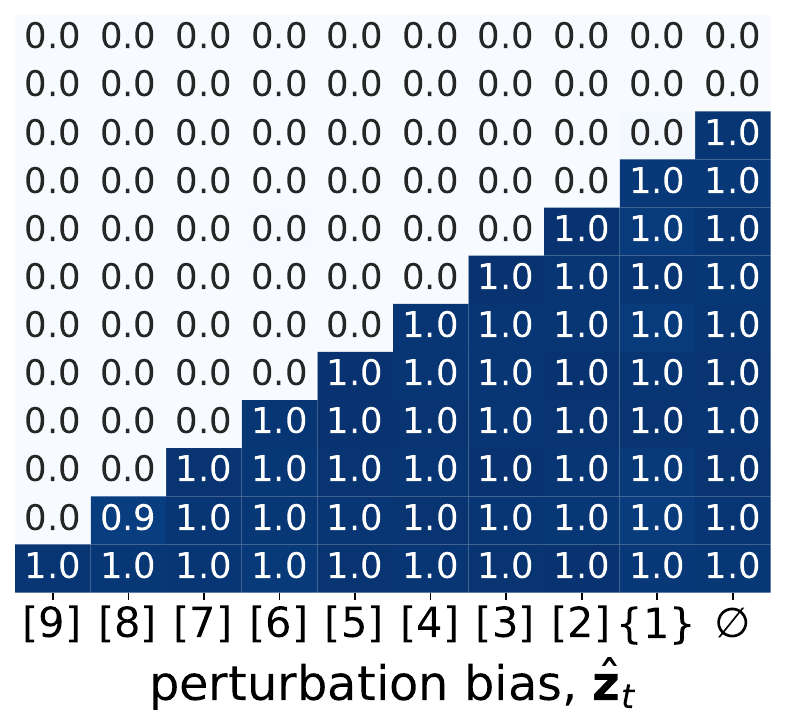}
    \hfill
    \includegraphics[width=0.228\textwidth]{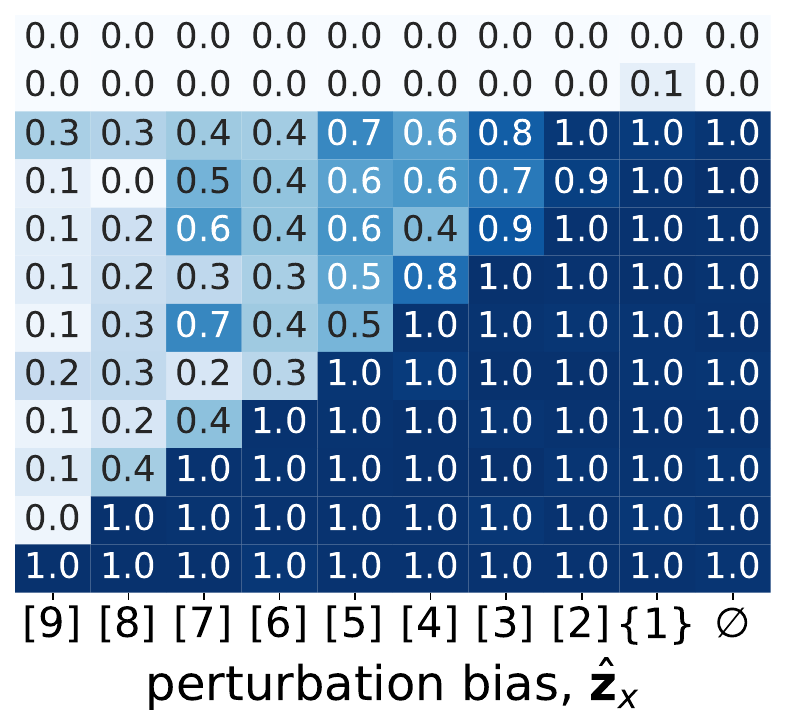}
    \hfill
    \includegraphics[width=0.228\textwidth]{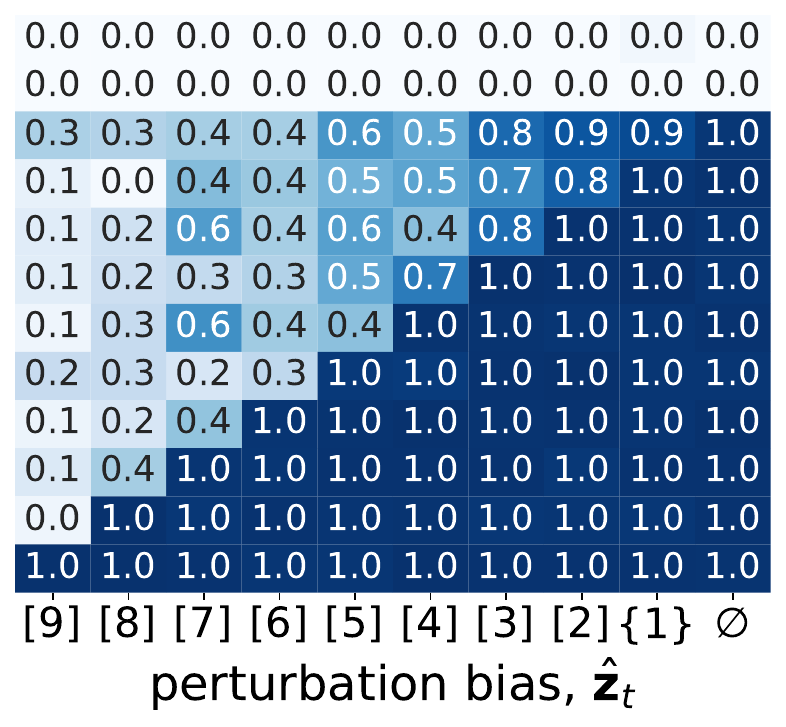}
    \caption{Mean \(R^2\) scores under perturbation bias settings. Left to right: predictions using representations \(\hat{\mathbf{z}}_x\) and \(\hat{\mathbf{z}}_t\) under independent latent semantic variables, followed by those under dependent latent semantic variables.}
    \label{fig:numeric_nonlinear_perturb}
\end{figure}

\begin{figure}[tbp]
    \centering
    \includegraphics[width=0.271\textwidth]{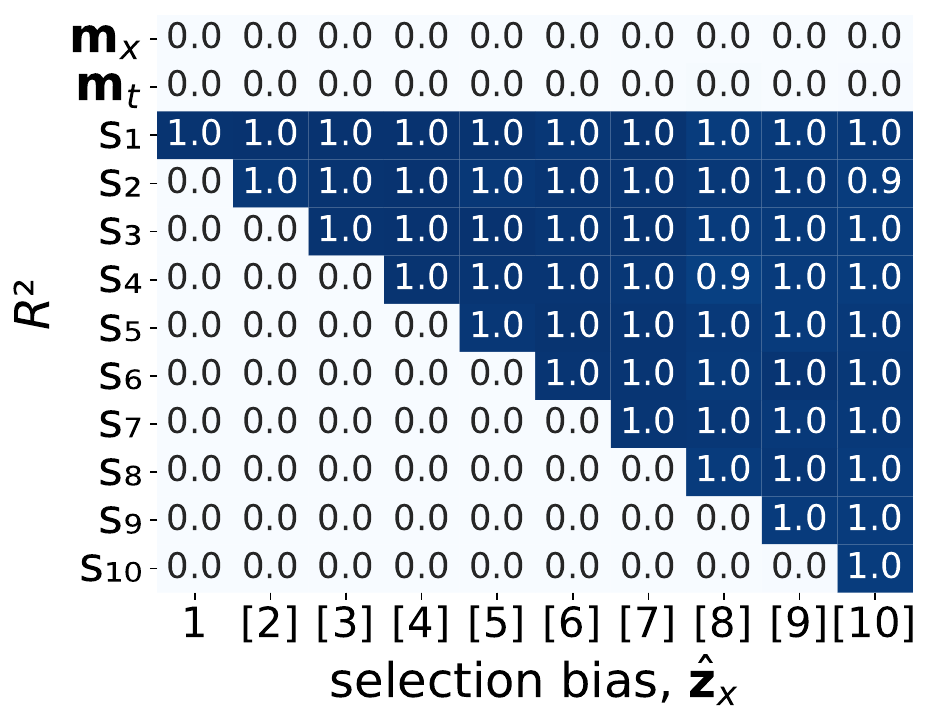}
    \hfill
    \includegraphics[width=0.228\textwidth]{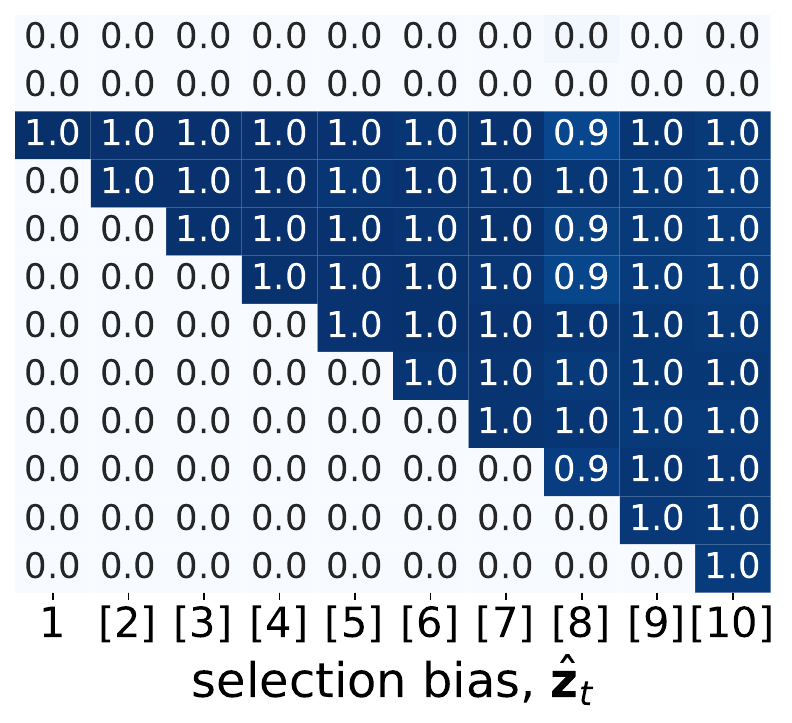}
    \hfill
    \includegraphics[width=0.228\textwidth]{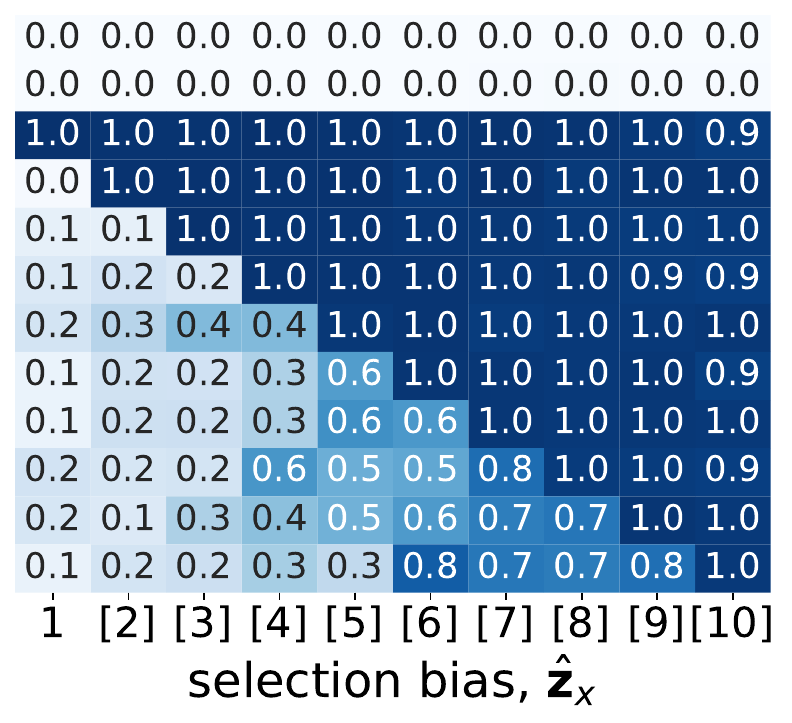}
    \hfill
    \includegraphics[width=0.228\textwidth]{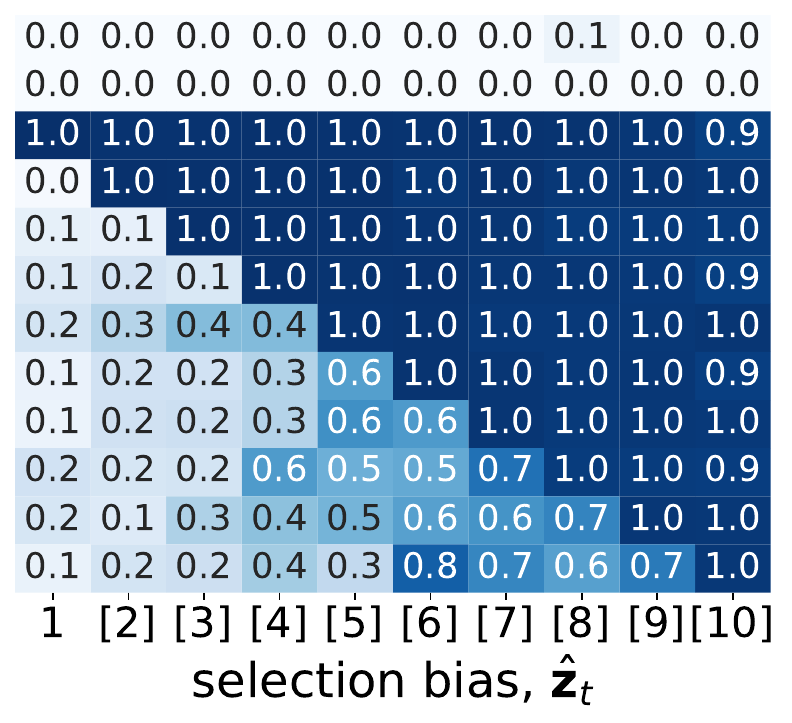}
    \caption{Evaluating linearity of learned representations under selection bias. Left to right: predictions using representations \(\hat{\mathbf{z}}_x\) and \(\hat{\mathbf{z}}_t\) under independent latent semantic variables, followed by those under dependent latent semantic variables.}
    \label{fig:numeric_linear}
\end{figure}

\begin{figure}[tbp]
    \centering
    \includegraphics[width=0.271\textwidth]{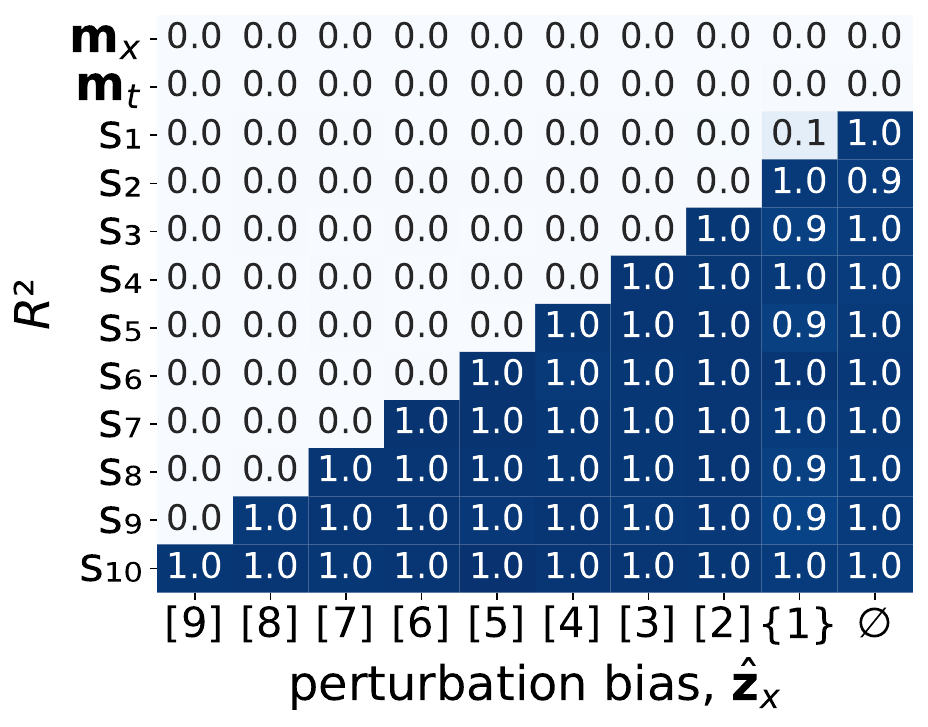}
    \hfill
    \includegraphics[width=0.228\textwidth]{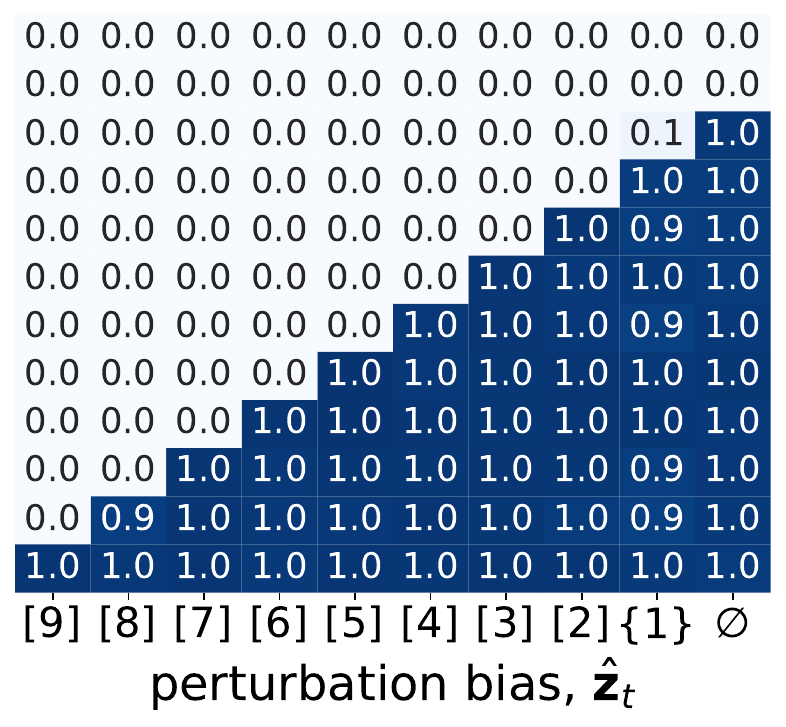}
    \hfill
    \includegraphics[width=0.228\textwidth]{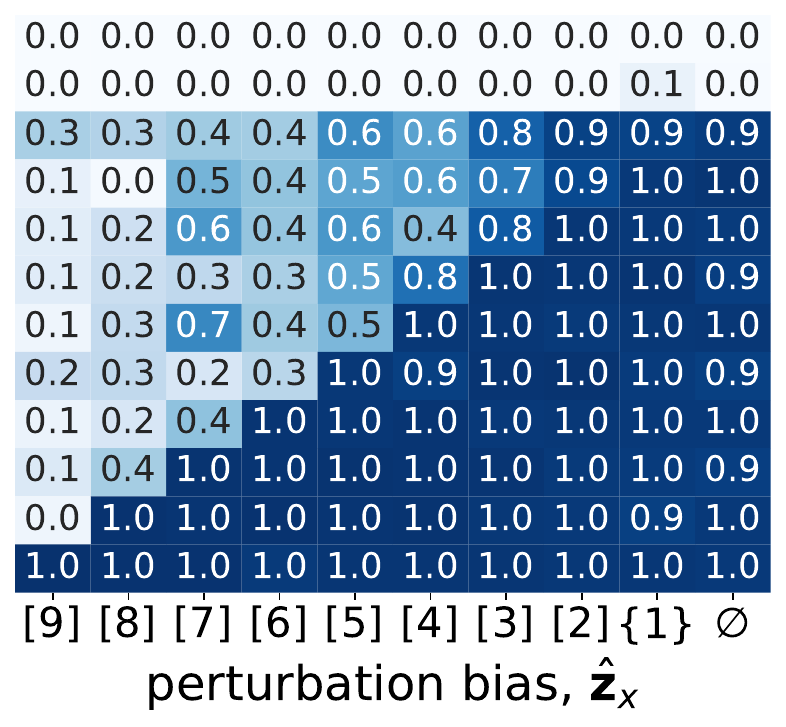}
    \hfill
    \includegraphics[width=0.228\textwidth]{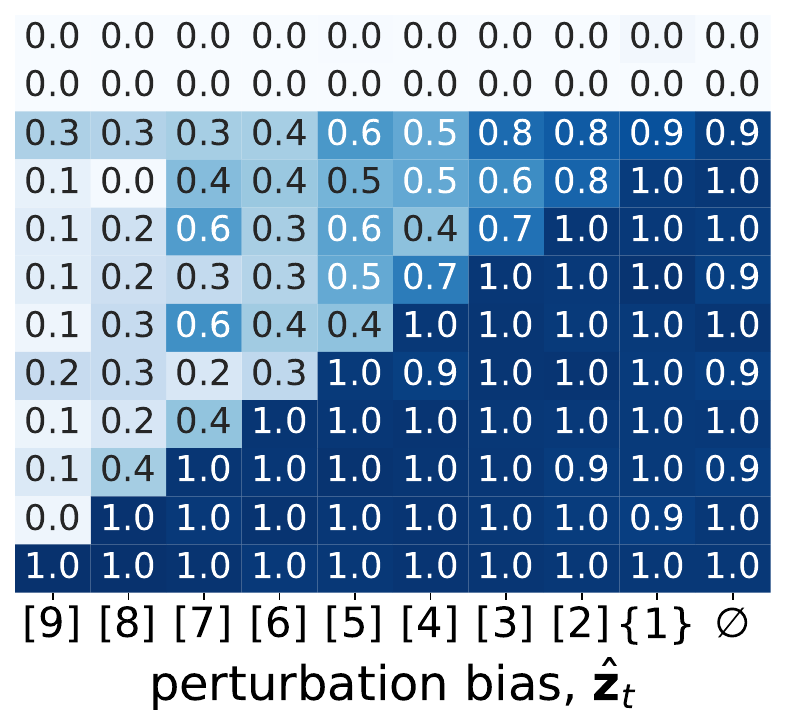}
    \caption{Linearity of learned representations under perturbation bias. Left to right: predictions using representations \(\hat{\mathbf{z}}_x\) and \(\hat{\mathbf{z}}_t\) under independent latent semantic variables, followed by those under dependent latent semantic variables.}
    \label{fig:numeric_linear_perturb}
\end{figure}

\paragraph{Linearity of learned representations.} 
We further analyze the linear decodability of the semantic variables by reporting the \(R^2\) scores obtained from linear regression applied to the learned features for predicting the true latent variables. As shown in \Cref{fig:numeric_linear} (selection bias) and \Cref{fig:numeric_linear_perturb} (perturbation bias), the performance of linear regression closely mirrors that of nonlinear regression reported in \Cref{fig:numeric_nonlinear}. This correspondence indicates that the evaluated linear decoders recover the retained factors accurately in these experiments.

\paragraph{Ablation studies.}
We perform two ablation studies to further examine the behavior beyond the theorem's exact-dimension setting. First, we investigate the impact of assigning an incorrect representation dimension. Specifically, we consider a scenario in which the true dimension of the selected semantic variables is \(3\) (with selection \(\mathbb{I}_\theta = [3]\)), but we intentionally set the representation dimension to \(5\). As shown in \Cref{tab:ablation_wrongdim}, in the independent case, the evaluated decoders achieve high scores for selected variables and low scores for omitted variables. In contrast, the dependent scenario yields significantly different patterns of \(R^2\) scores compared to those obtained using the correct representation dimension (see \Cref{fig:numeric_linear,fig:numeric_nonlinear}). These results suggest that redundant representation dimensions tend to encode exogenous noise, potentially introducing unnecessary complexity into the learned representations.

Second, we explore the joint effect of selection and perturbation biases by defining a scenario with selection bias \(\mathbb{I}_\theta = [8]\) and perturbation bias \(\mathbb{I}_\rho = [2]\). Results presented in \Cref{tab:ablation_bothbiases} demonstrate that when both biases coexist, their effects on measured semantic predictability remain consistent: semantic variables that are either omitted or perturbed are poorly predicted by the evaluated decoders, while unbiased semantic variables---those that are selected and yet unperturbed---are accurately predicted by the evaluated decoders. These ablations describe empirical behavior, including settings outside the representation-dimension assumption in \Cref{thm:idsemantics}.

\begin{table}[tbp]
    \caption{The effect of using an incorrect encoding size. The representation size is set to 5, whereas the true dimension should be 3. The biases are defined as \(\mathbb{I}_\theta = [3]\) and \(\mathbb{I}_\rho = \emptyset\).}
    \label{tab:ablation_wrongdim}
    \centering
    \resizebox{\linewidth}{!}{%
        \begin{tabular}{@{}p{1.5cm} p{1.5cm} c ccccccccccccc@{}}
            \toprule
            \multicolumn{2}{c}{\multirow{2}{*}{\textbf{Setting}}} & \multirow{2}{*}{\textbf{Reps.}} & \multicolumn{12}{c}{\textbf{\(R^2\) of Predicting Latent Semantic Variables under \(\mathbb{I}_\theta = [3]\) and \(\mathbb{I}_\rho = \emptyset\)}} \\
            \cmidrule(l){4-15}
            \multicolumn{2}{c}{} & & \(s_1\) & \(s_2\) & \(s_3\) & \(s_4\) & \(s_5\) & \(s_6\) & \(s_7\) & \(s_8\) & \(s_9\) & \(s_{10}\) & \(\mathbf{m}_x\) & \(\mathbf{m}_t\) \\
            \midrule
            \multirow{4}{*}{independ.} 
                & \multirow{2}{*}{linear} 
                    & \(\hat{\mathbf{z}}_x\) & 0.98 & 0.96 & 0.98 & 0.01 & 0.03 & 0.03 & 0.02 & 0.04 & 0.01 & 0.02 & 0.02 & 0.00 \\
                \cmidrule(l){3-15}
                & 
                    & \(\hat{\mathbf{z}}_t\) & 1.00 & 1.00 & 1.00 & 0.00 & 0.00 & 0.00 & 0.00 & 0.00 & 0.00 & 0.00 & 0.00 & 0.04 \\
            \cmidrule(l){2-15}
                & \multirow{2}{*}{non-lin.} 
                    & \(\hat{\mathbf{z}}_x\) & 0.98 & 0.99 & 0.99 & 0.00 & 0.02 & 0.02 & 0.03 & 0.04 & 0.00 & 0.00 & 0.00 & 0.00 \\
                \cmidrule(l){3-15}
                & 
                    & \(\hat{\mathbf{z}}_t\) & 1.00 & 1.00 & 1.00 & 0.00 & 0.00 & 0.00 & 0.00 & 0.00 & 0.00 & 0.00 & 0.00 & 0.06 \\
            \midrule
            \multirow{4}{*}{dependent} 
                & \multirow{2}{*}{linear} 
                    & \(\hat{\mathbf{z}}_x\) & 0.98 & 0.99 & 0.99 & 0.14 & 0.23 & 0.16 & 0.10 & 0.22 & 0.42 & 0.36 & 0.01 & 0.00 \\
                \cmidrule(l){3-15}
                & 
                    & \(\hat{\mathbf{z}}_t\) & 1.00 & 1.00 & 1.00 & 0.13 & 0.20 & 0.13 & 0.09 & 0.20 & 0.41 & 0.34 & 0.00 & 0.01 \\
            \cmidrule(l){2-15}
                & \multirow{2}{*}{non-lin.} 
                    & \(\hat{\mathbf{z}}_x\) & 0.99 & 1.00 & 0.99 & 0.13 & 0.22 & 0.16 & 0.15 & 0.20 & 0.43 & 0.36 & 0.02 & 0.00 \\
                \cmidrule(l){3-15}
                & 
                    & \(\hat{\mathbf{z}}_t\) & 1.00 & 1.00 & 1.00 & 0.13 & 0.20 & 0.13 & 0.09 & 0.20 & 0.42 & 0.35 & 0.00 & 0.05 \\
            \bottomrule
        \end{tabular}%
    }
\end{table}

\begin{table}[tbp]
    \caption{Coexistence of both selection and perturbation biases. The biases are defined as \(\mathbb{I}_\theta = [8]\) and \(\mathbb{I}_\rho = [2]\).}
    \label{tab:ablation_bothbiases}
    \centering
    \resizebox{\linewidth}{!}{%
        \begin{tabular}{@{}p{1.5cm} p{1.5cm} c ccccccccccccc@{}}
            \toprule
            \multicolumn{2}{c}{\multirow{2}{*}{\textbf{Setting}}} & \multirow{2}{*}{\textbf{Reps.}} & \multicolumn{12}{c}{\textbf{\(R^2\) of Predicting Latent Semantic Variables (\(\mathbb{I}_\theta = [8]\), \(\mathbb{I}_\rho = [2]\))}} \\
            \cmidrule(l){4-15}
            \multicolumn{2}{c}{} & & \(s_1\) & \(s_2\) & \(s_3\) & \(s_4\) & \(s_5\) & \(s_6\) & \(s_7\) & \(s_8\) & \(s_9\) & \(s_{10}\) & \(\mathbf{m}_x\) & \(\mathbf{m}_t\) \\
            \midrule
            \multirow{4}{*}{independ.} 
                & \multirow{2}{*}{linear} 
                    & \(\hat{\mathbf{z}}_x\) & 0.00 & 0.01 & 0.97 & 0.98 & 0.97 & 0.95 & 0.99 & 0.98 & 0.00 & 0.00 & 0.00 & 0.00 \\
                \cmidrule(l){3-15}
                & 
                    & \(\hat{\mathbf{z}}_t\) & 0.00 & 0.00 & 0.98 & 0.98 & 0.98 & 0.98 & 0.99 & 0.98 & 0.00 & 0.00 & 0.00 & 0.00 \\
            \cmidrule(l){2-15}
                & \multirow{2}{*}{non-lin.} 
                    & \(\hat{\mathbf{z}}_x\) & 0.00 & 0.01 & 0.98 & 0.98 & 0.98 & 0.98 & 0.96 & 0.99 & 0.00 & 0.00 & 0.00 & 0.00 \\
                \cmidrule(l){3-15}
                & 
                    & \(\hat{\mathbf{z}}_t\) & 0.00 & 0.00 & 0.99 & 0.99 & 0.99 & 0.99 & 0.99 & 0.99 & 0.00 & 0.00 & 0.00 & 0.00 \\
            \midrule
            \multirow{4}{*}{dependent} 
                & \multirow{2}{*}{linear} 
                    & \(\hat{\mathbf{z}}_x\) & 0.67 & 0.57 & 0.97 & 0.99 & 0.99 & 0.98 & 0.99 & 0.97 & 0.63 & 0.63 & 0.00 & 0.00 \\
                \cmidrule(l){3-15}
                & 
                    & \(\hat{\mathbf{z}}_t\) & 0.64 & 0.53 & 0.98 & 0.97 & 0.99 & 0.98 & 0.99 & 0.98 & 0.61 & 0.60 & 0.00 & 0.00 \\
            \cmidrule(l){2-15}
                & \multirow{2}{*}{non-lin.} 
                    & \(\hat{\mathbf{z}}_x\) & 0.68 & 0.57 & 0.99 & 0.99 & 0.99 & 0.99 & 0.99 & 0.99 & 0.64 & 0.64 & 0.00 & 0.00 \\
                \cmidrule(l){3-15}
                & 
                    & \(\hat{\mathbf{z}}_t\) & 0.65 & 0.53 & 0.99 & 0.99 & 0.99 & 0.99 & 0.99 & 0.99 & 0.61 & 0.61 & 0.00 & 0.00 \\
            \bottomrule
        \end{tabular}%
    }
\end{table}

\subsection{Additional Downstream Results}
\label{sec:num_task_details}


We further report downstream task performance under varying perturbation bias settings. With full semantic selection, we vary \(I_\gamma\) from \(\emptyset\) to \([9]\), so the retained set \(I_\gamma^\complement\) ranges from \([10]\) to \(\{10\}\). The results in \Cref{fig:perturb_downstream} illustrate how the consistently shared semantic block affects downstream regression and classification performance in the evaluated independent and dependent latent settings.

\begin{figure}[tbp]
    \centering
    \includegraphics[width=0.25\textwidth]{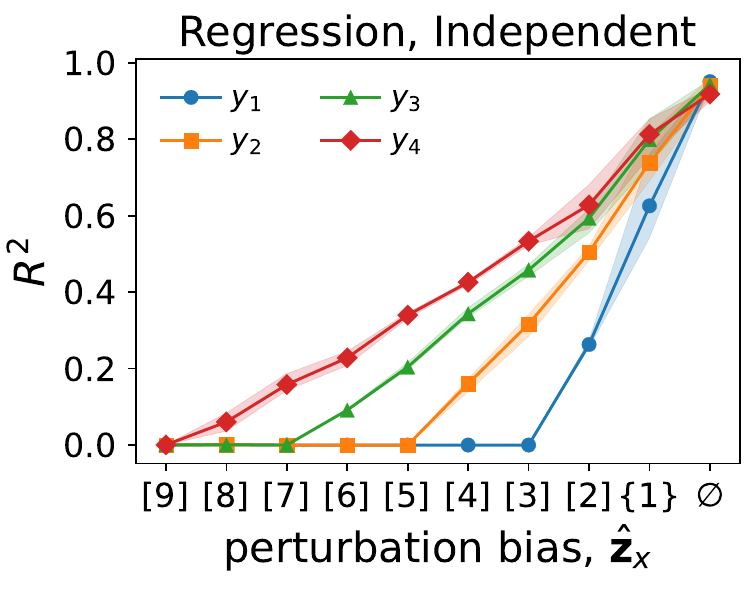}
    \includegraphics[width=0.234\textwidth]{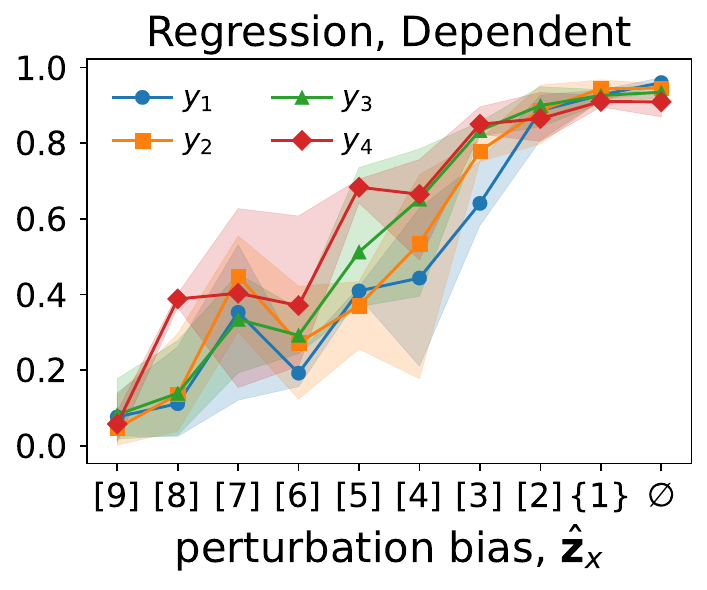}
    \hfill
    \includegraphics[width=0.25\textwidth]{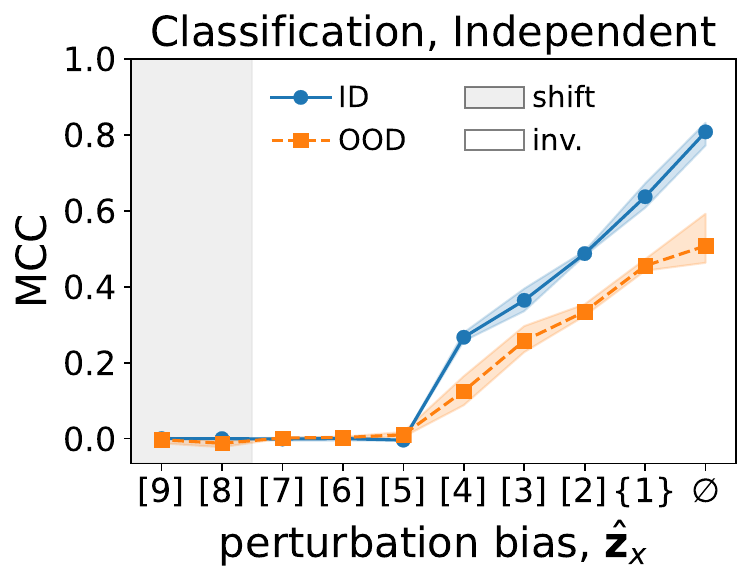}
    \includegraphics[width=0.234\textwidth]{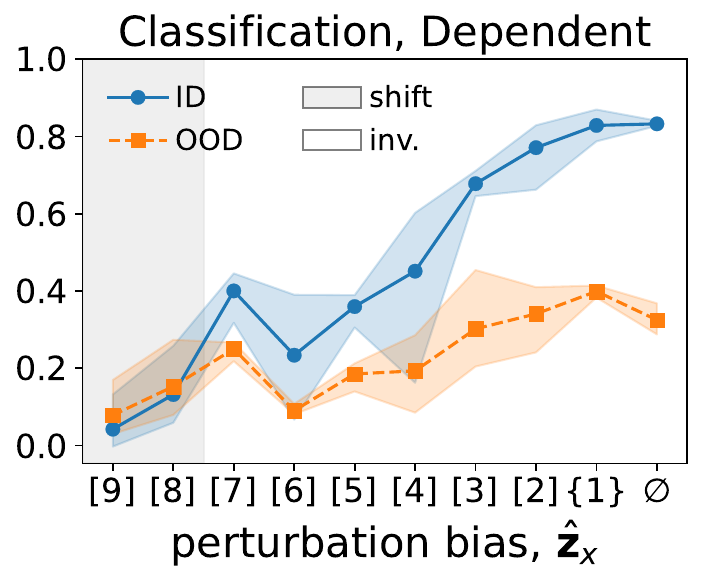}
    \caption{Downstream performance of pretrained representations \(Z_x\) under perturbation bias. Left: in-distribution (ID) regression performance on four tasks. Right: ID classification and out-of-distribution (OOD) generalization. Curves show means and shaded regions show ranges across seeds; negative regression \(R^2\) scores are clipped to zero before aggregation.}
    \label{fig:perturb_downstream}
\end{figure}

\section{Experiment Details on MPI3D-Complex Dataset}
\label{sec:detail_mpi3d}

We provide additional details on the MPI3D-Complex dataset that are not fully covered in \Cref{sec:MPIReal3D}. In \Cref{sec:mpi_setup_details}, we comprehensively describe the experimental setup, including a dataset overview, the selection and perturbation bias configurations used to generate text, model architecture, and training parameters. In \Cref{sec:ablation_mpi}, we present additional results, including the use of a linear classifier for predicting latent factors and an ablation study on encoder dimensionality.

\subsection{Detailed Experimental Setup}
\label{sec:mpi_setup_details}

\paragraph{MPI3D-Complex dataset.} 
MPI3D-Complex~\citep{NEURIPS2019_d97d404b} contains 460{,}800 real-world images of resolution \(64 \times 64 \times 3\), spanning all combinations of seven \emph{mutually independent, discrete factors} (see \Cref{tab:mpi_factors}). We designate horizontal and vertical positions (\texttt{hori.}, \texttt{vert.}) as image-specific due to their low-level visual nature, while treating the remaining five as semantic variables for representation learning and evaluation purposes.

\paragraph{Text latent factors.}  
Text descriptions are generated from the ground-truth semantic attributes using content-word mappings (see \Cref{tab:mpi_factors}). Under unbiased settings, all five semantic attributes are included. For text-specific variation, a generation template is chosen via a discrete latent variable
\(
m_t \sim \mathrm{Uniform}([3])
\), each corresponding to a different human-written sentence structure.

\begin{table}[tbp]
    \caption{Latent variables of MPI3D-Complex and corresponding content words in text.}
    \label{tab:mpi_factors}
    \centering
    \begin{tabular}{@{}lll@{}}
        \toprule
        \textbf{Factor Name} & \textbf{Distribution} & \textbf{Content Words} \\
        \midrule
        Object color (\texttt{color}) &
        \(\mathrm{Uniform}(\{0, \dots, 3\})\) &
        \texttt{yellow}, \texttt{green}, \texttt{olive}, \texttt{red} \\
        \midrule
        Object shape (\texttt{shape}) &
        \(\mathrm{Uniform}(\{0, \dots, 3\})\) &
        \makecell[l]{\texttt{coffee-cup}, \texttt{tennis-ball}, \texttt{croissant},\\ \texttt{beer-cup}} \\
        \midrule
        Object size (\texttt{size}) &
        \(\mathrm{Uniform}(\{0, 1\})\) &
        \texttt{small size}, \texttt{large size} \\
        \midrule
        Camera height (\texttt{cam.}) &
        \(\mathrm{Uniform}(\{0, \dots, 2\})\) &
        \makecell[l]{\texttt{in a purple background},\\ \texttt{in a sea-green background},\\ \texttt{in a salmon background}} \\
        \midrule
        Background color (\texttt{back.}) &
        \(\mathrm{Uniform}(\{0, \dots, 2\})\) &
        \texttt{purple}, \texttt{sea-green}, \texttt{salmon} \\
        \midrule
        Horizontal axis (\texttt{hori.}) &
        \(\mathrm{Uniform}(\{0, \dots, 39\})\) &
        --- (image-specific factor) \\
        \midrule
        Vertical axis (\texttt{vert.}) &
        \(\mathrm{Uniform}(\{0, \dots, 39\})\) &
        --- (image-specific factor) \\
        \bottomrule
    \end{tabular}
\end{table}

\paragraph{Text generation under misalignment settings.}  
We generate text for each image under various selection and perturbation bias settings to investigate the effects of misalignment. To ensure computational tractability---since exhaustively enumerating all possible configurations is both infeasible and unnecessary---we adopt a progressively incremental strategy for introducing biases into the latent semantic variables, as outlined in \Cref{tab:mpi_bias_settings}.  

In selection bias settings (where the perturbable set \(\mathbb{I}_\rho\) is empty), textual descriptions include only the content words corresponding to a subset of the true image semantic variables. We define five incremental settings: \circled{1}: \{\texttt{color}\},
\circled{2}: \{\texttt{color}, \texttt{shape}\},
\circled{3}: \{\texttt{color}, \texttt{shape}, \texttt{size}\},
\circled{4}: \{\texttt{color}, \texttt{shape}, \texttt{size}, \texttt{cam.}\},
\circled{5}: all five attributes.
The specific text generation templates associated with \(g_{t\sx{(\theta)}}\) under each selection setting are detailed in \Cref{tab:mpi_text_templates}.  

In perturbation bias settings (where all semantic indices are selected), we apply random substitutions to a subset \(\mathbb{I}_\rho\). For each factor in \(\mathbb{I}_\rho\), its text value is replaced with a value sampled uniformly from its possible values with probability 0.9, including the original value. Five perturbation configuration are used:
\circled{1}: \(\emptyset\),
\circled{2}: \{\texttt{back.}\},
\circled{3}: \{\texttt{cam.}, \texttt{back.}\},
\circled{4}: \{\texttt{size}, \texttt{cam.}, \texttt{back.}\},
\circled{5}: \{\texttt{shape}, \texttt{size}, \texttt{cam.}, \texttt{back.}\}.

\begin{table}[tbp]
\caption{Misalignment settings for text generation of MPI3D-Complex dataset.}
\label{tab:mpi_bias_settings}
\centering
\begin{tabular}{@{}c l l@{}}
    \toprule
    \textbf{Setting} & \textbf{Selection Bias, $\mathbb{I}_\theta$} & \textbf{Perturbation Bias, $\mathbb{I}_\rho$} \\
    \midrule
    \(\circled{1}\) & \{\texttt{color}\} & $\emptyset$ \\
    \midrule
    \(\circled{2}\) & \{\texttt{color}, \texttt{shape}\} & \{\texttt{back.}\} \\
    \midrule
    \(\circled{3}\) & \{\texttt{color}, \texttt{shape}, \texttt{size}\} & \{\texttt{cam.}, \texttt{back.}\} \\
    \midrule
    \(\circled{4}\) & \{\texttt{color}, \texttt{shape}, \texttt{size}, \texttt{cam.}\} & \{\texttt{size}, \texttt{cam.}, \texttt{back.}\} \\
    \midrule
    \(\circled{5}\) & \{\texttt{color}, \texttt{shape}, \texttt{size}, \texttt{cam.}, \texttt{back.}\} & \{\texttt{shape}, \texttt{size}, \texttt{cam.}, \texttt{back.}\} \\
    \bottomrule
\end{tabular}
\end{table}

\begin{table}[tbp]
    \caption{Text generation templates for MPI3D-Complex dataset for different selection settings.}
    \label{tab:mpi_text_templates}
    \centering
    \begin{tabular}{c|l}
        \toprule
        \textbf{Setting} & \textbf{Text Generation Templates for Each Selection View} \\
        \midrule
        \(\circled{1}\)  & \begin{tabular}[c]{@{}l@{}} 
            "An object colored \{\texttt{color}\}." \\ 
            "It has a \{\texttt{color}\} appearance." \\ 
            "Something with \{\texttt{color}\}."
        \end{tabular} \\
        \midrule
        \(\circled{2}\) & \begin{tabular}[c]{@{}l@{}} 
            "A \{\texttt{shape}\} that is \{\texttt{color}\}." \\ 
            "The \{\texttt{color}\} \{\texttt{shape}\}." \\ 
            "An object shaped like a \{\texttt{shape}\}, colored \{\texttt{color}\}."
        \end{tabular} \\
        \midrule
        \(\circled{3}\) & \begin{tabular}[c]{@{}l@{}} 
            "A \{ \texttt{size}\} \{\texttt{shape}\} in \{\texttt{color}\}." \\ 
            "\{\texttt{color}\}, \{ \texttt{size}\}, \{\texttt{shape}\}." \\ 
            "The object is \{ \texttt{size}\}, shaped as a \{\texttt{shape}\}, and colored \{\texttt{color}\}."
        \end{tabular} \\
        \midrule
        \(\circled{4}\) & \begin{tabular}[c]{@{}l@{}} 
            "A \{ \texttt{size}\} \{\texttt{shape}\} in \{\texttt{color}\}, seen from \{\texttt{cam.}\}." \\ 
            "Viewed from \{\texttt{cam.}\}, a \{\texttt{color}\}, \{ \texttt{size}\} \{\texttt{shape}\}." \\ 
            "A \{ \texttt{size}\} \{\texttt{shape}\} with \{\texttt{color}\}, perspective: \{\texttt{cam.}\}."
        \end{tabular} \\
        \midrule
        \(\circled{5}\) & \begin{tabular}[c]{@{}l@{}} 
            "A \{ \texttt{size}\} \{\texttt{shape}\} in \{\texttt{color}\}, viewed from \{\texttt{cam.}\}, \{\texttt{back.}\}." \\ 
            "From \{\texttt{cam.}\}, you see a \{\texttt{color}\}, \{ \texttt{size}\} \{\texttt{shape}\}, \{\texttt{back.}\}." \\ 
            "A \{ \texttt{size}\} \{\texttt{shape}\}, \{\texttt{color}\}, placed \{\texttt{back.}\}, observed from \{\texttt{cam.}\}."
        \end{tabular} \\
        \bottomrule
    \end{tabular}%
\end{table}

\paragraph{Training details.}
For each setting, the dataset is partitioned into training, evaluation, and test subsets containing \(414{,}720\), \(23{,}040\), and \(23{,}040\) samples, respectively. Across all configurations, we train the image and text encoders for approximately \(200{,}000\) steps using the training subset and repeat the experiments with three random seeds. The encoders are trained with a symmetric InfoNCE loss using cosine similarity and both cross-modal and within-modality negatives, optimized using Adam with a learning rate of \(1 \times 10^{-5}\), a batch size of 256, and temperature \(\tau=1.0\). 

For image encoding, we use a ResNet18 backbone configured to output 100 features, followed by a LeakyReLU activation and a linear projection to the chosen representation dimension. For text encoding, we use NLTK's \texttt{sent\_tokenize} and \texttt{word\_tokenize} functions. Token sequences are padded and represented as one-hot vectors, which are mapped to learned 128-dimensional token embeddings and processed by a two-dimensional CNN whose depth depends on the input sequence length. A final linear layer maps the resulting features to the same representation dimension as the image encoder. The encoding size is generally set to match the number of unbiased semantic factors, i.e., \(|\mathbb{I}^c_\rho|\). However, when this number is less than 3, we set the encoding size to 3 to ensure minimal representational capacity. An ablation study on this design choice is provided in the following section.

\paragraph{Evaluation metrics.}
After training, we freeze the encoders and evaluate image and text representations separately. For discrete factors, we fit logistic regression and a nonlinear MLP classifier; for continuous factors, we fit ordinary least-squares linear regression and a nonlinear MLP regressor. The MLPs have one hidden layer with ReLU activation. Probes are fitted on the evaluation subset after standardizing the representations using statistics from that subset. Logistic regression and both MLP probes use \texttt{max\_iter=10000}.

Image probes predict the original image-side latent factors. Text probes predict the text-side factors stored for the corresponding condition, including resampled values under perturbation and the template identity \texttt{phrase}. Consequently, color targets are continuous hues for image probes and categorical palette indices for text probes.

Prediction performance is measured on the test subset using MCC for discrete factors and \(R^2\) for continuous factors. Metrics are computed separately for each factor and modality. For the reported heatmaps, scores are clipped to \([0,1]\) before averaging across three runs.

\subsection{Additional Results}
\label{sec:ablation_mpi}

\paragraph{Linear predictability of latent factors.}
We evaluate linear predictability by reporting MCC scores for predicting each latent factor using logistic regression. Comparing \Cref{fig:mpi3d_linear} with the nonlinear-probe results in \Cref{fig:mpi3d_nonlinear}, we find that \texttt{size}, when selected and unperturbed, is less accurately predicted by the linear classifier than by the nonlinear classifier. In contrast, selected and unperturbed factors such as \texttt{color}, \texttt{shape}, \texttt{cam.}, and \texttt{back.} generally achieve higher linear-probe scores, indicating that these factors are more readily accessible through linear classification.



\begin{figure}[tbp]
    \centering
    \includegraphics[width=0.303\textwidth]{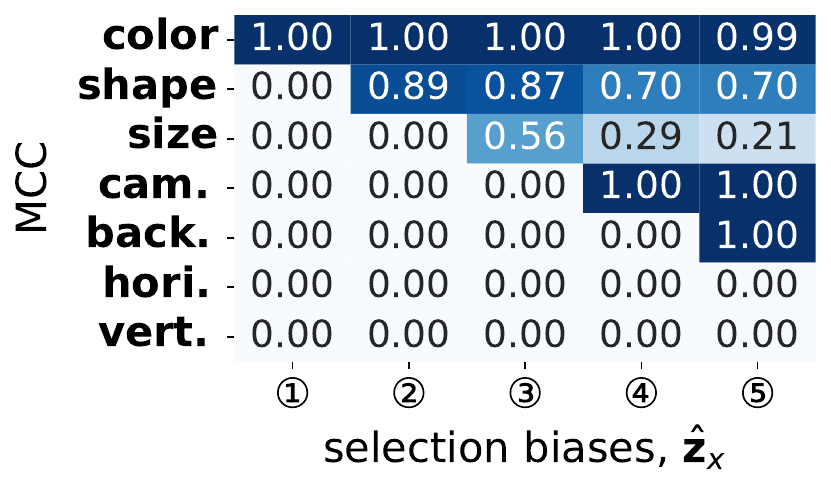}
    \hfill
    \includegraphics[width=0.222\textwidth]{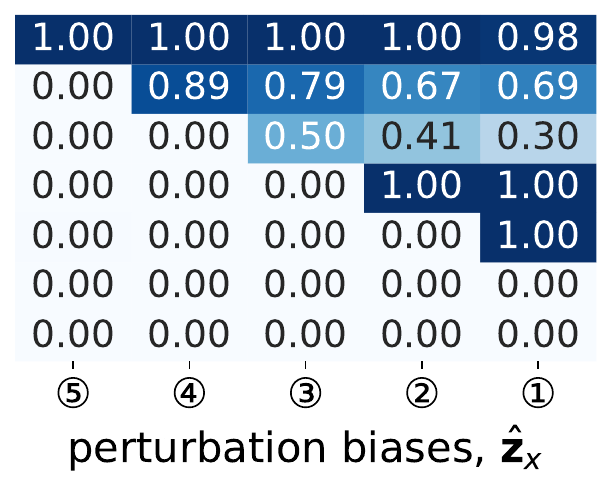}
    \hfill
    \includegraphics[width=0.222\textwidth]{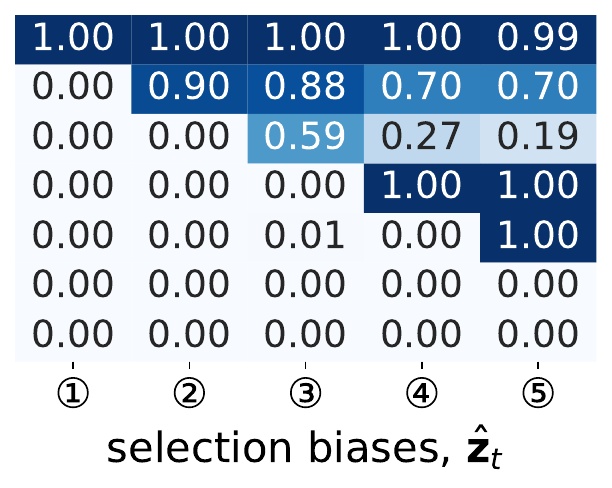}
    \hfill
    \includegraphics[width=0.222\textwidth]{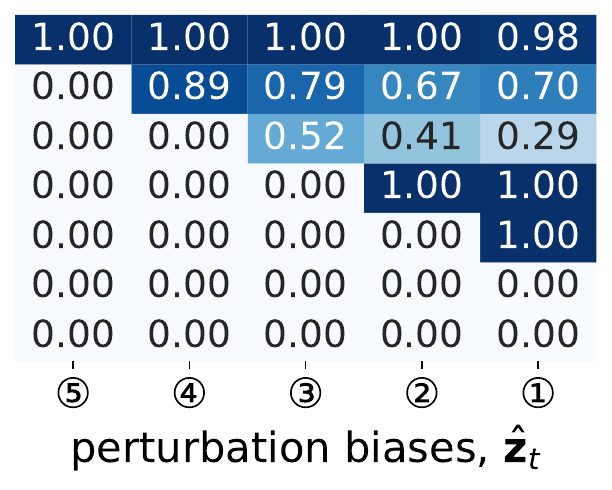}
    \caption{Mean MCC of linear classifiers predicting latent factors under misalignment. Left to right: image features \(Z_x\) under selection bias, image features \(Z_x\) under perturbation bias, text features \(Z_t\) under selection bias, and text features \(Z_t\) under perturbation bias. Scores are clipped to \([0,1]\) before averaging across three runs.}
    \label{fig:mpi3d_linear}
\end{figure}

\paragraph{Ablations on the encoding size.}  
We conduct an ablation study on the encoding size by evaluating a selection bias setting in which only the semantic attribute \{\texttt{color}\} is selected for text generation. All other training parameters are kept consistent with those used in our main experiments, except for the encoding dimensionality.


As shown in \Cref{fig:mpi-ablation}, factor predictability is sensitive to the encoding size in this image-text setting. An encoding size of 1 corresponds to one representation coordinate for the single selected discrete factor, \texttt{color}. At this size, color prediction is weaker and the reported scores exhibit greater dispersion. Encoding sizes of 2--4 yield near-perfect color prediction with both linear and nonlinear classifiers. Across the tested encoding sizes, omitted factors remain poorly predicted by both classifiers. Thus, in this setting with independent latent factors, additional representation capacity improves prediction of the selected factor without an appreciable increase in the predictability of omitted factors.

\begin{figure}[tbp]
    \centering
    \includegraphics[width=\linewidth]{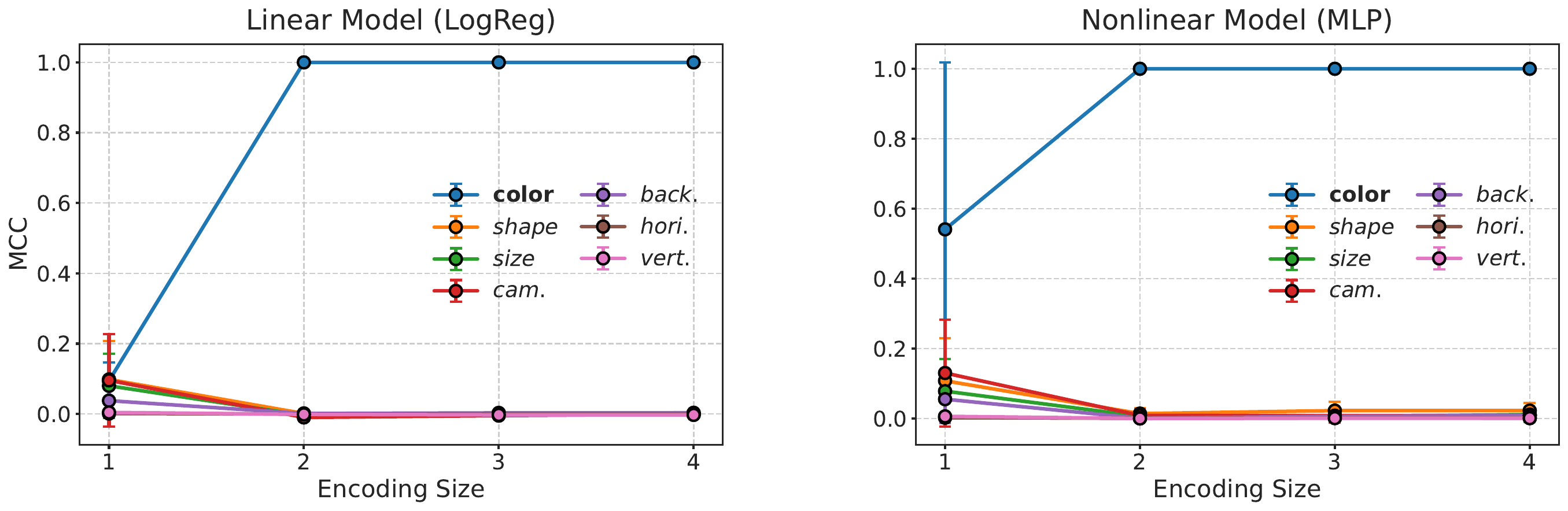}
    \caption{Ablation study on encoding size when only \texttt{color} is included in text generation. Points show mean MCC pooled across image and text representations from three runs, and error bars show the standard deviation of the pooled scores. Left: logistic regression. Right: nonlinear MLP.}
    \label{fig:mpi-ablation}
\end{figure}

\section{Experiment Details on Causal3DIdent Dataset}
\label{sec:detail_3dident}

We provide additional details on the Causal3DIdent dataset that are not fully covered in \Cref{sec:Causal3DIdent}. In \Cref{sec:c3d_detail_setup}, we offer a comprehensive description of the experimental setup, including the image and text latent factors, the image generation process, the design of selection and perturbation bias settings for text generation, and training configurations. In \Cref{sec:c3d_ablations}, we present supplementary results, including analyses of the learned text representations and assessments of the linearity of the learned representations.

\begin{figure}[tbp]
    \centering
    \includegraphics[width=0.75\linewidth]{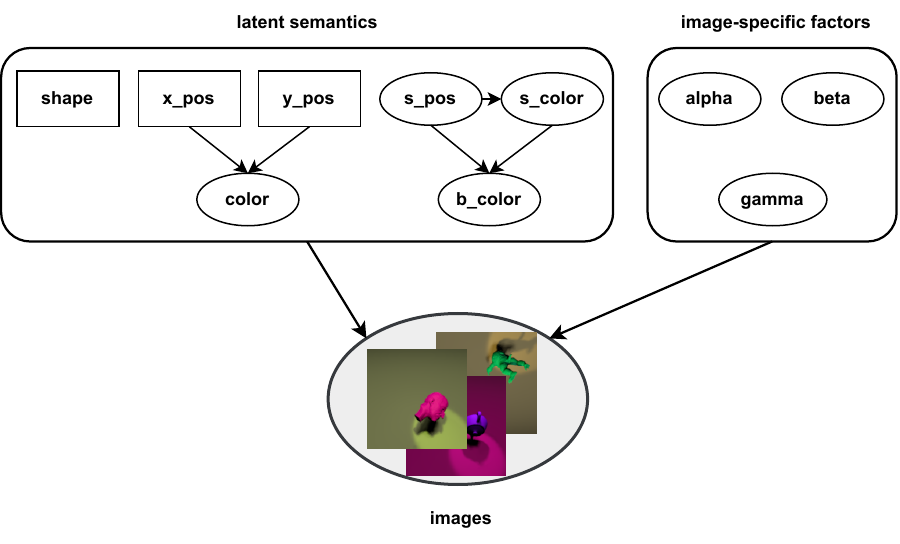}
    \caption{Latent causal model governing image generation in the Causal3DIdent dataset. Rectangular nodes represent discrete latent random variables, while elliptical nodes denote continuous ones. Object shape (\texttt{shape}), horizontal and vertical position (\texttt{x\_pos}, \texttt{y\_pos}), object color (\texttt{color}), spotlight position and color (\texttt{s\_pos}, \texttt{s\_color}), and background color (\texttt{b\_color}) are treated as latent semantic variables shared across modalities and potentially subject to misalignment. In contrast, the rotation angles—\texttt{alpha}, \texttt{beta}, and \texttt{gamma}—are considered image-specific latent factors.}
    \label{fig:c3did_gen}
\end{figure}

\subsection{Detailed Experimental Setup}
\label{sec:c3d_detail_setup}

\paragraph{Image latent factors and image generation.}  
Following prior work~\citep{zimmermann2021contrastive, von2021self, daunhawer2022identifiability, yao2023multi}, we utilize the \emph{Causal3DIdent} dataset to synthesize images from a predefined latent causal structure. Images are generated using the Blender renderer~\citep{benlder2018}, which applies a complex rendering function parameterized by 11 input variables. In our configuration, the object's \(z\)-position is fixed, leaving 10 latent factors that govern image generation.

These include 3 discrete variables—object shape (\texttt{shape}), and object positions along the horizontal (\texttt{x\_pos}) and vertical (\texttt{y\_pos}) axes—and 7 continuous variables: object color (\texttt{color}), spotlight position (\texttt{s\_pos}) and color (\texttt{s\_color}), background color (\texttt{b\_color}), and the three object rotation angles (\texttt{alpha}, \texttt{beta}, \texttt{gamma}). We treat the rotation angles (\texttt{alpha}, \texttt{beta}, \texttt{gamma}) as image-specific latent variables, while the remaining factors are considered semantic latent variables, structured according to the causal graph shown in \Cref{fig:c3did_gen}.

We synthesize \(80{,}000\) samples for MMCL, \(10{,}000\) samples for classifier or regressor training, and another \(10{,}000\) samples for test-time evaluation. Images are rendered at a resolution of \(128 \times 128 \times 3\).

\paragraph{Text latent factors.}  
We map the continuous \texttt{color}, \texttt{s\_color}, and \texttt{b\_color} variables to categorical color names using \texttt{TABLEAU\_COLORS}, \texttt{CSS4\_COLORS}, and \texttt{XKCD\_COLORS}, respectively. Each hue is converted to RGB and assigned the nearest palette color in Euclidean RGB distance. The spotlight position \texttt{s\_pos} remains continuous in the stored text-side latent variables but is rendered using five directional descriptions corresponding to equal-width intervals of \([0,1]\). Both the color quantization and the spotlight-position descriptions therefore discard information about the underlying continuous values.

To introduce text-specific variation, we employ five manually designed templates to generate text from the latent factors under each bias setting, adapting the text rendering pipeline from~\citep{daunhawer2022identifiability}. A complete list of latent factors and their types is provided in \Cref{tab:c3did_factors}.

\begin{table}[tbp]
    \caption{Latent factors of Causal3DIdent and corresponding content words in text.}
    \label{tab:c3did_factors}
    \centering
    \resizebox{\textwidth}{!}{
    \begin{tabular}{@{}llll@{}}
        \toprule
        \textbf{Factor Name} & \textbf{Image Modality} & \textbf{Text Modality} & \textbf{Content Words} \\
        \midrule
        \texttt{shape} &
        \(\mathrm{Uniform}(\{0, \dots, 6\})\) &
        \(\mathrm{Uniform}(\{0, \dots, 6\})\) &
        \makecell[l]{\texttt{teapot}, \texttt{hare}, \texttt{dragon}, \texttt{cow},\\
                     \texttt{armadillo}, \texttt{horse}, \texttt{head}} \\
        \midrule
        \texttt{x\_pos} &
        \(\mathrm{Uniform}(\{0, 1, 2\})\) &
        \(\mathrm{Uniform}(\{0, 1, 2\})\) &
        \texttt{left}, \texttt{center}, \texttt{right} \\
        \midrule
        \texttt{y\_pos} &
        \(\mathrm{Uniform}(\{0, 1, 2\})\) &
        \(\mathrm{Uniform}(\{0, 1, 2\})\) &
        \texttt{top}, \texttt{mid}, \texttt{bottom} \\
        \midrule
        \texttt{s\_pos} &
        \(\mathrm{Uniform}([0,1])\) &
        \(\mathrm{Uniform}([0,1])\) & \makecell[l]{\texttt{northwest}, \texttt{northeast}, \texttt{center},\\  \texttt{southwest}, \texttt{southeast}} \\
        \midrule
        \texttt{color} &
        \(\frac{1}{6}(\texttt{x\_pos} + \texttt{y\_pos}) + \frac{1}{3} \mathrm{Uniform}([0,1])\) &
        Up to 10 colors &
        Color names in \texttt{TABLEAU\_COLORS} \\
        \midrule
        \texttt{s\_color} &
        \(\frac{1}{2}(\texttt{s\_pos} + \mathrm{Uniform}([0,1]))\) &
        Up to 147 colors &
        Color names in \texttt{CSS4\_COLORS} \\
        \midrule
        \texttt{b\_color} &
        \(\frac{1}{3}(\texttt{s\_pos} + \texttt{s\_color} + \mathrm{Uniform}([0,1]))\) &
        Up to 954 colors &
        Color names in \texttt{XKCD\_COLORS} \\
        \midrule
        \texttt{alpha} &
        \(\mathrm{Uniform}([0,1])\) & --- & --- \\
        \midrule
        \texttt{beta} &
        \(\mathrm{Uniform}([0,1])\) & --- & --- \\
        \midrule
        \texttt{gamma} &
        \(\mathrm{Uniform}([0,1])\) & --- & --- \\
        \midrule
        \texttt{phrase} & ---
        & \(\mathrm{Uniform}(\{0,\dots,4\})\) & --- \\
        \bottomrule
    \end{tabular}
    }
\end{table}

\paragraph{Text generation under different misalignment settings.}  
Following the MPI3D-Complex experiments, we explore a series of incrementally increasing selection and perturbation bias configurations to introduce varying degrees and types of cross-modal misalignment. These settings enable a systematic investigation of how different forms of alignment impact representation learning. Each configuration is indexed using circled numerals and summarized in \Cref{tab:c3did_bias_settings}.


For every sample, each factor in the designated perturbable subset is independently resampled. The discrete variables \texttt{x\_pos} and \texttt{y\_pos} are sampled uniformly from \(\{0,1,2\}\), while the continuous variables are sampled uniformly from \([0,1]\). Discrete resampling may reproduce the original value. The resulting values are then passed through the same text-rendering mappings. \Cref{tab:c3did_text_templates} provides the text generation templates associated with each selection setting. Representative image–text pairs generated under different selection and perturbation bias configurations are shown in \Cref{fig:c3did_sample}.

\begin{table}[tbp]
\caption{Separate selection- and perturbation-bias settings for Causal3DIdent text generation. The selection experiment varies \(\mathbb I_\theta\) with \(\mathbb I_\rho=\emptyset\); the perturbation experiment varies \(\mathbb I_\rho\) with \(\mathbb I_\theta=I_S\). Entries in the same row describe separate experiments.}
\label{tab:c3did_bias_settings}
\centering
\resizebox{\linewidth}{!}{
\begin{tabular}{@{}c l l@{}}
    \toprule
    \(\circled{1}\) & \{\texttt{shape}\} & $\emptyset$ \\
    \midrule
    \(\circled{2}\) & \{\texttt{shape}, \texttt{x\_pos}\} & \{\texttt{b\_color}\} \\
    \midrule
    \(\circled{3}\) & \{\texttt{shape}, \texttt{x\_pos}, \texttt{y\_pos}\} & \{\texttt{s\_color}, \texttt{b\_color}\} \\
    \midrule
    \(\circled{4}\) & \{\texttt{shape}, \texttt{x\_pos}, \texttt{y\_pos}, \texttt{s\_pos}\} & \{\texttt{color}, \texttt{s\_color}, \texttt{b\_color}\} \\
    \midrule
    \(\circled{5}\) & \{\texttt{shape}, \texttt{x\_pos}, \texttt{y\_pos}, \texttt{s\_pos}, \texttt{color}\} & \{\texttt{s\_pos}, \texttt{color}, \texttt{s\_color}, \texttt{b\_color}\} \\
    \midrule
    \(\circled{6}\) & \{\texttt{shape}, \texttt{x\_pos}, \texttt{y\_pos}, \texttt{s\_pos}, \texttt{color}, \texttt{s\_color}\} & \{\texttt{y\_pos}, \texttt{s\_pos}, \texttt{color}, \texttt{s\_color}, \texttt{b\_color}\} \\
    \midrule
    \(\circled{7}\) & \{\texttt{shape}, \texttt{x\_pos}, \texttt{y\_pos}, \texttt{s\_pos}, \texttt{color}, \texttt{s\_color}, \texttt{b\_color}\} & \{\texttt{x\_pos}, \texttt{y\_pos}, \texttt{s\_pos}, \texttt{color}, \texttt{s\_color}, \texttt{b\_color}\} \\
    \bottomrule
\end{tabular}
}
\end{table}

\begin{figure}[tbp]
    \centering
    \includegraphics[width=\linewidth]{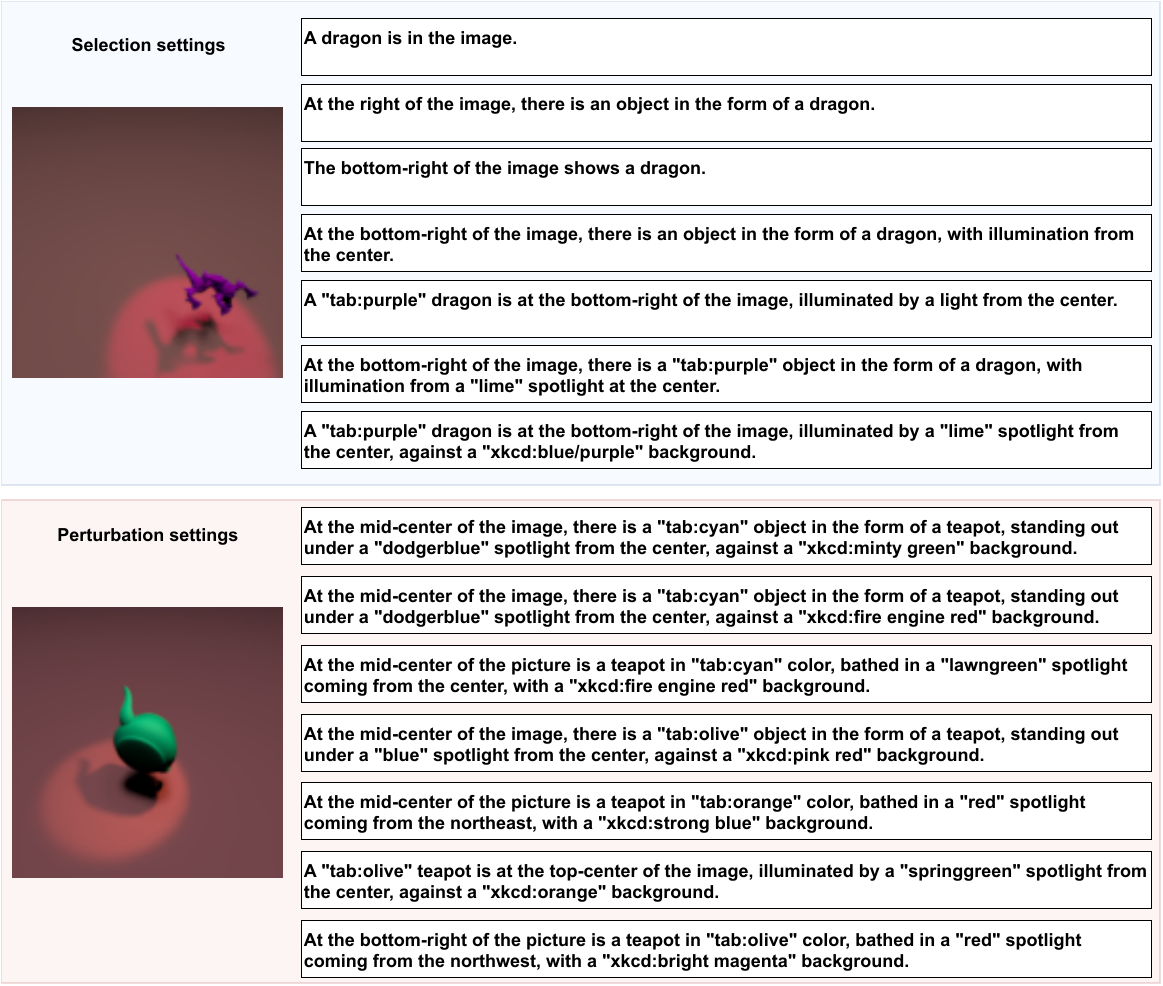}
\caption{Example image-text pairs from Causal3DIdent under different selection and perturbation bias settings. The left panel shows randomly selected images; the right panel presents the corresponding text from top to bottom, each generated under selection settings \circled{1} to \circled{7} with no perturbations, and perturbation bias \circled{7} to \circled{1} with full selections.}
\label{fig:c3did_sample}
\end{figure}

\paragraph{Training details.}  
Across all experimental settings, we train the image and text encoders for \(100{,}000\) steps on the training subset, using three different random seeds to ensure robustness. The training objective is a symmetric InfoNCE loss with cosine similarity and both within-modality and cross-modal negatives, excluding same-modality self-pairs, and optimized using the Adam optimizer with an initial learning rate of \(1 \times 10^{-5}\), a batch size of 256, and a temperature parameter \(\tau = 1.0\).

For both the image and text encoders, we adopt the same architectures used in the MPI3D-Complex experiments. The encoding dimensionality is adjusted according to the bias setting: for selection settings \circled{1} through \circled{7}, the encoding sizes are set to 3, 3, 4, 5, 5, 6, and 7, respectively; for perturbation settings \circled{1} through \circled{7}, the encoding sizes are assigned in reverse order: 7, 6, 5, 5, 4, 3, and 3.

\begin{table}[tbp]
    \caption{Text generation templates for Causal3DIdent dataset under different selection settings.}
    \label{tab:c3did_text_templates}
    \centering
    \small 
    \begin{tabular}{c|p{0.87\textwidth}}
        \toprule
        \textbf{Setting} & \textbf{Text Generation Templates} \\
        \midrule
        \(\circled{1}\) & 
        "A \{\texttt{shape}\} is visible." \newline
        "A \{\texttt{shape}\} is in the image." \newline
        "The image shows a \{\texttt{shape}\}." \newline
        "The picture is a \{\texttt{shape}\}." \newline
        "There is an object in the form of a \{\texttt{shape}\}." \\
        \midrule
        \(\circled{2}\) & 
        "A \{\texttt{shape}\} is visible, positioned at the \{\texttt{x\_pos}\} of the image." \newline
        "A \{\texttt{shape}\} is at the \{\texttt{x\_pos}\} of the image." \newline
        "The \{\texttt{x\_pos}\} of the image shows a \{\texttt{shape}\}." \newline
        "At the \{\texttt{x\_pos}\} of the picture is a \{\texttt{shape}\}." \newline
        "At the \{\texttt{x\_pos}\} of the image, there is an object in the form of a \{\texttt{shape}\}." \\
        \midrule
        \(\circled{3}\) & 
        "A \{\texttt{shape}\} is visible, positioned at the \{\texttt{y\_pos}\}-\{\texttt{x\_pos}\} of the image." \newline
        "A \{\texttt{shape}\} is at the \{\texttt{y\_pos}\}-\{\texttt{x\_pos}\} of the image." \newline
        "The \{\texttt{y\_pos}\}-\{\texttt{x\_pos}\} of the image shows a \{\texttt{shape}\}." \newline
        "At the \{\texttt{y\_pos}\}-\{\texttt{x\_pos}\} of the picture is a \{\texttt{shape}\}." \newline
        "At the \{\texttt{y\_pos}\}-\{\texttt{x\_pos}\} of the image, there is an object in the form of a \{\texttt{shape}\}." \\
        \midrule
        \(\circled{4}\) &
        "A \{\texttt{shape}\} is visible, positioned at the \{\texttt{y\_pos}\}-\{\texttt{x\_pos}\} of the image, with a spotlight shining from the \{\texttt{s\_pos}\}." \newline
        "A \{\texttt{shape}\} is at the \{\texttt{y\_pos}\}-\{\texttt{x\_pos}\} of the image, illuminated by a light from the \{\texttt{s\_pos}\}." \newline
        "The \{\texttt{y\_pos}\}-\{\texttt{x\_pos}\} of the image shows a \{\texttt{shape}\}, highlighted by a light source positioned at the \{\texttt{s\_pos}\}." \newline
        "At the \{\texttt{y\_pos}\}-\{\texttt{x\_pos}\} of the picture is a \{\texttt{shape}\}, standing out under a light coming from the \{\texttt{s\_pos}\}." \newline
        "At the \{\texttt{y\_pos}\}-\{\texttt{x\_pos}\} of the image, there is an object in the form of a \{\texttt{shape}\}, with illumination from the \{\texttt{s\_pos}\}." \\
        \midrule
        \(\circled{5}\) &
        "A \{\texttt{shape}\} of "\{\texttt{color}\}" color is visible, positioned at the \{\texttt{y\_pos}\}-\{\texttt{x\_pos}\} of the image, with a spotlight shining from the \{\texttt{s\_pos}\}." \newline
        "A "\{\texttt{color}\}" \{\texttt{shape}\} is at the \{\texttt{y\_pos}\}-\{\texttt{x\_pos}\} of the image, illuminated by a light from the \{\texttt{s\_pos}\}." \newline
        "The \{\texttt{y\_pos}\}-\{\texttt{x\_pos}\} of the image shows a "\{\texttt{color}\}" colored \{\texttt{shape}\}, highlighted by a light source positioned at the \{\texttt{s\_pos}\}." \newline
        "At the \{\texttt{y\_pos}\}-\{\texttt{x\_pos}\} of the picture is a \{\texttt{shape}\} in "\{\texttt{color}\}" color, standing out under a light coming from the \{\texttt{s\_pos}\}." \newline
        "At the \{\texttt{y\_pos}\}-\{\texttt{x\_pos}\} of the image, there is a "\{\texttt{color}\}" object in the form of a \{\texttt{shape}\}, with illumination from the \{\texttt{s\_pos}\}." \\
        \midrule
        \(\circled{6}\) &
        "A \{\texttt{shape}\} of "\{\texttt{color}\}" color is visible, positioned at the \{\texttt{y\_pos}\}-\{\texttt{x\_pos}\} of the image, under a "\{\texttt{s\_color}\}" spotlight shining from the \{\texttt{s\_pos}\}." \newline
        "A "\{\texttt{color}\}" \{\texttt{shape}\} is at the \{\texttt{y\_pos}\}-\{\texttt{x\_pos}\} of the image, illuminated by a "\{\texttt{s\_color}\}" spotlight from the \{\texttt{s\_pos}\}." \newline
        "The \{\texttt{y\_pos}\}-\{\texttt{x\_pos}\} of the image shows a "\{\texttt{color}\}" colored \{\texttt{shape}\}, highlighted by a "\{\texttt{s\_color}\}" spotlight positioned at the \{\texttt{s\_pos}\}." \newline
        "At the \{\texttt{y\_pos}\}-\{\texttt{x\_pos}\} of the picture is a \{\texttt{shape}\} in "\{\texttt{color}\}" color, standing out under a "\{\texttt{s\_color}\}" spotlight coming from the \{\texttt{s\_pos}\}." \newline
        "At the \{\texttt{y\_pos}\}-\{\texttt{x\_pos}\} of the image, there is a "\{\texttt{color}\}" object in the form of a \{\texttt{shape}\}, with illumination from a "\{\texttt{s\_color}\}" spotlight at the \{\texttt{s\_pos}\}." \\
        \midrule
        \(\circled{7}\) &
        "A \{\texttt{shape}\} of "\{\texttt{color}\}" color is visible, positioned at the \{\texttt{y\_pos}\}-\{\texttt{x\_pos}\} of the image, under a "\{\texttt{s\_color}\}" spotlight shining from the \{\texttt{s\_pos}\}, with a "\{\texttt{b\_color}\}" background." \newline
        "A "\{\texttt{color}\}" \{\texttt{shape}\} is at the \{\texttt{y\_pos}\}-\{\texttt{x\_pos}\} of the image, illuminated by a "\{\texttt{s\_color}\}" spotlight from the \{\texttt{s\_pos}\}, against a "\{\texttt{b\_color}\}" background." \newline
        "The \{\texttt{y\_pos}\}-\{\texttt{x\_pos}\} of the image shows a "\{\texttt{color}\}" colored \{\texttt{shape}\}, highlighted by a "\{\texttt{s\_color}\}" spotlight positioned at the \{\texttt{s\_pos}\}, on a "\{\texttt{b\_color}\}" backdrop." \newline
        "At the \{\texttt{y\_pos}\}-\{\texttt{x\_pos}\} of the picture is a \{\texttt{shape}\} in "\{\texttt{color}\}" color, bathed in a "\{\texttt{s\_color}\}" spotlight coming from the \{\texttt{s\_pos}\}, with a "\{\texttt{b\_color}\}" background." \newline
        "At the \{\texttt{y\_pos}\}-\{\texttt{x\_pos}\} of the image, there is a "\{\texttt{color}\}" object in the form of a \{\texttt{shape}\}, standing out under a "\{\texttt{s\_color}\}" spotlight from the \{\texttt{s\_pos}\}, against a "\{\texttt{b\_color}\}" background." \\
        \bottomrule
    \end{tabular}
\end{table}

\paragraph{Evaluation metrics.}  
After training the representations, we freeze the encoders and, for each bias setting, train both a linear classifier (logistic regression) and a nonlinear classifier (a two-layer MLP with ReLU activation) for each discrete latent factor. Similarly, for continuous latent factors, we train both a linear regressor and a nonlinear regressor (a two-layer MLP with ReLU activation). All classifiers and regressors are trained for \(10{,}000\) steps using the evaluation subset.

We assess the predictive performance of the learned representations by evaluating their ability to recover the ground-truth latent factors corresponding to their respective modalities. This evaluation accounts for the fact that some semantic variables may appear in discrete or continuous form, depending on the modality and rendering process.

Prediction performance is measured on the test set using the Matthews Correlation Coefficient (MCC) for discrete factors and the coefficient of determination (\(R^2\)) for continuous factors. Metrics are computed separately for each latent factor.

\begin{figure}[tbp]
\centering
    \includegraphics[width=0.45\textwidth]{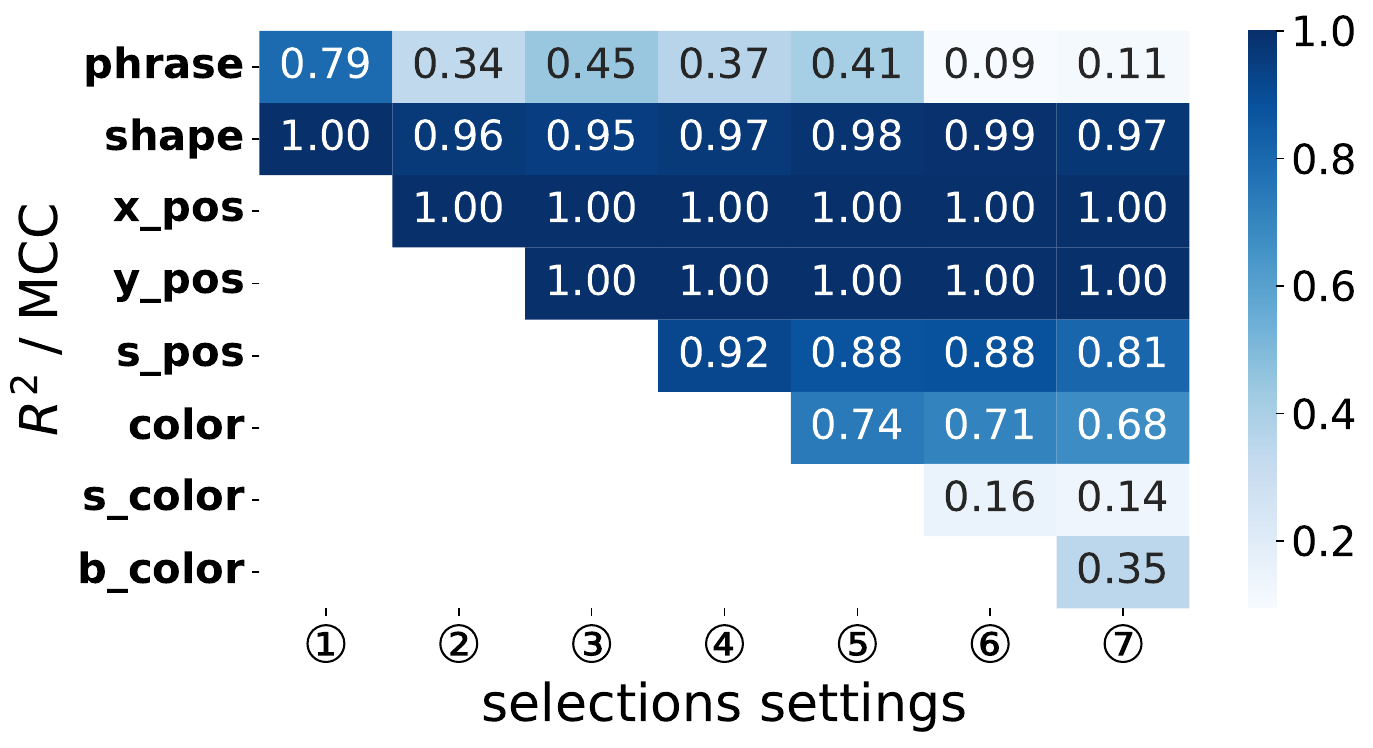}
    \hfill
    \includegraphics[width=0.433\textwidth]{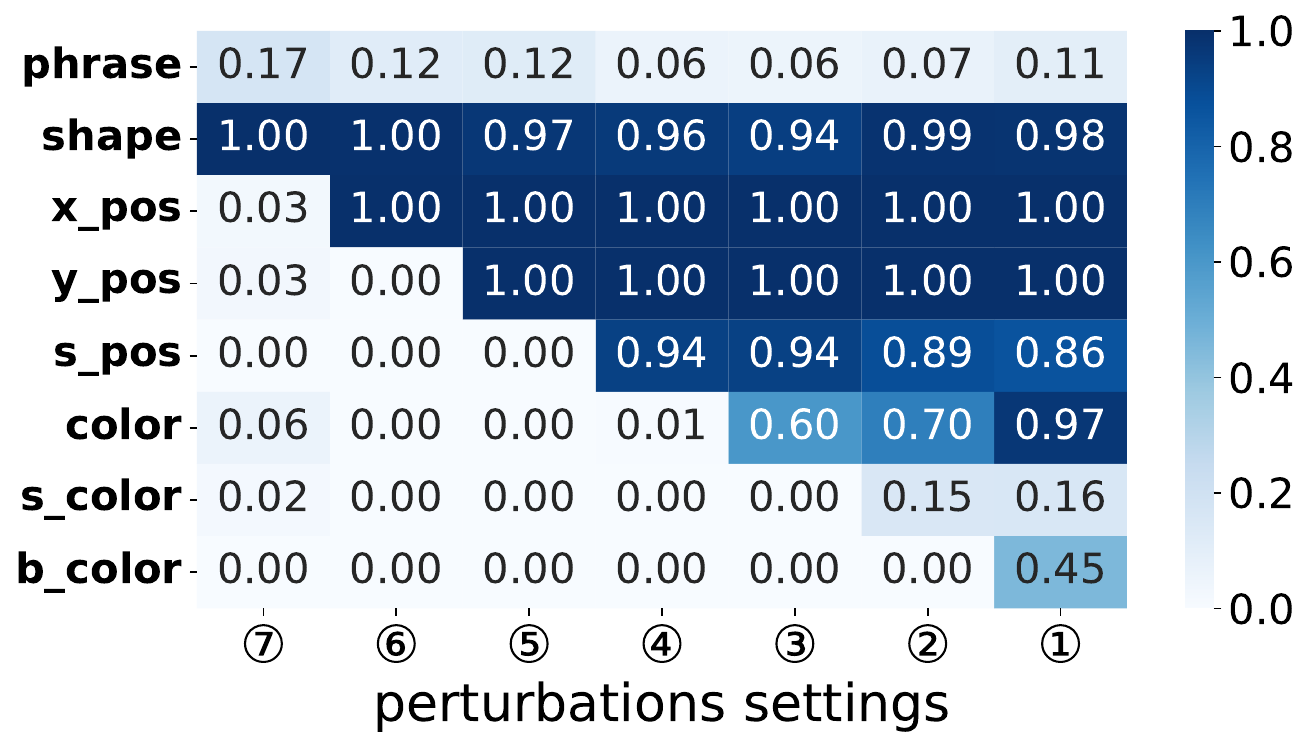}
\caption{Prediction of latent semantic variables under misalignment with text features. \(R^2\) for continuous and MCC for discrete factors. Left: Selection bias. Right: Perturbation bias.}
\label{fig:c3did_text}
\end{figure}

\subsection{Additional Results}
\label{sec:c3d_ablations}

\paragraph{Results of text representations.}
We now turn to the analysis of the text representations learned by MMCL. As shown in \Cref{fig:c3did_text}, we observe patterns similar to those found in the image modality. In particular, discrete latent semantic variables---such as \texttt{shape}, \texttt{x\_pos}, and \texttt{y\_pos}---are accurately predicted when selected and unperturbed. A notable case is \texttt{s\_pos}, whose continuous latent value is mapped to only five directional descriptions in the text observations. The relatively high \(R^2\) indicates that the learned representations retain useful positional information despite this quantization, while within-bin variation is not directly specified by these descriptions.

The discretized color attributes---\texttt{color}, \texttt{s\_color}, and \texttt{b\_color}---show varying predictive performance, with spotlight and background color generally less accurately predicted than object color. These text-side classification results concern categorical palette indices and are not directly comparable to the image-side regression scores for continuous hues. They motivate further examination of how color vocabularies, sentence structure, and encoder architecture affect the retention and accessibility of different semantic attributes.

When text-side \texttt{color} is independently resampled, accurate prediction of \texttt{x\_pos} and \texttt{y\_pos} does not imply accurate prediction of that resampled color value. In contrast, the original image-side color retains its statistical dependence on the positional factors. The observed difference between modalities is consistent with this change in the text-side dependence structure; the image-generating causal model remains unchanged. 

The text-specific factor \texttt{phrase} remains partially predictable from the learned text representations, particularly under selection bias. In comparison, the evaluated probes achieve low scores for the image-specific rotation factors. These results show that practical contrastive training can retain some template information alongside the aligned semantics.

Overall, these experiments provide qualitative empirical support for the effects of cross-modal misalignment identified by our analysis, in a practical setting with discrete factors and lossy text generation.

\begin{figure}[tbp]
\centering
    \includegraphics[width=0.45\textwidth]{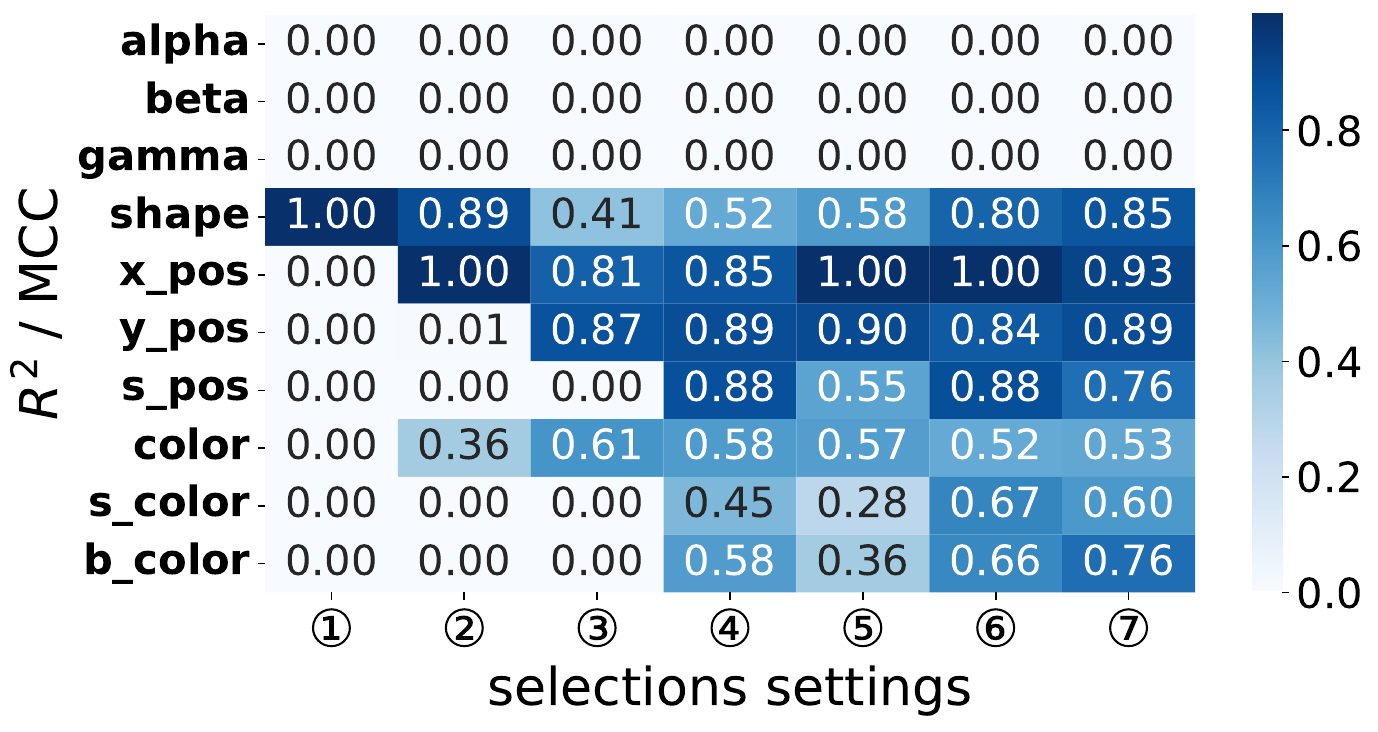}
    \hfill
    \includegraphics[width=0.43\textwidth]{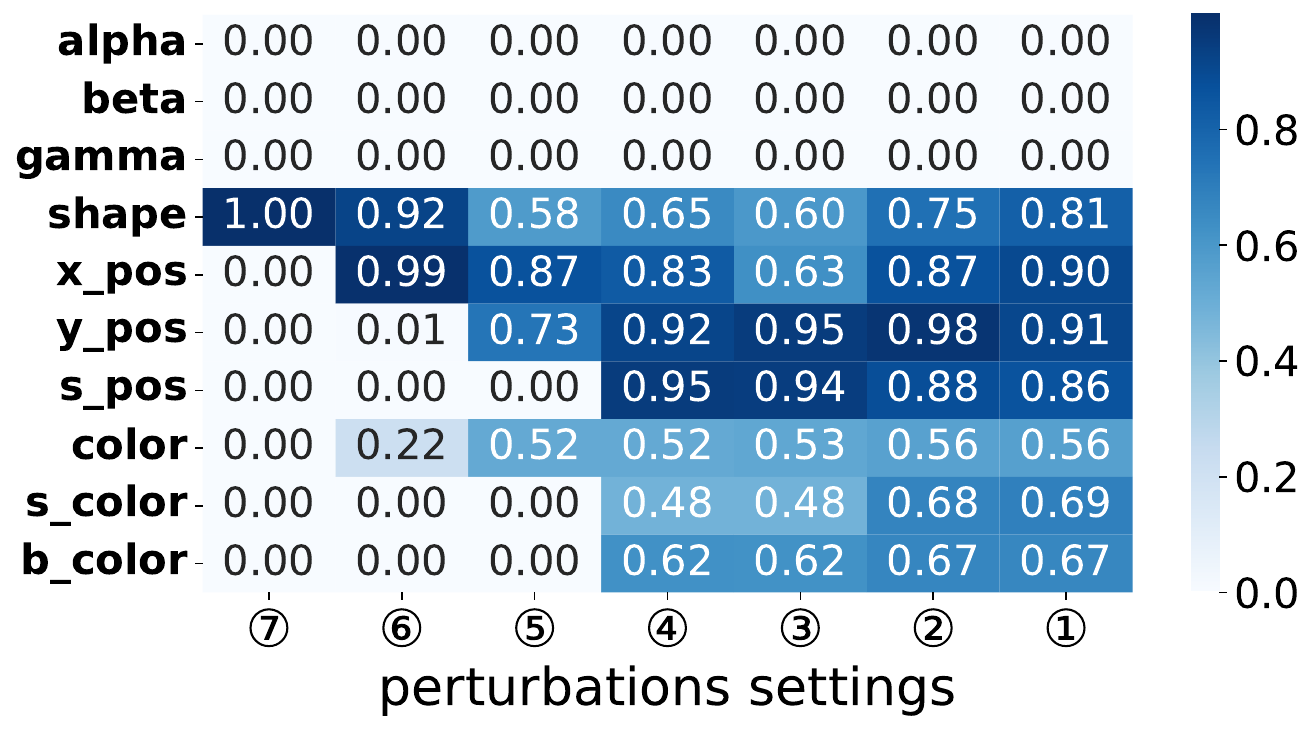}
\caption{Linear prediction of latent factors from image features under misalignment settings. \(R^2\) for continuous and MCC for discrete factors. Left: Selection bias. Right: Perturbation bias.}
\label{fig:c3did_image_linear}
\end{figure}

\begin{figure}[tbp]
\centering
    \includegraphics[width=0.45\textwidth]{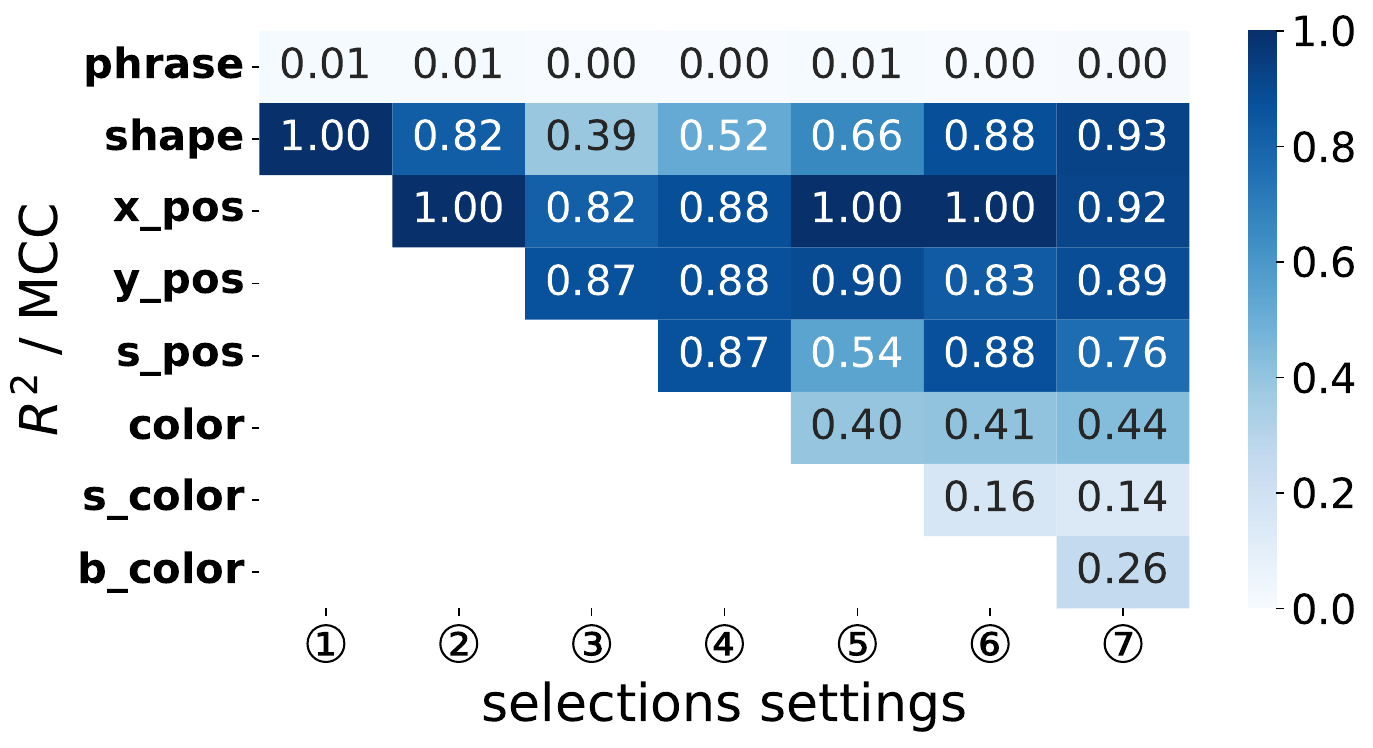}
    \hfill
    \includegraphics[width=0.43\textwidth]{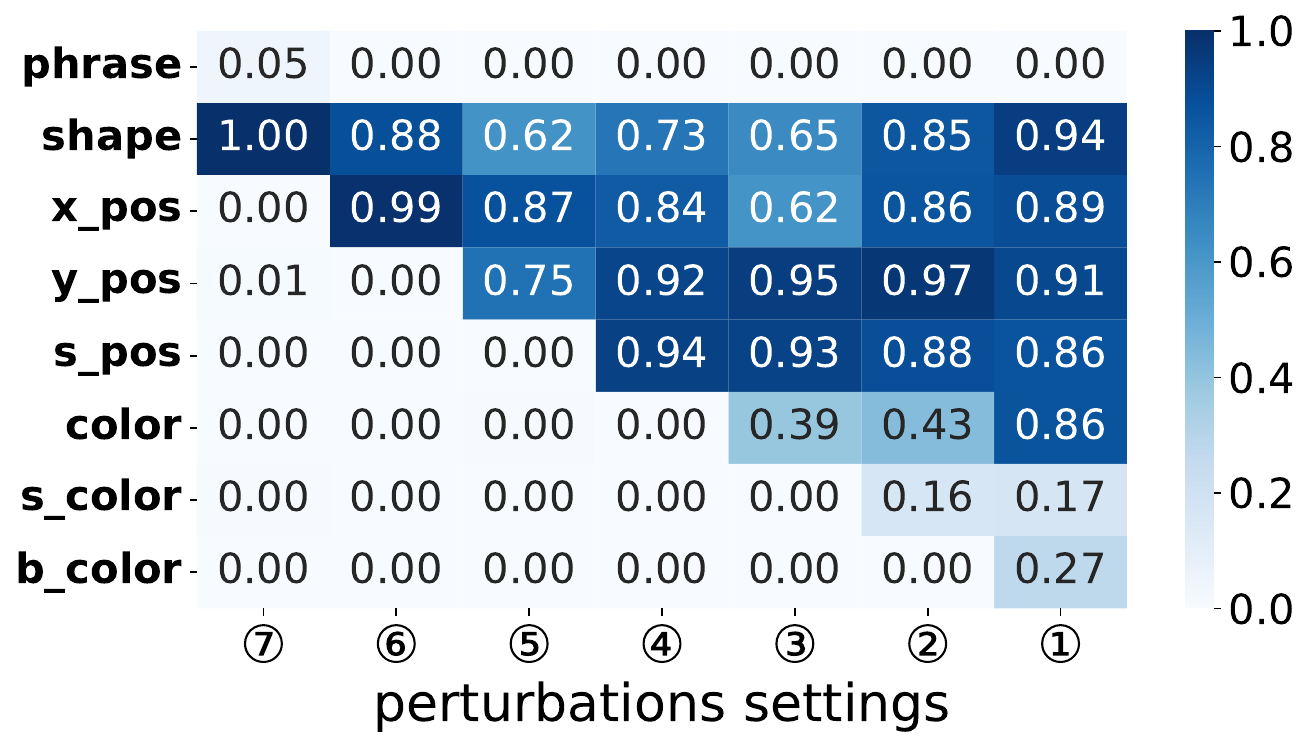}
\caption{Linear prediction of latent factors from text features under misalignment settings. \(R^2\) for continuous and MCC for discrete factors. Left: Selection bias. Right: Perturbation bias.}
\label{fig:c3did_text_linear}
\end{figure}

\paragraph{Linear predictability of latent factors.}
We evaluate linear classifiers and regressors trained to predict each latent factor from the learned image and text representations. Comparing \Cref{fig:c3did_image_linear,fig:c3did_text_linear} with the nonlinear prediction results shows that nonlinear probes often achieve higher scores, including for \texttt{color} when selected and unperturbed. Factors such as \texttt{x\_pos} and \texttt{y\_pos} are accurately predicted by linear classifiers in several settings. These results indicate that the accessibility of latent information to linear readouts varies across factors and experimental settings.


\section{Details on OpenCLIP Case Study}
\label{sec:app_casestudy}

\subsection{Concepts Taxonomy and Data Collection}
\label{sec:app_data_collect}
\paragraph{Concepts taxonomy.} To probe the representations of pretrained OpenCLIP models, we curated a set of 146 concepts organized into 15 distinct concept groups. Each concept group defines a separate set of candidate labels for zero-shot classification. Labels within a group may overlap semantically, and each image is evaluated against its assigned reference label. \Cref{tab:concept_taxonomy} presents all the concepts used in the case study. For simplicity, abbreviations are used in the main text, with their full meanings provided in the concept group column.

\begin{table}[tbp]
\caption{Taxonomy of concept categories with abbreviations, grouped by conceptual relevance.}
\centering
\small
\renewcommand{\arraystretch}{1.2}
\begin{tabularx}{\textwidth}{llX}
\toprule
\textbf{Group} & \textbf{Explanation} & \textbf{Concept Name} \\
\midrule
\texttt{Animal} & animals & dog, cat, horse, bird, elephant, giraffe, cow, zebra, rabbit, duck \\
\texttt{Clothing} & clothing & shirt, pants, dress, shoes, hat, jacket, skirt, tie, hoodie, socks \\
\texttt{Color} & colors & red, blue, green, yellow, black, white, orange, pink, gray, brown \\
\texttt{Food} & food & pizza, burger, sandwich, salad, cake, coffee, tea, beer, ice cream, noodles \\
\texttt{Object} & object & chair, table, phone, laptop, car, bottle, bag, cup, backpack, television \\
\texttt{Role} & roles & chef, teacher, athlete, doctor, engineer, artist, pilot, firefighter, police officer, lawyer \\
\texttt{Scene} & scenes & beach, kitchen, forest, street, park, office, bedroom, classroom, stadium, playground \\
\texttt{Vehicle} & vehicles & bus, truck, airplane, train, motorcycle, bicycle, boat, van, taxi, scooter \\
\texttt{Weather} & weather and lighting & sandstorm, drought, rain, snow, fog, rainbow, storm, moonlight, overcast, frost \\
\texttt{Texture} & texture & glossy, matte, rough, smooth, fuzzy, silky, grainy, wrinkled, slippery, furry \\
\texttt{POV} & viewpoint & close-up, wide shot, top-down, low angle, high angle, aerial, shallow-depth, long exposure \\
\texttt{Emot.} & mental state & tired, focused, surprised, proud, shy, bored, confused, excited, thoughtful, nervous \\
\texttt{Postproc.} & post-processing & hdr, desaturated, color-filtered, over-sharpened, saturated, sepia, heavily edited, bokeh, tilt-shift \\
\texttt{Trait} & moral bias & easy-going, cold-hearted, evil, criminal, immoral, pure, corrupt, hero, cheater, deserving \\
\texttt{Stere.} & stereotypes & exotic look, ethnic look, slacker, aggressive look, gangster look, terrorist look, illegal look, privileged, ghetto look \\
\bottomrule
\end{tabularx}
\label{tab:concept_taxonomy}
\end{table}

\paragraph{Data collection.} We collected up to 200 images per concept by using the concept names as search queries on Flickr, filtering for content licensed under Creative Commons. To ensure relevance, all images were manually reviewed, and those deemed unrelated to the queried concept were removed. The remaining images were retained for subsequent experiments. \Cref{tab:concept_image_counts_compact} summarizes the number of images retained for each concept. Some groups, particularly \texttt{Trait} and \texttt{Stere.}, contain subjective or socially loaded labels. These labels are used to examine associations between images and queried terms, rather than to establish personal attributes of depicted individuals. Their inclusion does not endorse the judgments or stereotypes expressed by those terms.

\begin{table}[htbp]
\caption{Concept-wise image counts grouped by concept category for zero-shot evaluation.}
\centering
\small
\label{tab:concept_image_counts_compact}

\begin{tabularx}{\textwidth}{@{} X X c | X X c | X X c @{}}
\toprule
\textbf{Group} & \textbf{Concept} & \textbf{Count} &
\textbf{Group} & \textbf{Concept} & \textbf{Count} &
\textbf{Group} & \textbf{Concept} & \textbf{Count} \\
\midrule
\texttt{Animal} & dog & 124 & \texttt{Animal} & cat & 112 & \texttt{Animal} & horse & 109 \\
\texttt{Animal} & bird & 106 & \texttt{Animal} & elephant & 98 & \texttt{Animal} & giraffe & 90 \\
\texttt{Animal} & cow & 88 & \texttt{Animal} & zebra & 81 & \texttt{Animal} & rabbit & 77 \\
\texttt{Animal} & duck & 65 & \texttt{Clothing} & shirt & 113 & \texttt{Clothing} & pants & 104 \\
\texttt{Clothing} & dress & 101 & \texttt{Clothing} & shoes & 96 & \texttt{Clothing} & hat & 92 \\
\texttt{Clothing} & jacket & 87 & \texttt{Clothing} & skirt & 84 & \texttt{Clothing} & tie & 78 \\
\texttt{Clothing} & hoodie & 72 & \texttt{Clothing} & socks & 68 & \texttt{Color} & red & 130 \\
\texttt{Color} & blue & 127 & \texttt{Color} & green & 119 & \texttt{Color} & yellow & 114 \\
\texttt{Color} & black & 111 & \texttt{Color} & white & 109 & \texttt{Color} & orange & 97 \\
\texttt{Color} & pink & 91 & \texttt{Color} & gray & 88 & \texttt{Color} & brown & 82 \\
\texttt{Food} & pizza & 118 & \texttt{Food} & burger & 113 & \texttt{Food} & sandwich & 107 \\
\texttt{Food} & salad & 101 & \texttt{Food} & cake & 95 & \texttt{Food} & coffee & 93 \\
\texttt{Food} & tea & 90 & \texttt{Food} & beer & 85 & \texttt{Food} & ice cream & 82 \\
\texttt{Food} & noodles & 78 & \texttt{Emot.} & tired & 58 & \texttt{Emot.} & focused & 55 \\
\texttt{Emot.} & surprised & 51 & \texttt{Emot.} & proud & 49 & \texttt{Emot.} & shy & 45 \\
\texttt{Emot.} & bored & 44 & \texttt{Emot.} & confused & 42 & \texttt{Emot.} & excited & 40 \\
\texttt{Emot.} & thoughtful & 39 & \texttt{Emot.} & nervous & 36 & \texttt{Trait} & easy-going & 29 \\
\texttt{Trait} & cold-hearted & 27 & \texttt{Trait} & evil & 25 & \texttt{Trait} & criminal & 24 \\
\texttt{Trait} & immoral & 23 & \texttt{Trait} & pure & 22 & \texttt{Trait} & corrupt & 20 \\
\texttt{Trait} & hero & 18 & \texttt{Trait} & cheater & 17 & \texttt{Trait} & deserving & 15 \\
\texttt{Object} & chair & 120 & \texttt{Object} & table & 115 & \texttt{Object} & phone & 112 \\
\texttt{Object} & laptop & 110 & \texttt{Object} & car & 108 & \texttt{Object} & bottle & 105 \\
\texttt{Object} & bag & 100 & \texttt{Object} & cup & 97 & \texttt{Object} & backpack & 93 \\
\texttt{Object} & television & 90 & \texttt{Postproc.} & hdr & 33 & \texttt{Postproc.} & desaturated & 32 \\
\texttt{Postproc.} & color-filtered & 31 & \texttt{Postproc.} & over-sharpened & 30 & \texttt{Postproc.} & saturated & 28 \\
\texttt{Postproc.} & sepia & 27 & \texttt{Postproc.} & heavily edited & 26 & \texttt{Postproc.} & bokeh & 24 \\
\texttt{Postproc.} & tilt-shift & 22 & \texttt{Role} & chef & 88 & \texttt{Role} & teacher & 86 \\
\texttt{Role} & athlete & 85 & \texttt{Role} & doctor & 83 & \texttt{Role} & engineer & 80 \\
\texttt{Role} & artist & 78 & \texttt{Role} & pilot & 74 & \texttt{Role} & firefighter & 71 \\
\texttt{Role} & police officer & 68 & \texttt{Role} & lawyer & 66 & \texttt{Scene} & beach & 120 \\
\texttt{Scene} & kitchen & 115 & \texttt{Scene} & forest & 110 & \texttt{Scene} & street & 107 \\
\texttt{Scene} & park & 104 & \texttt{Scene} & office & 102 & \texttt{Scene} & bedroom & 98 \\
\texttt{Scene} & classroom & 95 & \texttt{Scene} & stadium & 92 & \texttt{Scene} & playground & 89 \\
\texttt{Stere.} & exotic look & 21 & \texttt{Stere.} & ethnic look & 20 & \texttt{Stere.} & slacker & 19 \\
\texttt{Stere.} & aggressive look & 18 & \texttt{Stere.} & gangster look & 17 & \texttt{Stere.} & terrorist look & 16 \\
\texttt{Stere.} & illegal look & 15 & \texttt{Stere.} & privileged & 14 & \texttt{Stere.} & ghetto look & 13 \\
\texttt{Texture} & glossy & 47 & \texttt{Texture} & matte & 45 & \texttt{Texture} & rough & 43 \\
\texttt{Texture} & smooth & 41 & \texttt{Texture} & fuzzy & 40 & \texttt{Texture} & silky & 38 \\
\texttt{Texture} & grainy & 36 & \texttt{Texture} & wrinkled & 34 & \texttt{Texture} & slippery & 33 \\
\texttt{Texture} & furry & 30 & \texttt{Vehicle} & bus & 106 & \texttt{Vehicle} & truck & 104 \\
\texttt{Vehicle} & airplane & 102 & \texttt{Vehicle} & train & 100 & \texttt{Vehicle} & motorcycle & 97 \\
\texttt{Vehicle} & bicycle & 93 & \texttt{Vehicle} & boat & 89 & \texttt{Vehicle} & van & 86 \\
\texttt{Vehicle} & taxi & 83 & \texttt{Vehicle} & scooter & 80 & \texttt{POV} & close-up & 55 \\
\texttt{POV} & wide shot & 52 & \texttt{POV} & top-down & 50 & \texttt{POV} & low angle & 48 \\
\texttt{POV} & high angle & 47 & \texttt{POV} & aerial & 44 & \texttt{POV} & shallow-depth & 42 \\
\texttt{POV} & long exposure & 39 & \texttt{Weather} & sandstorm & 66 & \texttt{Weather} & drought & 63 \\
\texttt{Weather} & rain & 61 & \texttt{Weather} & snow & 58 & \texttt{Weather} & fog & 56 \\
\texttt{Weather} & rainbow & 53 & \texttt{Weather} & storm & 50 & \texttt{Weather} & moonlight & 47 \\
\texttt{Weather} & overcast & 45 & \texttt{Weather} & frost & 42 & & & \\
\bottomrule
\end{tabularx}
\end{table}

\subsection{Details on Caption Frequency Statistics}
\label{sec:app_cap_stats}
For each concept, we estimate its lexical mention frequency in the analyzed LAION-400M captions using a predefined list of terms and related expressions. Each caption contributes at most one match per concept. The concept-level rate is the matched-caption count divided by the total number of analyzed captions, and the group-level rate is the arithmetic mean of its concept-level rates. \Cref{fig:group_converage} in \Cref{sec:openclip} presents these group-level averages. These statistics measure lexical mentions rather than the probability that a visually present concept is described.



\Cref{fig:concept_rate} shows substantial variation in lexical mention frequency both within and across concept groups. Terms associated with object categories, colors, and food generally occur more frequently than terms associated with mental states, post-processing, moral judgments, and stereotypes. These differences may reflect variation in concept prevalence, vocabulary, and captioning practices; caption frequencies alone do not distinguish these explanations.

\begin{figure}
    \centering
    \includegraphics[width=0.32\linewidth]{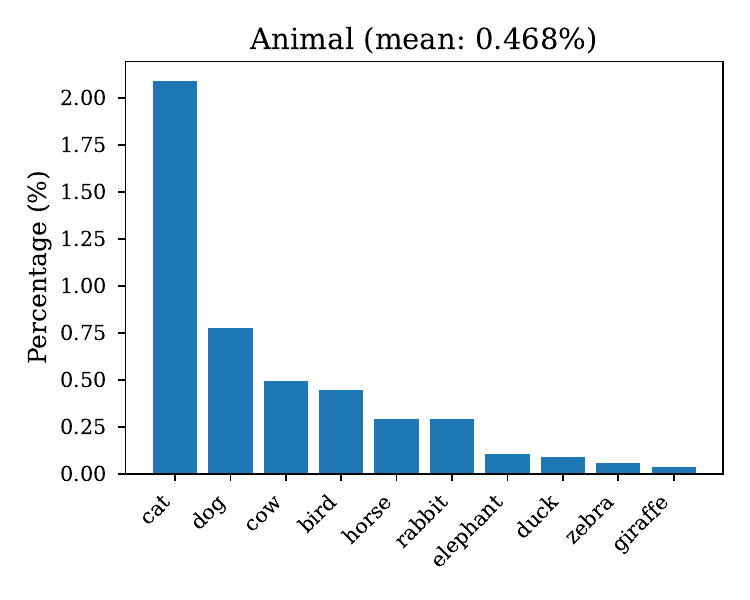}
    \hfill
    \includegraphics[width=0.32\linewidth]{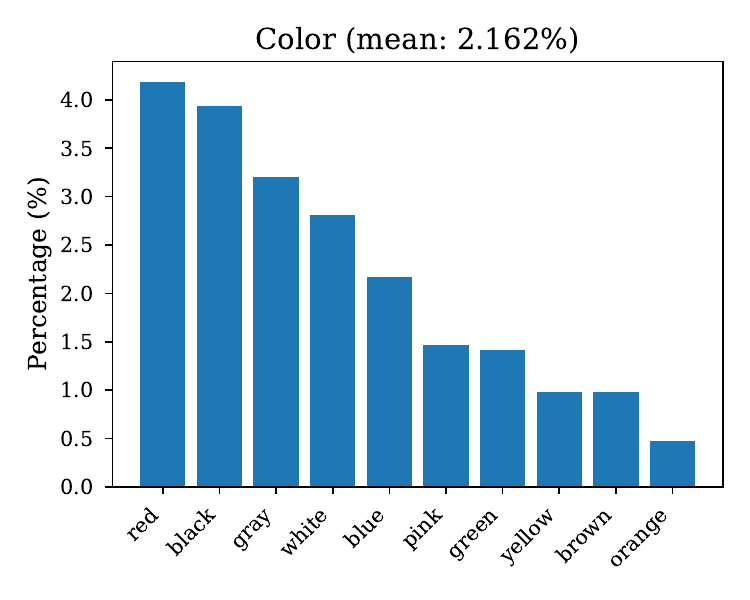}
    \hfill
    \includegraphics[width=0.32\linewidth]{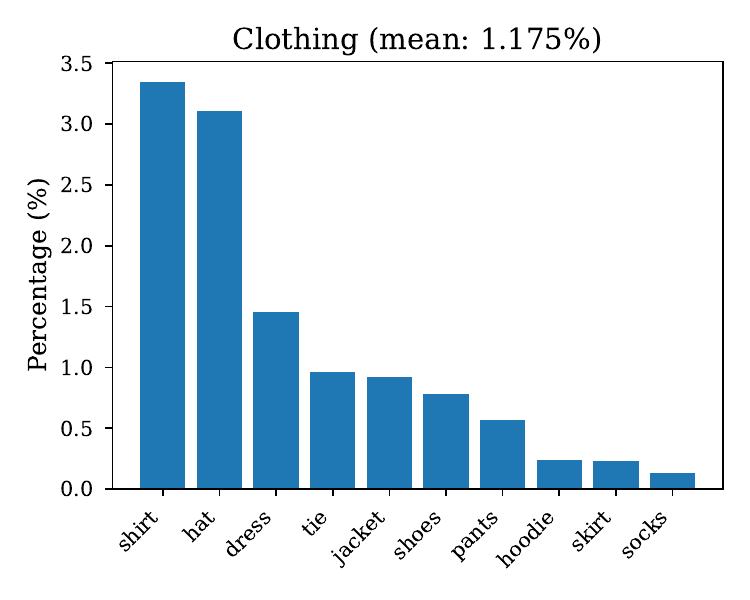}
    \\
    \includegraphics[width=0.32\linewidth]{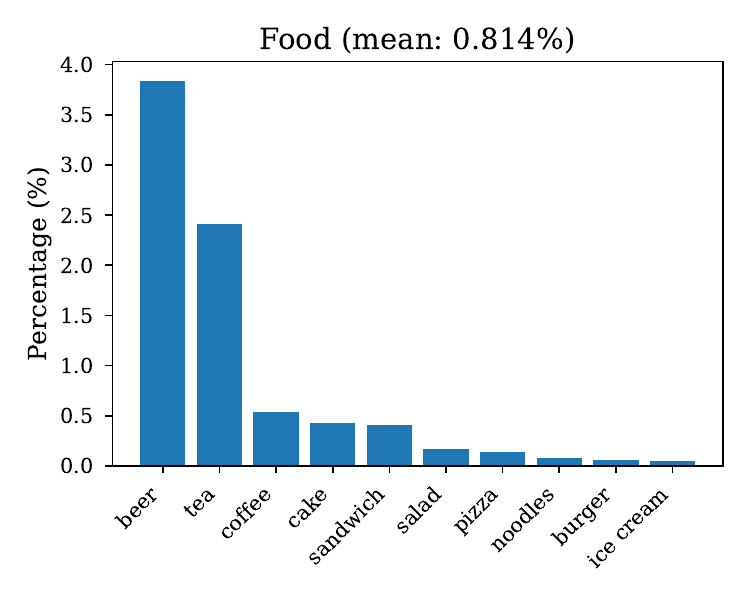}
    \hfill
    \includegraphics[width=0.32\linewidth]{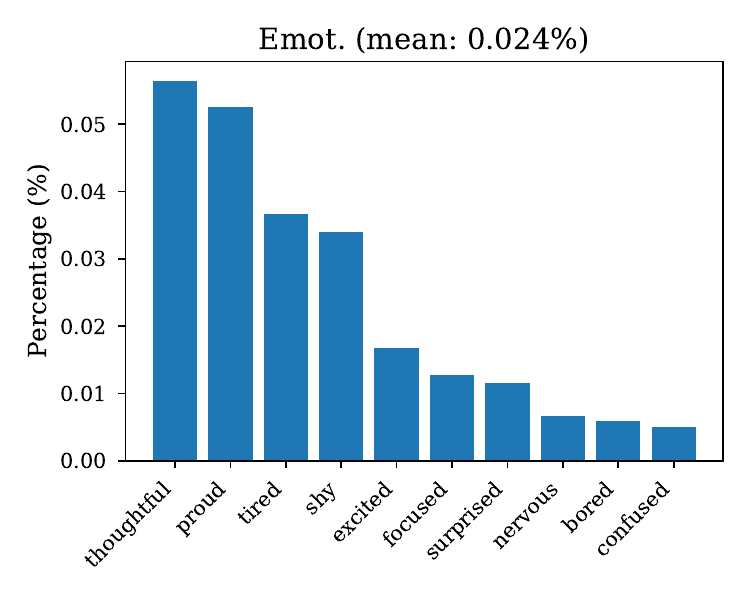}
    \hfill
    \includegraphics[width=0.32\linewidth]{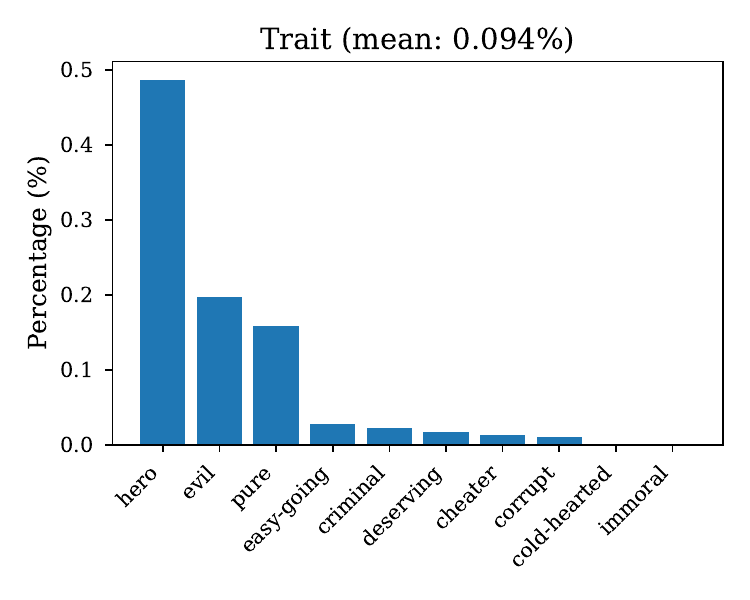}
    \\
    \includegraphics[width=0.32\linewidth]{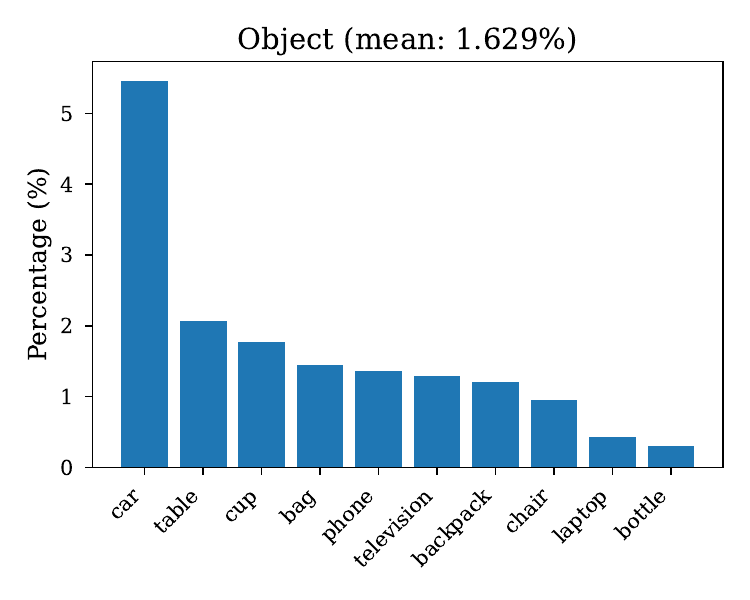}
    \hfill
    \includegraphics[width=0.32\linewidth]{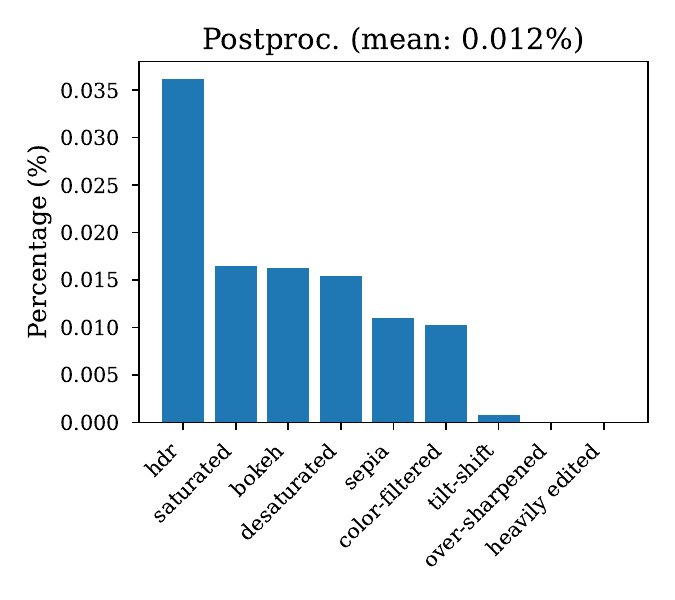}
    \hfill
    \includegraphics[width=0.32\linewidth]{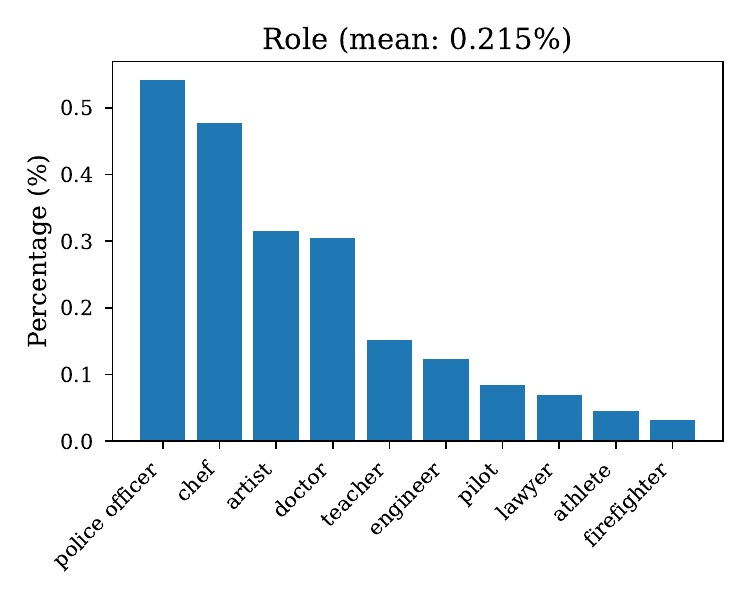}
    \\
    \includegraphics[width=0.32\linewidth]{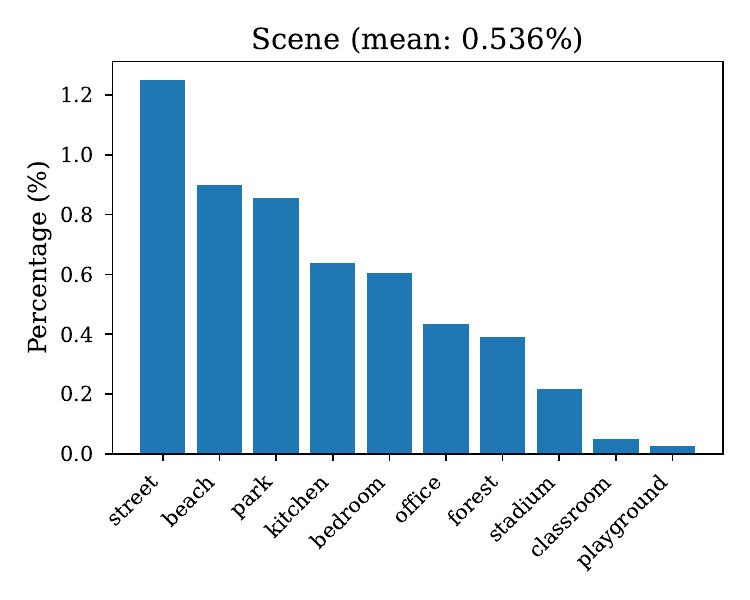}
    \hfill
    \includegraphics[width=0.32\linewidth]{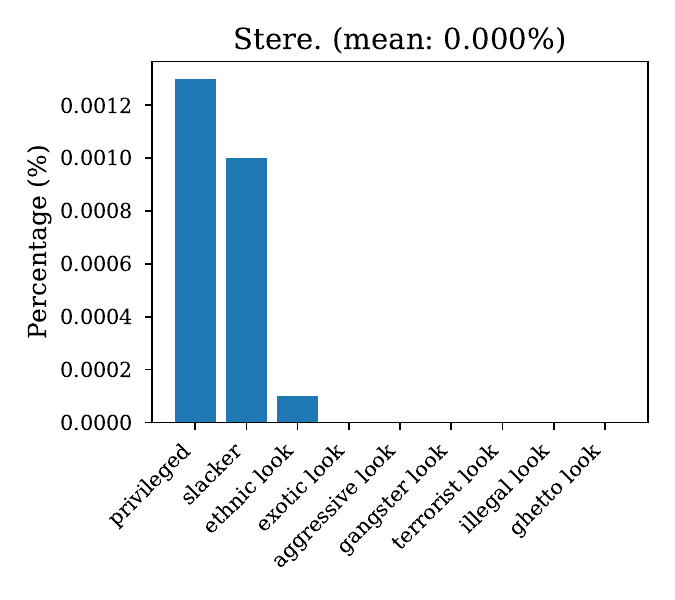}
    \hfill
    \includegraphics[width=0.32\linewidth]{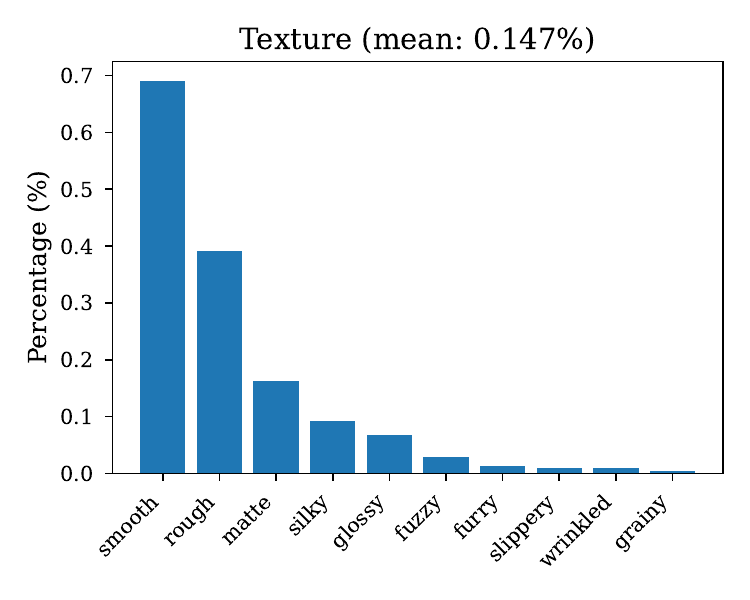}
    \\
    \includegraphics[width=0.32\linewidth]{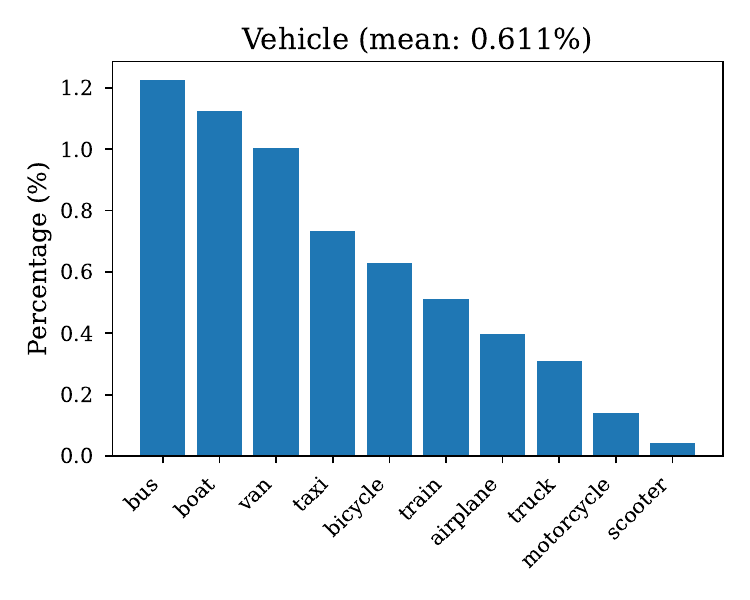}
    \hfill
    \includegraphics[width=0.28\linewidth]{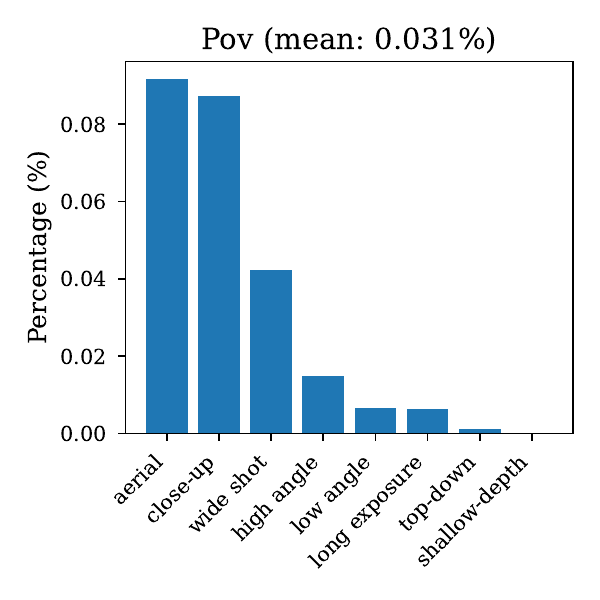}
    \hfill
    \includegraphics[width=0.32\linewidth]{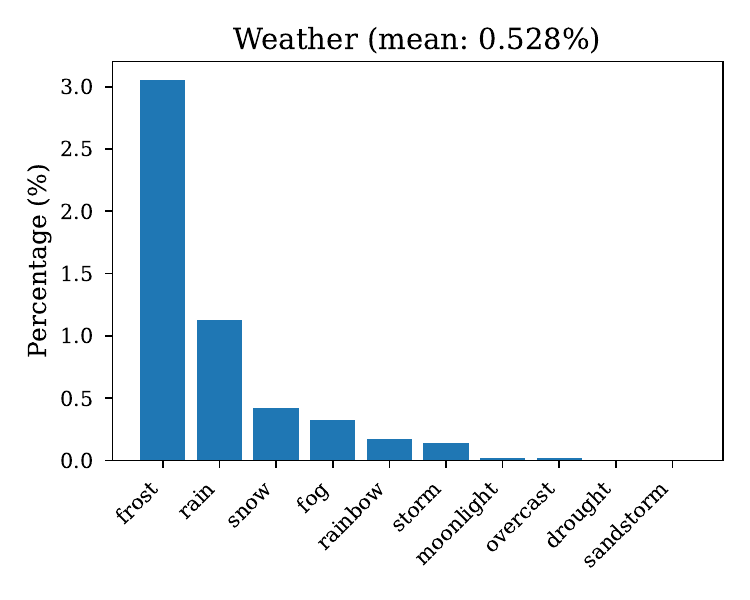}
    \caption{Concept-level lexical mention frequencies in the analyzed LAION-400M captions, grouped by category.}
    \label{fig:concept_rate}
\end{figure}

\subsection{Zero-Shot Evaluation Details}
\label{sec:app_zs_details}

We evaluate pretrained OpenCLIP ViT-B-32 and ViT-L-14 using the \texttt{laion400m\_e32} checkpoints. Each concept group defines a separate classification task. The stored concept labels are used directly as text prompts, and each image is assigned the label with the highest cosine similarity between normalized image and text embeddings. We report group-level macro-averaged F1 scores to give each concept equal weight despite unequal class sizes.


As shown in \Cref{fig:f1_bar}, groups with higher caption-term frequencies generally achieve higher zero-shot macro-F1 scores, while several less frequently mentioned groups perform substantially worse. The broad pattern appears in both model variants and is qualitatively consistent with the importance of textual supervision highlighted by our analysis. This observational comparison does not isolate the effect of caption omission, since the groups also differ in visual recognizability, label ambiguity, and compatibility with the chosen prompts.


\Cref{fig:confmat} presents within-group confusion matrices for OpenCLIP ViT-B-32. Higher-performing groups exhibit stronger diagonal patterns, whereas several groups with lower caption-term frequencies show greater confusion among their candidate labels.


For potentially useful visual attributes, such as those in \texttt{Texture}, lower zero-shot performance motivates investigating whether more explicit captioning improves recognition. For subjective groups such as \texttt{Emot.}, \texttt{Trait}, and \texttt{Stere.}, interpretation also depends on whether the labels are visually grounded and consistently assigned. Although reduced reproduction of certain stereotypical associations may be desirable, low classification scores alone do not establish that those associations are absent from the representations or that the model is ethically aligned. These observations motivate studying caption selection as a means of shaping learned associations, with its effects assessed directly.


\begin{figure}[tbp]
    \centering
    \includegraphics[width=0.325\linewidth]{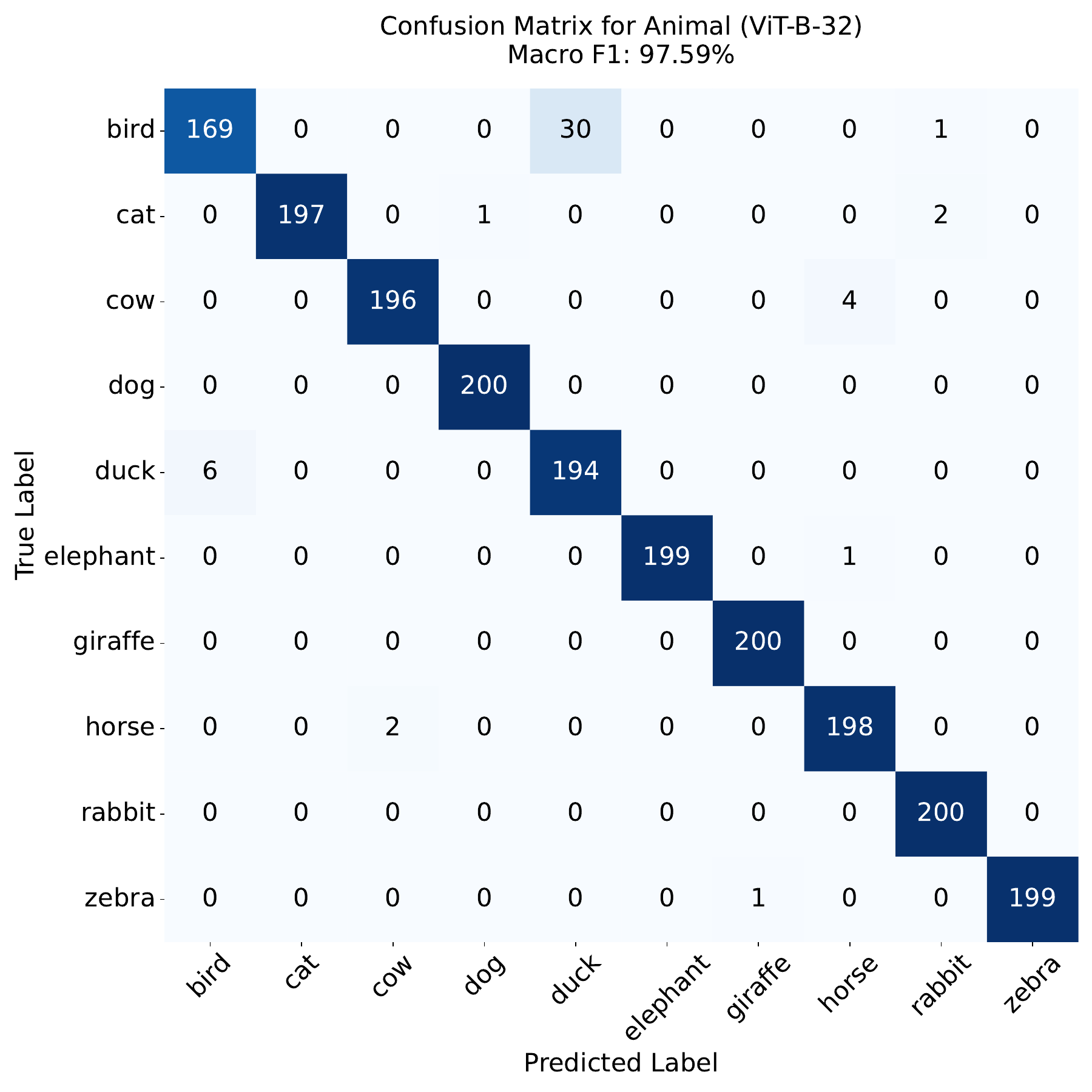}
    \hfill
    \includegraphics[width=0.325\linewidth]{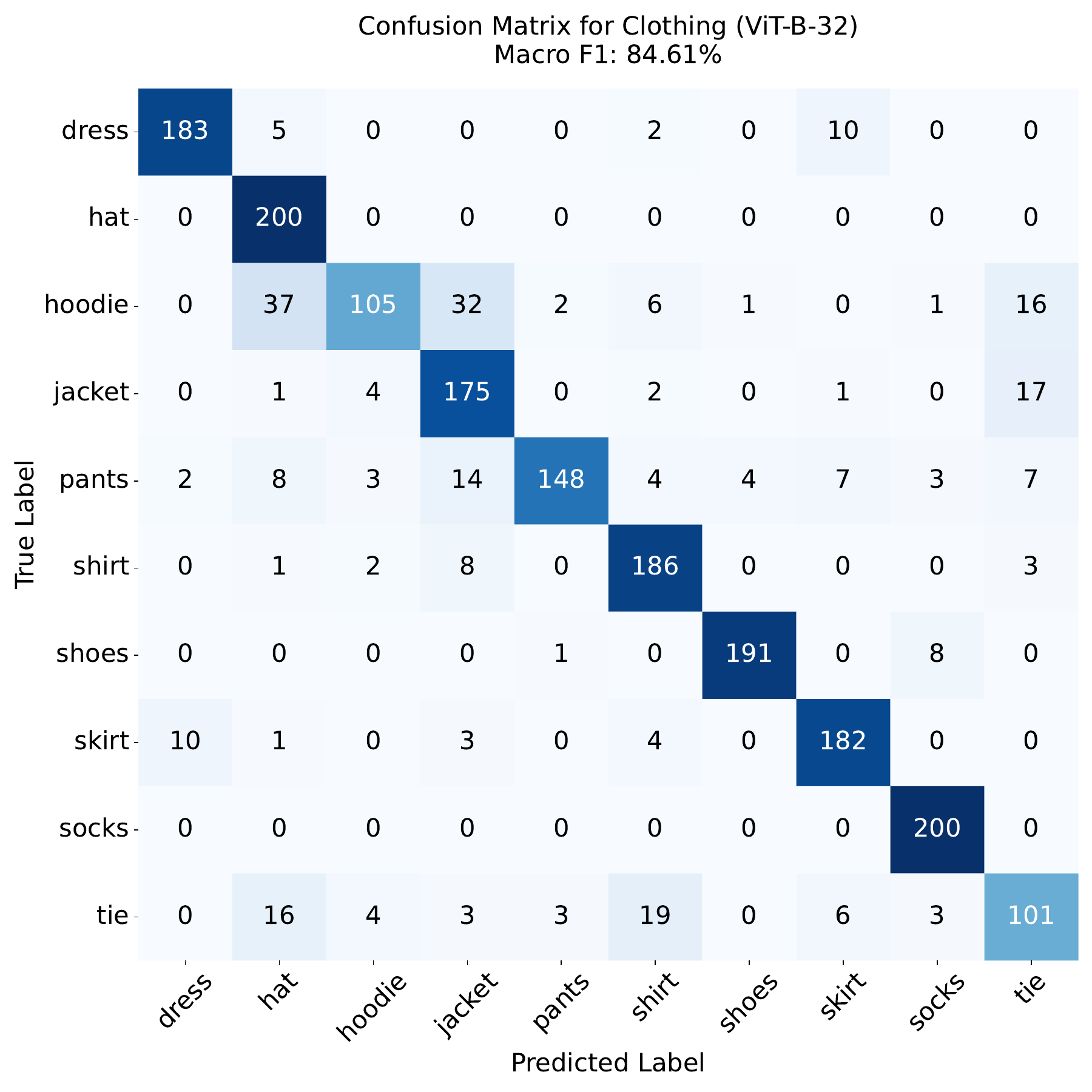}
    \hfill
    \includegraphics[width=0.325\linewidth]{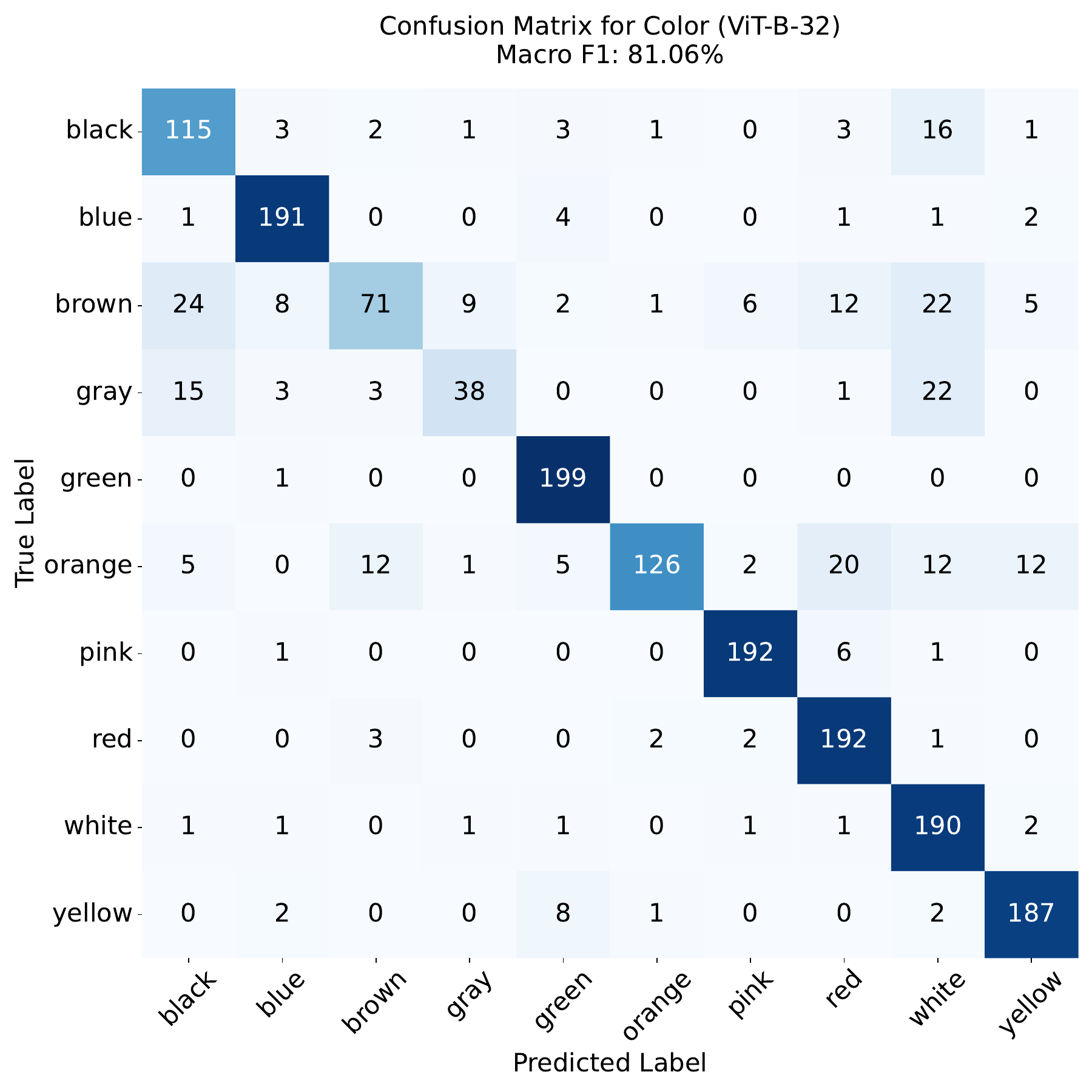}
    \\
    \includegraphics[width=0.325\linewidth]{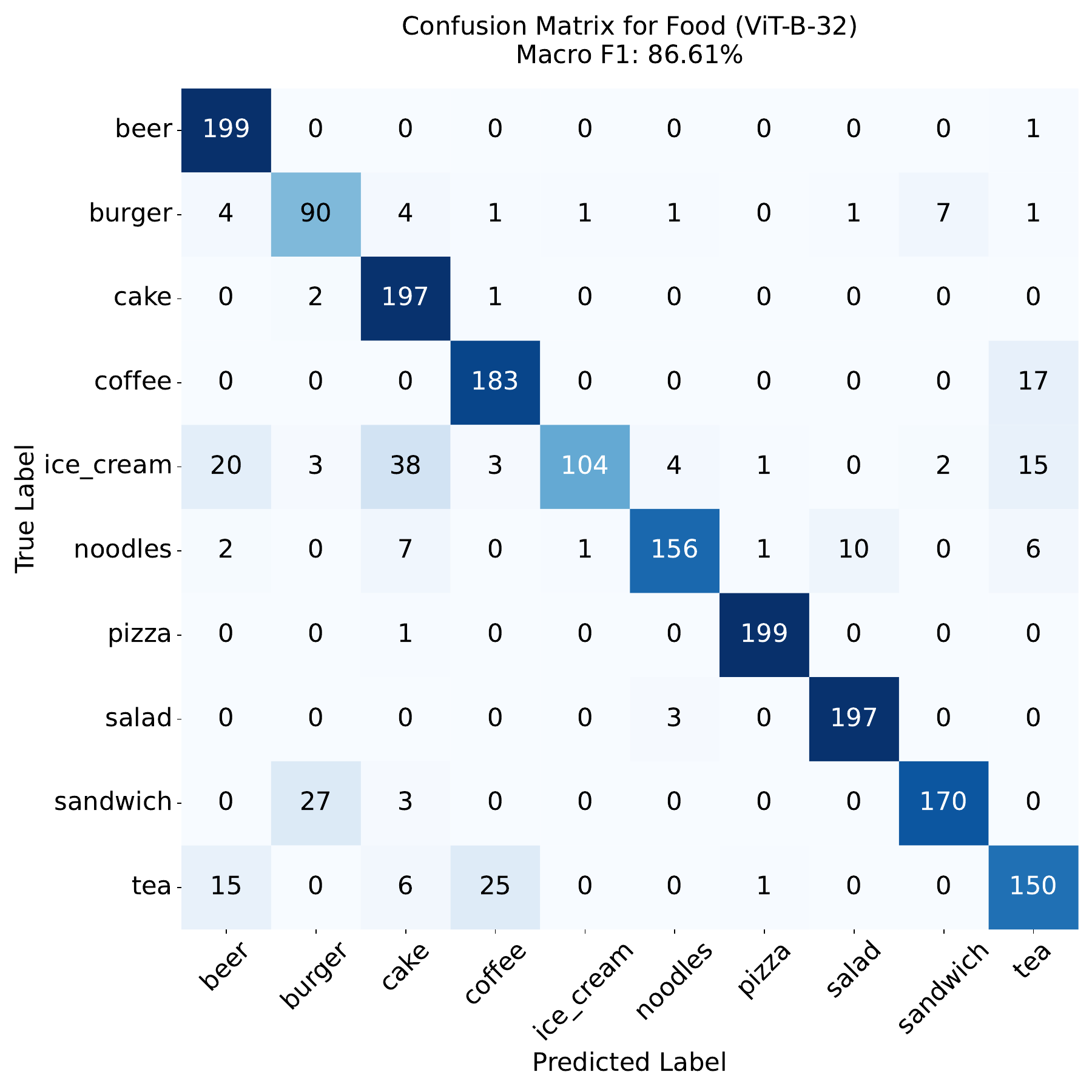}
    \hfill
    \includegraphics[width=0.325\linewidth]{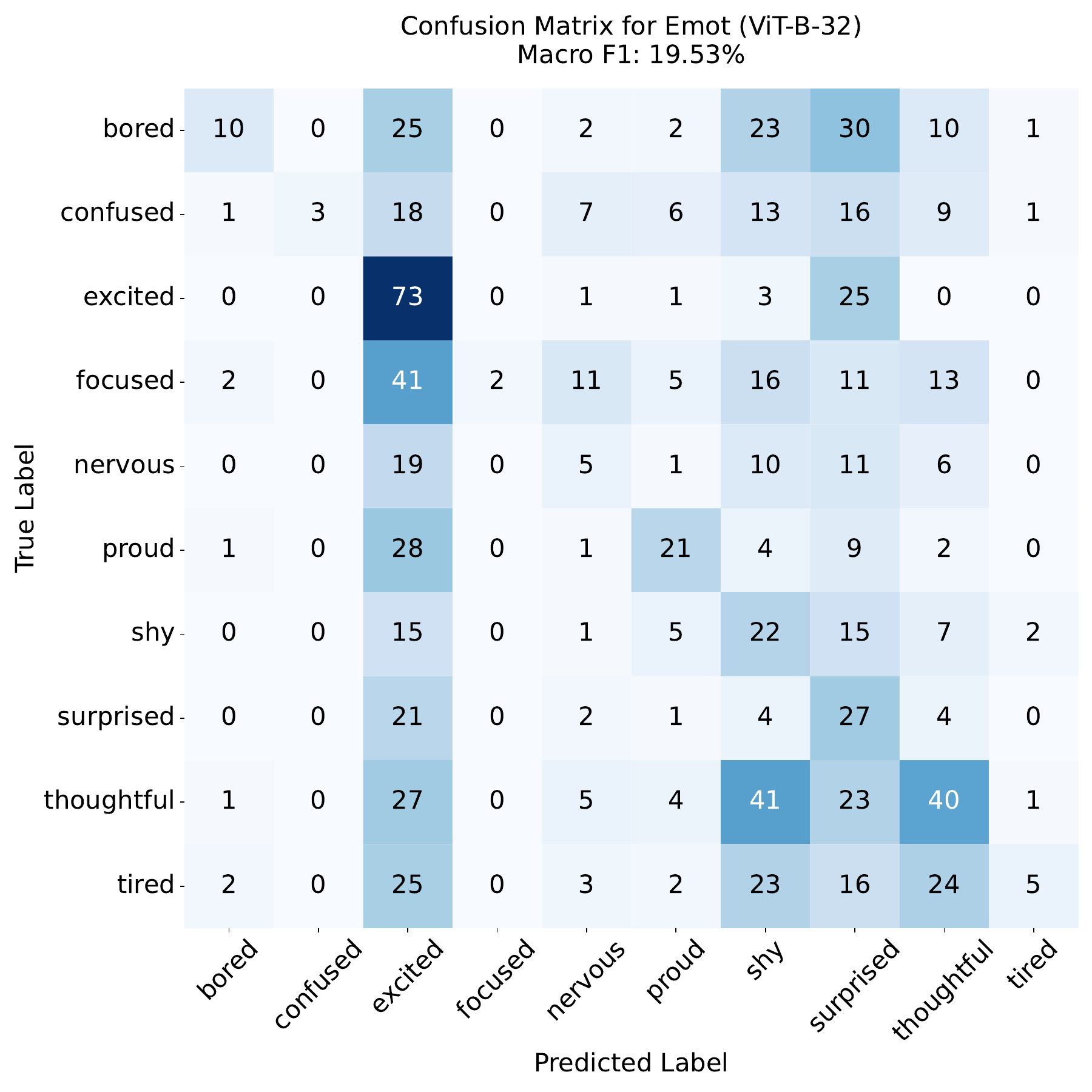}
    \hfill
    \includegraphics[width=0.325\linewidth]{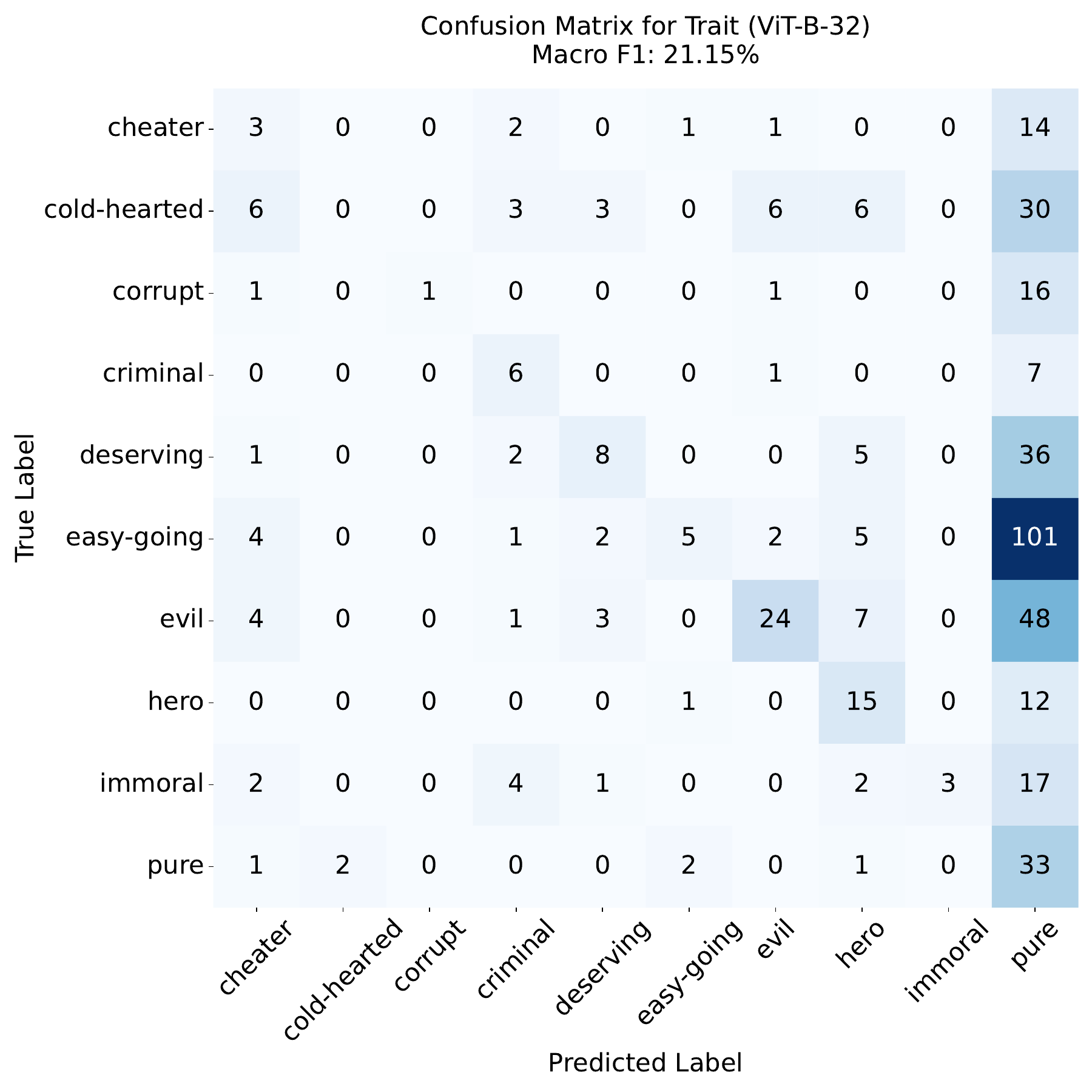}
    \\
    \includegraphics[width=0.325\linewidth]{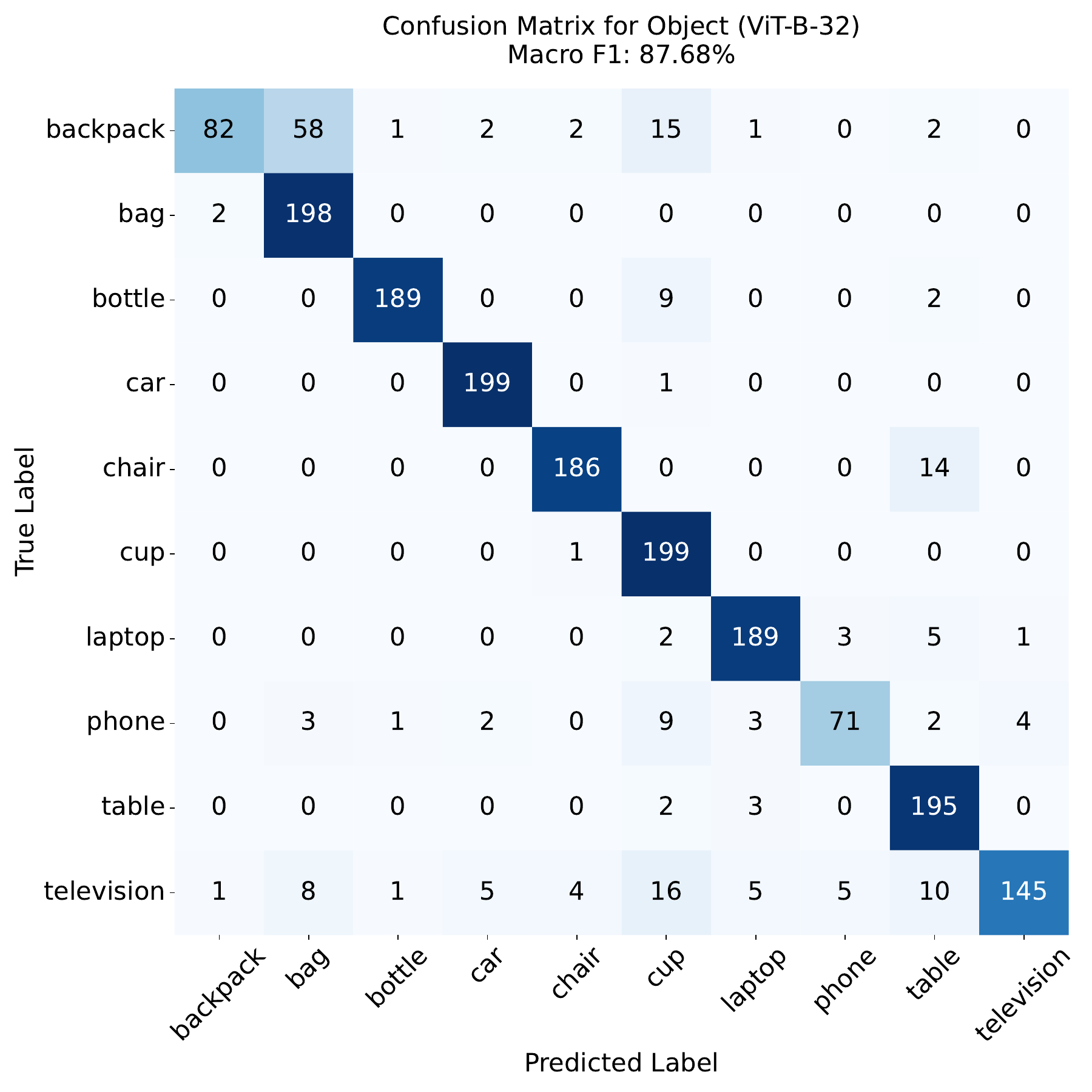}
    \hfill
    \includegraphics[width=0.325\linewidth]{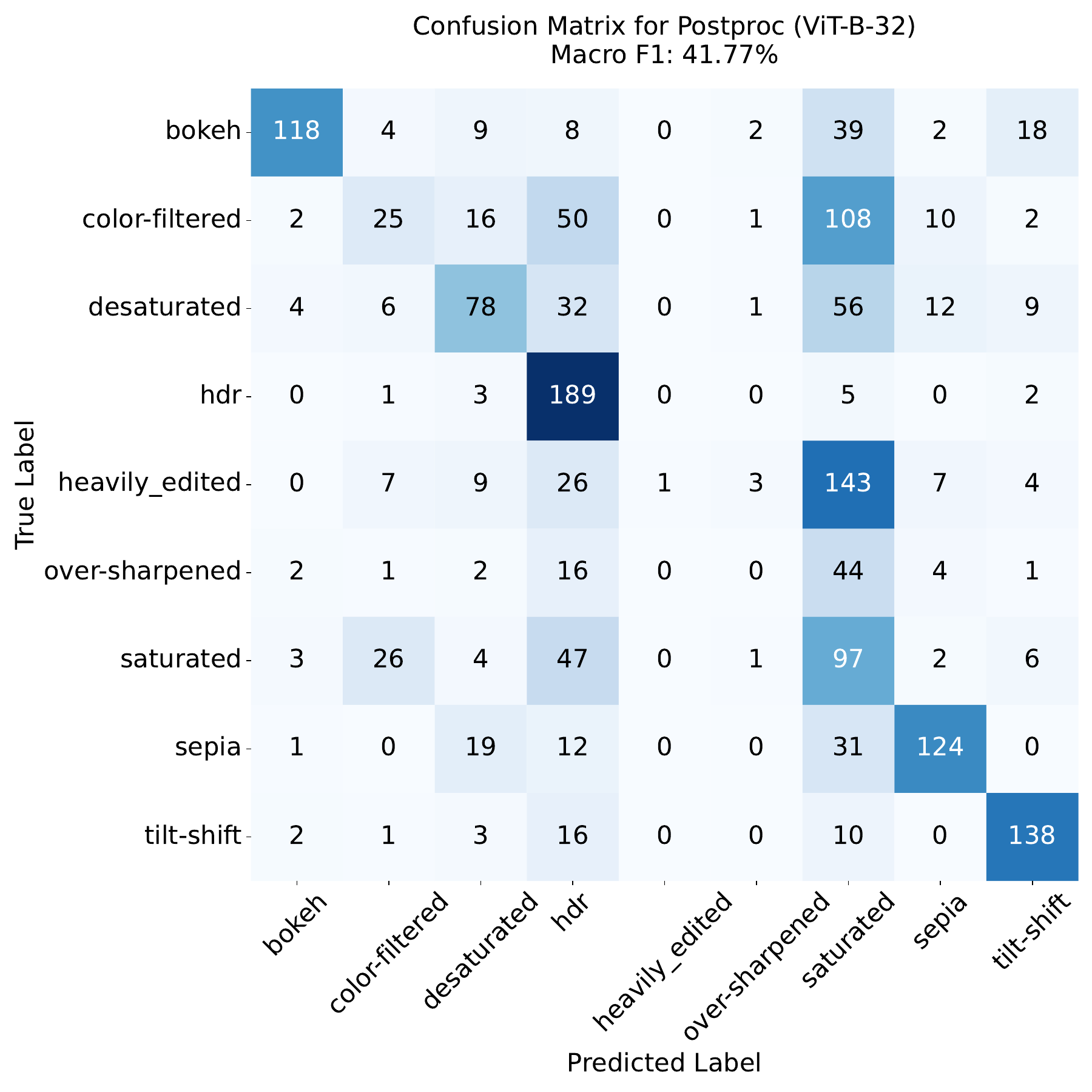}
    \hfill
    \includegraphics[width=0.325\linewidth]{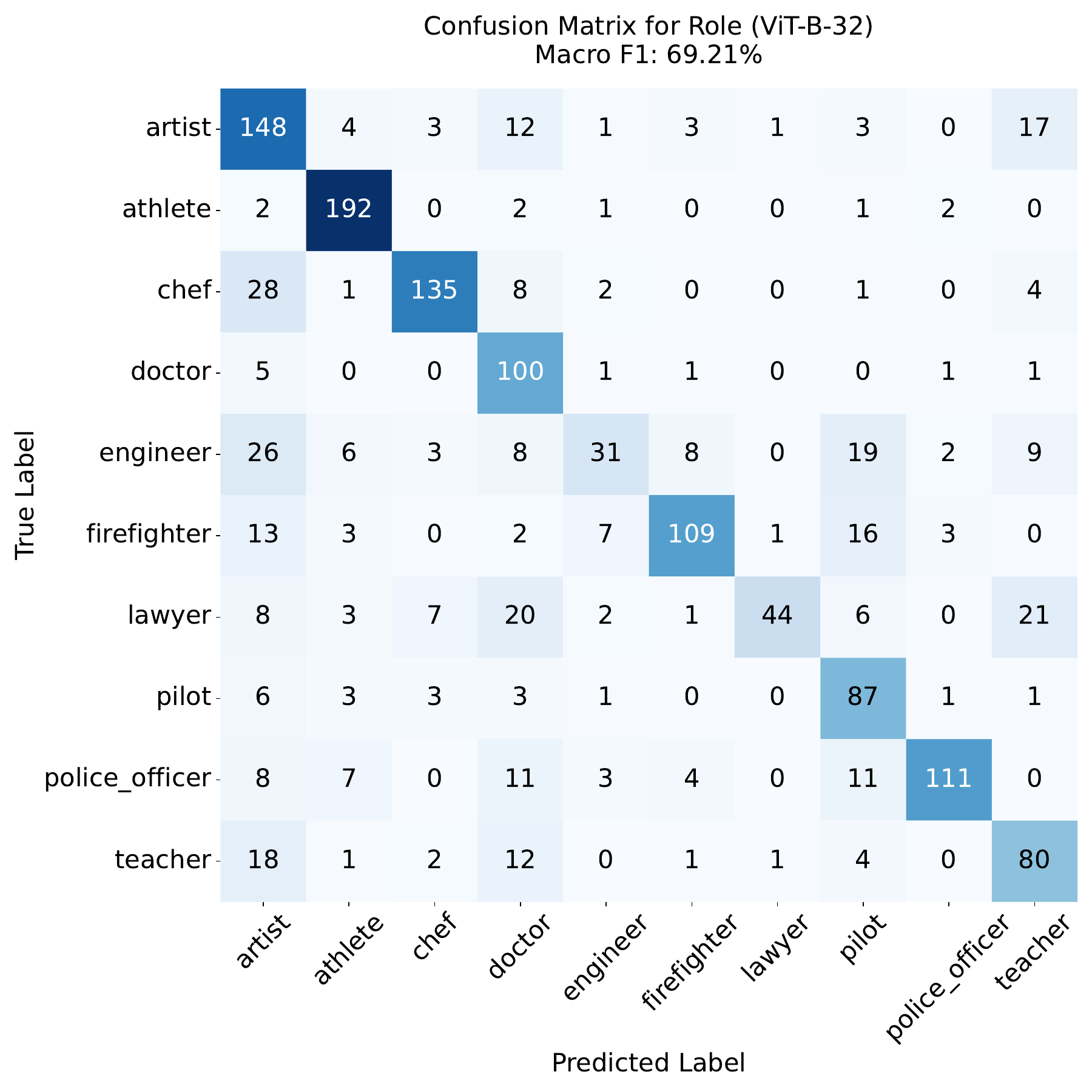}
    \\
    \includegraphics[width=0.325\linewidth]{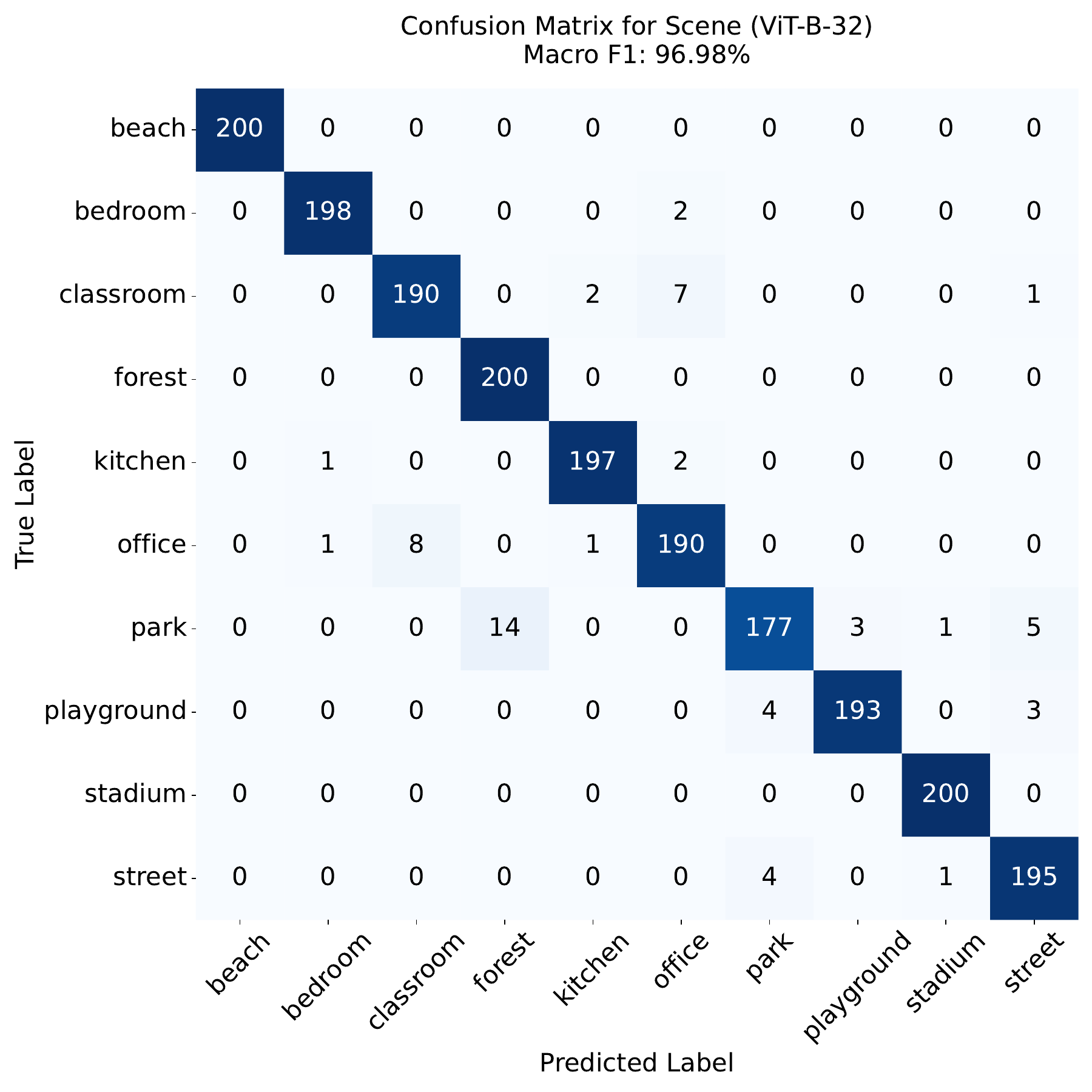}
    \hfill
    \includegraphics[width=0.325\linewidth]{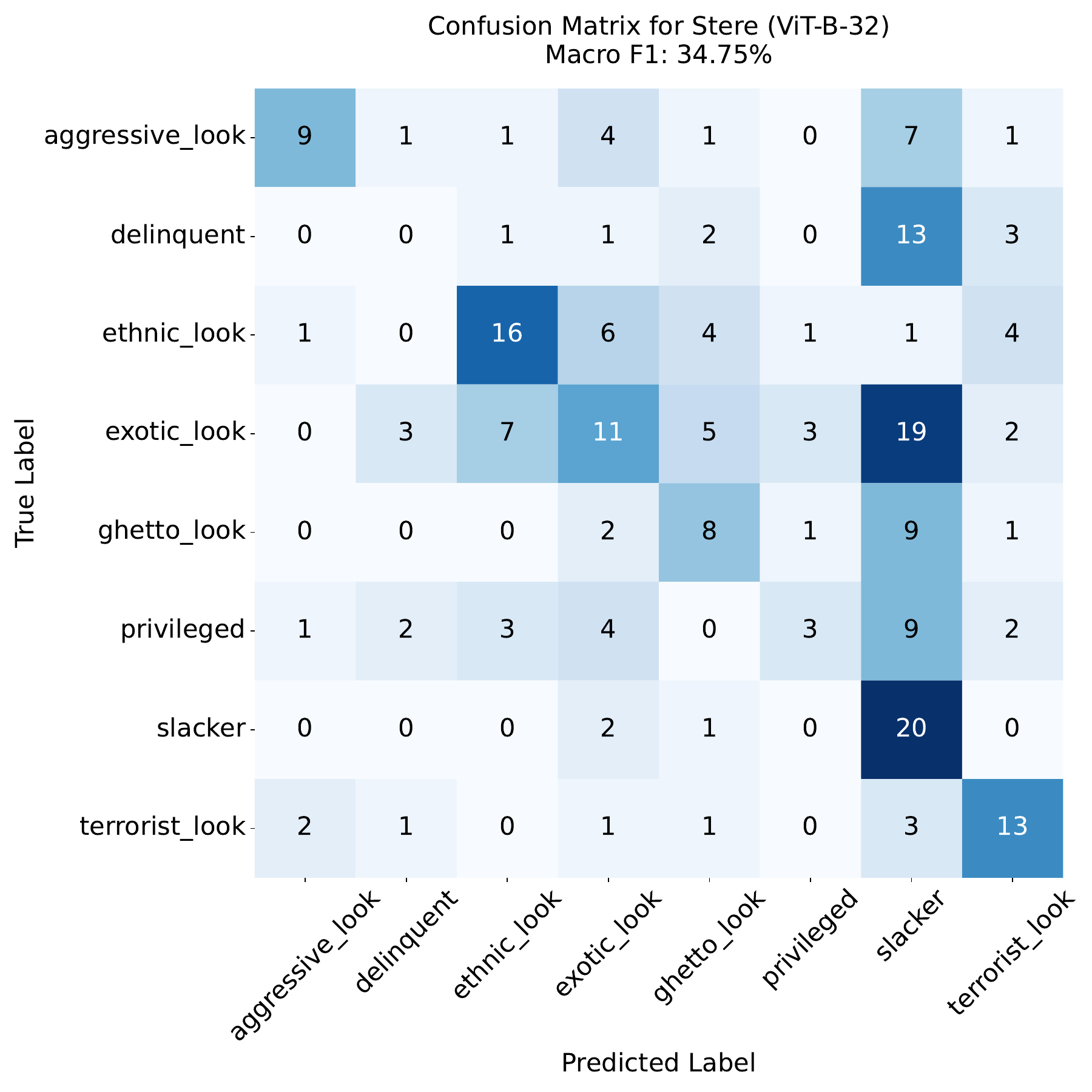}
    \hfill
    \includegraphics[width=0.325\linewidth]{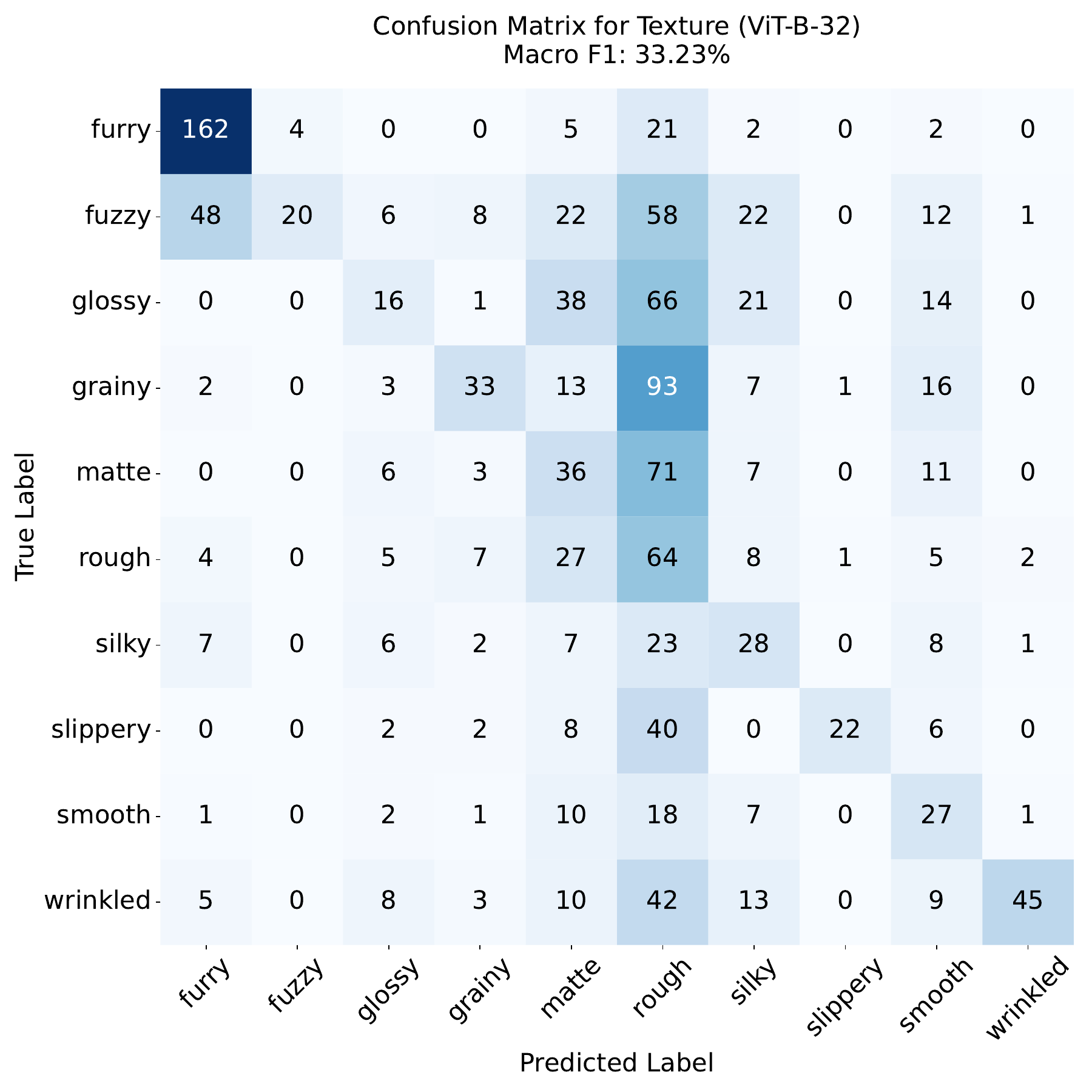}
    \\
    \includegraphics[width=0.325\linewidth]{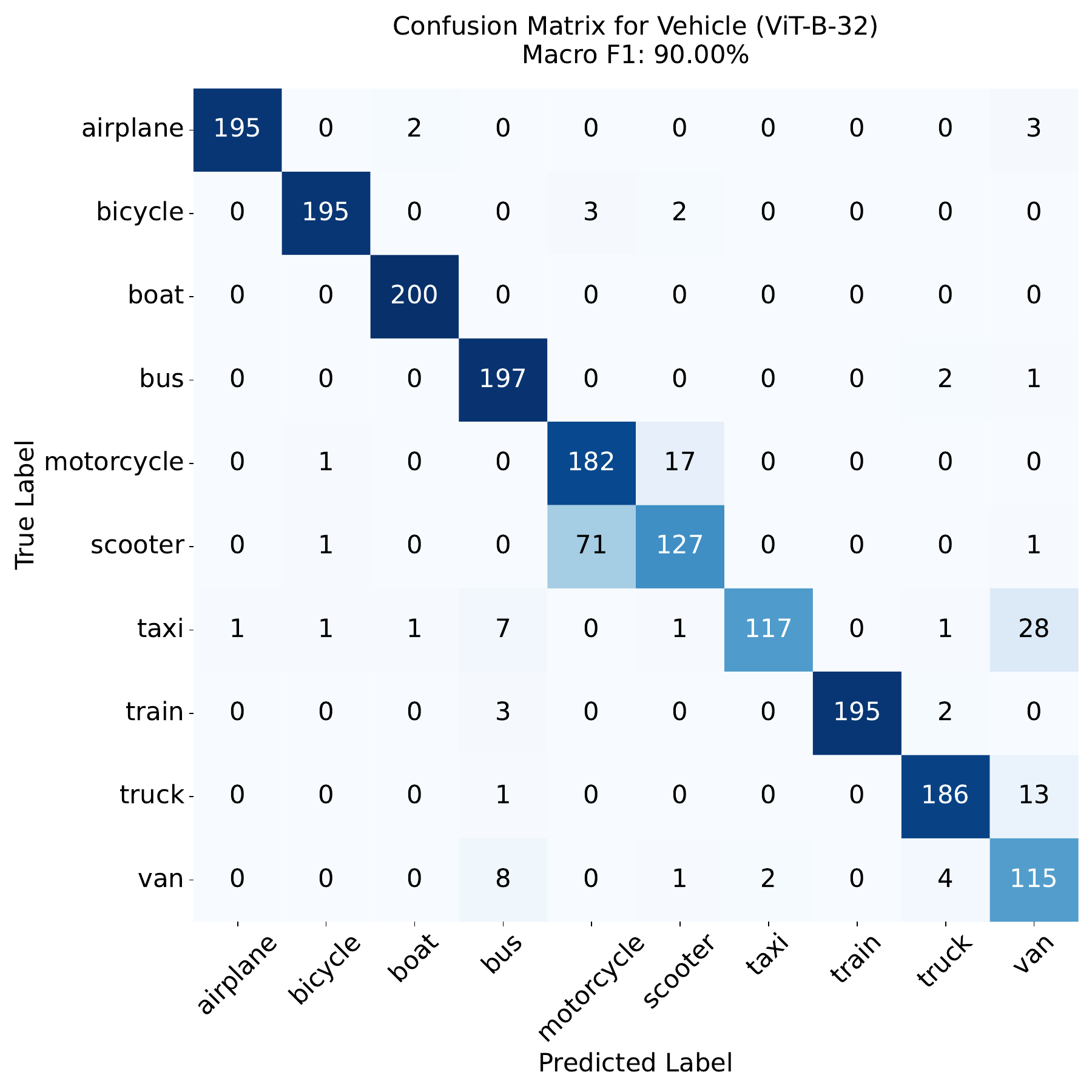}
    \hfill
    \includegraphics[width=0.325\linewidth]{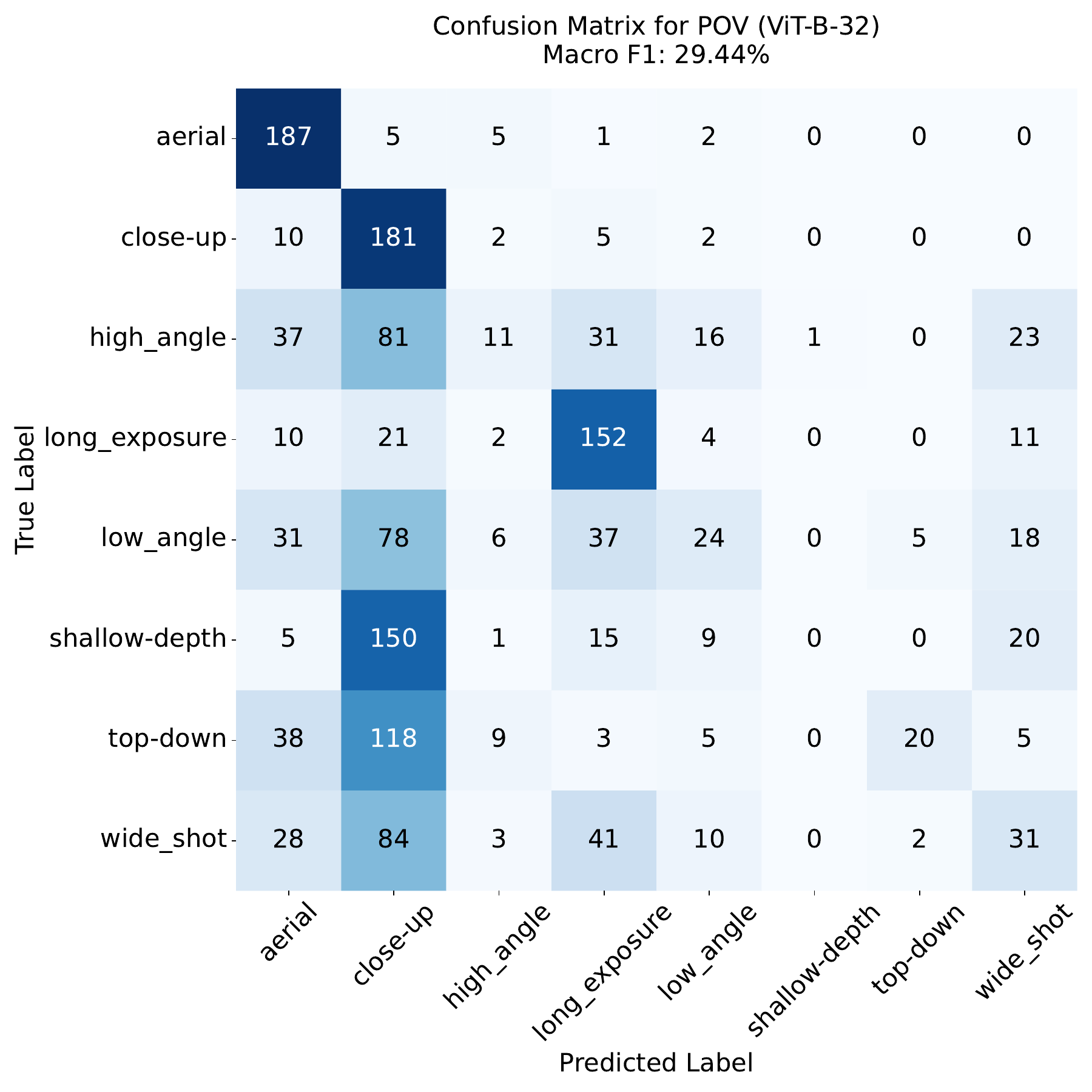}
    \hfill
    \includegraphics[width=0.325\linewidth]{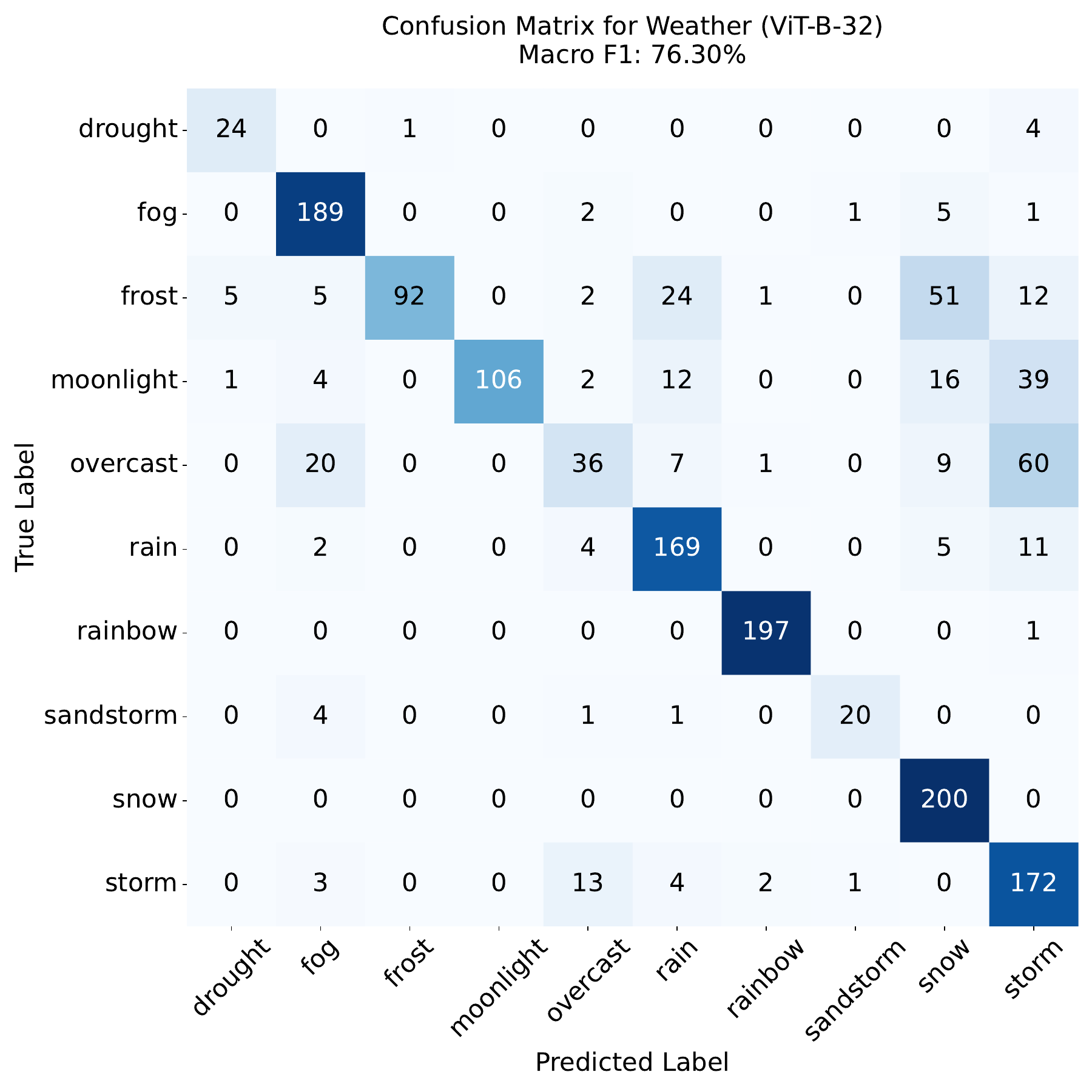}
    \caption{Confusion matrix for zero-shot prediction using OpenCLIP ViT-B-32 of each group.}
    \label{fig:confmat}
\end{figure}

\section{Computation Resources}
\label{sec:computation}

All experiments were conducted on a high-performance computing cluster equipped with 4×NVIDIA A100 GPUs (40 GB each), running CUDA 12.2 and driver version 535.161.07. The system also included an AMD EPYC 7313 16-core processor and 503 GB of RAM. For the numerical simulations, we trained over 120 models in total, requiring approximately 70 GPU-hours across 4 GPUs. On the MPI3D-Complex dataset, we trained 36 models, consuming approximately 27 GPU-hours. For the Causal3DIdent dataset, we trained 42 models, which required roughly 25 GPU-hours across 4 GPUs. Additionally, we generated 100,000 synthetic images for the Causal3DIdent dataset using Blender. Rendering was performed over four days on a separate workstation equipped with an AMD Ryzen 7 7700X 8-core processor (4.50 GHz) and a single NVIDIA RTX 4090 GPU (24 GB).

\section{Discussions}
\label{sec:discussion}
In this section, we reflect on the limitations of the current study, propose future research directions informed by our findings, and discuss broader implications.

\subsection{Limitations and Future Directions}

\paragraph{Estimating dataset biases.} Our study focuses on formalizing the effects of semantic misalignment, but does not directly address the empirical estimation of dataset-specific biases such as selection and perturbation bias. These biases are often latent and difficult to observe, particularly in large-scale, web-curated datasets. Yet, identifying and quantifying them is essential for improving interpretability, data quality, and downstream robustness. This task involves challenges in defining concept taxonomies, determining annotation granularity, and choosing evaluation metrics. Our framework suggests practical strategies: selection bias may be investigated through concept coverage conditioned on concept presence in the corresponding images, while comparisons with captions generated by strong image-to-text models can flag potential perturbations for image-grounded verification. We view the development of robust bias identification pipelines, such as those explored by \cite{yarom2023you}, as a valuable direction for future work.

\paragraph{Expanding the taxonomy of misalignment.} Multimodal data presents a wide range of potential misalignments across modalities. In this work, we deliberately focus on semantic misalignment in image-text pairs, a particularly prevalent and influential form of bias in vision-language learning. This type of misalignment is central to current multimodal research due to its relevance in large-scale pretraining paradigms and real-world applications (e.g., CLIP, ALIGN). By explicitly modeling selection and perturbation biases, our framework captures common semantic omissions and distortions that significantly affect representational quality.  Extending the analysis to other forms of misalignment, such as temporal lag in video-language data, modality-specific dropout, or ambiguity arising from multi-entity co-occurrence, poses additional challenges and may require substantially different modeling assumptions and techniques. We view our proposed abstraction as a tractable and principled foundation for understanding semantic misalignment and believe it can work as a building block for future explorations in this direction.

\paragraph{Modeling cross-modal semantic synergy.}
Multimodal representations often exhibit semantic synergies arising not only from the alignment of shared content but also from emergent meaning generated through cross-modal integration. At a theoretical level, this aligns with ideas in grounded cognition: cognition (and by extension, semantic representation) is rooted in the integration of perceptual, motor, and introspective modalities via \emph{simulation} \cite{barsalou2008grounded}. This motivates a conceptual distinction between two forms of semantic information: 1) shared information available from either modality, such as object identity in an image and its description, which our shared-latent formulation accommodates; and 2) information available only through joint observation of the modalities, as in some instances of sarcasm, irony, or humor. While our current theory focuses on shared semantics, extending the framework to capture such joint information remains a direction for future research.

\paragraph{Formulating missing data in the latent space.}
Our current framework addresses fixed semantic selection patterns and random perturbations within a prescribed subset of semantic variables in textual annotations. However, real-world datasets, particularly large-scale, user-generated corpora such as LAION-5B \citep{schuhmann2022laion}, often display random and unstructured semantic omissions. Extending our latent variable model to accommodate such random missingness (e.g., missing completely at random or missing not at random) poses both theoretical and practical challenges. Future research should examine the identifiability consequences of randomly missing semantic variables and establish conditions that ensure partial or probabilistic recovery of latent factors remains feasible.

\paragraph{Linearity of multimodal representations.}
Our theoretical analysis establishes almost-sure local identifiability of unbiased semantic factors at population global minima of the alignment--entropy surrogate, while empirical observations show that these factors can often be predicted accurately using linear probes. This motivates an important open question: under what conditions can identifiability up to a \emph{linear} transformation be rigorously established, without imposing overly restrictive assumptions on data-generating processes or latent structures? Clarifying this issue is not only theoretically significant but also practically beneficial, as linear representations enhance interpretability and simplify downstream applications. Future studies might investigate training objectives, regularization strategies, or architectural inductive biases explicitly designed to encourage linearity. Bridging this empirical regularity with robust theoretical guarantees represents a promising research frontier.

\subsection{Broader Implications}
\label{sec:impacts}

\paragraph{Linguistic relativity and epistemic constraints in multimodal AI.}
Our findings offer a computational perspective on how linguistic supervision shapes multimodal representations, drawing an analogy with linguistic relativity \citep{whorf2012language}. In multimodal AI, linguistic supervision implicitly determines which aspects of the visual world are foregrounded, and which are omitted. Within our theoretical setting, selection and perturbation biases in textual annotations act as \emph{epistemic filters}, determining which semantic variables the representations functionally depend on. This reframes dataset design as both an epistemic and normative act: inclusion or omission in annotations encodes a stance on salience, relevance, or appropriateness. Two practical insights follow. First, supporting diverse downstream tasks motivates supervision that faithfully captures the intended semantic scope without systematic omissions of relevant information. Second, curating the treatment of ethically sensitive or socially salient content offers a way to shape supervision in accordance with explicitly stated values \citep{bender2018data, crawford2021atlas}. Since omitted attributes may remain predictable through correlations with retained factors, their omission from captions alone does not ensure their exclusion from downstream predictions. When annotation is treated as an epistemic commitment, dataset design becomes a central tool for shaping both semantic fidelity and ethical alignment.


\paragraph{Human annotation biases as signals of implicit value judgments.} Beyond explicit annotation content, systematic omissions and mislabels may provide clues to annotators' attention and priorities \citep{itti2001computational}. Such patterns invite a complementary perspective: annotation errors can be studied not only as noise, but also as possible reflections of which semantic factors annotators consider salient. Interpreting these patterns requires accounting for annotation instructions, perceptual difficulty, ambiguity, and collection practices. Studying annotation biases alongside explicit human judgments may help distinguish these influences and inform models that better reflect human preferences and priorities \citep{gabriel2020artificial}.

\end{appendix}